\documentclass[11pt]{article}
\usepackage[bookmarksnumbered = true, colorlinks = true, linktocpage = true, bookmarksopen = true,linkcolor = black, urlcolor = black]{hyperref}
\usepackage[intlimits]{amsmath}
\usepackage{subcaption}
\usepackage{amssymb}
\usepackage{comment}
\usepackage{amsthm}
\usepackage{rotating}
\usepackage{mathtools}
\usepackage[]{algorithm}
\usepackage{algorithmic}
\usepackage{tikz}
\usepackage{authblk}
\usepackage{booktabs}
\usepackage{dsfont}
\usetikzlibrary{fit,shapes.misc}
\usepackage[utf8]{inputenc}
\usepackage[margin =1in]{geometry}
\usepackage{url}
\usepackage[titletoc]{appendix}
\usepackage[bottom,flushmargin,hang,multiple]{footmisc}
\usepackage{setspace}
\newtheorem{theorem}{Theorem}[section]
\newtheorem{lemma}[theorem]{Lemma}

\newtheorem{assumption}[theorem]{Assumption}

\newtheorem{remark}[theorem]{Remark}

\newtheorem{definition}[theorem]{Definition}

\usepackage{color, colortbl}
\usepackage{xcolor}
\definecolor{Lightgray}{rgb}{0.75,0.75,0.75}
\usepackage[titletoc]{appendix}
\usepackage{indentfirst}
\usepackage[numbers,sort]{natbib}
\usepackage[normalem]{ulem}
\newcommand{\floor}[1]{\left\lfloor #1 \right\rfloor}

\renewenvironment{proof}[1][Proof]{\noindent\textbf{#1.} }{\hfill\qed\vspace{\baselineskip}}

\title{Implicit Updates for Average-Reward Temporal Difference Learning}
\author{Hwanwoo Kim{$^\dagger$}, Dongkyu Derek Cho{$^\dagger$}, Eric Laber}
\affil{Department of Statistical Science, Duke University}
\date{\today}

\begin{document}
\maketitle
\def\thefootnote{$\dagger$}\footnotetext{equal contribution}
\begin{abstract}
Temporal difference (TD) learning is a cornerstone of reinforcement learning. In the average-reward setting, standard TD($\lambda$) is 
highly sensitive to the choice of step-size and thus requires careful tuning to maintain numerical stability. We introduce 
average-reward implicit TD($\lambda$), which employs an implicit fixed point update to provide data-adaptive stabilization while 
preserving the per iteration computational complexity of standard average-reward TD($\lambda$). In contrast to prior finite-time 
analyses of average-reward TD($\lambda$), which impose restrictive step-size conditions, we establish finite-time error bounds for the 
implicit variant under substantially weaker step-size requirements. Empirically, average-reward implicit TD($\lambda$) operates 
reliably over a much broader range of step-sizes and exhibits markedly improved numerical stability. This enables more efficient 
policy evaluation and policy learning, highlighting its effectiveness as a robust alternative to average-reward TD($\lambda$).
\end{abstract}

\section{Introduction}
Temporal difference (TD) learning \citep{barto2021reinforcement} is a core component of modern reinforcement learning (RL), combining 
the strengths of Monte Carlo sampling and dynamic programming thereby enabling efficient value estimation from state-action-reward 
trajectories exhibiting Markovian dependence. As a foundational method, TD learning underlies many RL algorithms and has been 
successfully applied across diverse domains, including robotics \citep{kober2013reinforcement}, financial decision-making 
\citep{nevmyvaka2006reinforcement}, and games \citep{tesauro1995temporal}. While originally developed in the discounted-reward 
setting, TD learning has since been adapted to the average-reward setting \citep{tsitsiklis1999average}, which can be more
natural in many applications \citep{giannoccaro2002inventory, dai2022queueing, tangamchit2002necessity}.

Despite its widespread use and practical relevance, standard average-reward TD($\lambda$) 
\citep{tsitsiklis1999average} is sensitive to step-size selection. 
From a theoretical standpoint, 
stability is certified by finite-time error bounds, and existing analyses establish such bounds only in 
small step-size regimes \citep{zhang2021average}.
In practice, larger step-sizes can accelerate learning 
but at the risk of numerical instability; 
conversely, smaller step
sizes are more numerically stable but can also yield slower learning.   
This stability–efficiency trade-off motivates methods that preserve the 
simplicity of the average-reward TD($\lambda$) while substantially expanding the range of step-sizes 
for which learning remains stable. We address this sensitivity by proposing an average-reward implicit 
TD($\lambda$) with finite-time error guarantees under substantially less restrictive step-size 
conditions. 
In addition, the proposed algorithm retains the 
computational complexity of standard 
average-reward TD($\lambda$).

\subsection{Related Literature} 
\paragraph{Discounted-Reward Setting.} Almost sure convergence of TD($\lambda$) with linear function 
approximation was first established in \citep{Tsitsiklis1997-hd}. Subsequent work derived finite-time error bounds under both i.i.d. data streams \citep{dalal2018finite} and Markovian samples, 
using projection-based mean-path analysis \citep{bhandari2018finite}, Lyapunov-function arguments 
\citep{srikant2019finite}, and induction-based proofs \citep{mitra2024simple}. In addition, TD-type 
methods formulated as two-time scale stochastic approximation algorithms—used for off-policy 
evaluation in the discounted setting—have been analyzed in 
\citep{sutton2008convergent,sutton2009fast,xu2019two}.

Despite the aforementioned theoretical developments in the discounted-reward setting, classical TD 
methods typically require restrictive step-size conditions 
\citep{bhandari2018finite,srikant2019finite, mitra2024simple} and display marked empirical 
sensitivity: larger steps may accelerate progress but risk divergence, whereas smaller steps improve 
stability at the cost of substantially slower convergence 
\citep{dabney2012adaptive,tamar2014implicit}. A principled remedy is to use implicit stochastic 
updates that recast the recursion as a fixed-point equation, providing data-adaptive stabilization, 
as shown in the stochastic optimization literature \citep{Toulis2014-mq,Toulis2014-bb,chee2023plus}. 
Building on this principle in reinforcement learning, recent work establishes 
asymptotic and finite-time error bounds 
for implicit variants of discounted TD in both on- and off-policy tasks without restrictive step-size requirements \citep{kim2025stabilizing}. 
Experiments further show improved numerical stability in both policy evaluation and control tasks.

\paragraph{Average-Reward Setting.} 
For foundations, background, and developments in average-reward policy evaluation, we refer to 
\citep{puterman2014markov,singh1994reinforcement,mahadevan1996average,abounadi2001learning,tsitsiklis2002average, dewanto2020average}. 
The first convergence analysis specific to average-reward TD($\lambda$) with linear feature 
approximation is due to \citep{tsitsiklis1999average}, under the assumption that the 
span of the feature vector does not include the constant vector of all ones 
(see Section \ref{sec:theoretical_analysis} for additional discussion). 
Under the same assumption, \citep{yu2009convergence} established 
asymptotic convergence of the average-reward LSPE($\lambda$), a least-squares based alternative to 
the average-reward TD($\lambda$). Relaxing the aforementioned feature space restriction, \citep{zhang2021finite} derived finite-time bounds for average-reward TD($\lambda$) with both 
constant and linearly decaying step-sizes (i.e., $t^{\text{th}}$ step-size $\propto 1/t$), while imposing a restrictive condition on the initial step-size. More recent progress on average-reward 
off-policy evaluation with function approximation includes an asymptotically convergent tabular off-policy TD algorithm \citep{wan2021learning} and extensions of gradient TD methods 
\citep{sutton2008convergent, sutton2009fast} for average-reward off-policy evaluation tasks with linear function approximation \citep{zhang2021finite}. Furthermore, TD-style methods for estimating a policy’s asymptotic variance of the cumulative 
reward in average-reward setting are developed in \citep{agrawalpolicy}.

\subsection{Contributions}
We show that the step-size sensitivity of TD($\lambda$), which has been
well documented in the discounted setting
also arises in the average-reward setting. 
To mitigate this sensitivity, 
we adopt the implicit stochastic update framework to construct 
average-reward implicit TD($\lambda$). We 
establish finite-time error bounds under markedly weaker step-size conditions 
than those in \citep{zhang2021finite}, and we demonstrate 
that this relaxation enables computationally efficient policy evaluation and learning across a range 
of examples. The primary contributions of our work are summarized as follows.  

\begin{itemize}
\item We propose average-reward implicit TD($\lambda$), which is more robust to step-size choice than the standard average-reward TD($\lambda$).
\item We provide finite-time error bounds under both constant and diminishing step-sizes, substantially relaxing the step-size conditions required by existing bounds for standard average-reward TD($\lambda$) \citep{zhang2021finite}, thereby explaining the improved numerical stability of the implicit variant.
\item In the case of a diminishing step-size, we establish the first finite-time 
 error bounds in the average-reward setting for step-size sequences of the form $\alpha_t \propto t^{-s}$ with $s\in(0,1)$, covering both square-summable ($s>1/2$) and non-square-summable ($0<s\le 1/2$) regimes, thereby further broadening the admissible family of step-sizes.
\item We empirically demonstrate the robustness and efficiency of the proposed method through comprehensive experiments in both policy evaluation and control tasks.
\end{itemize}

\section{Policy Evaluation in the Average-Reward Setting}

\paragraph{Problem Formulation.}
Consider an infinite‐horizon Markov decision process (MDP) defined by a finite state space $\mathcal{S}$, a finite action space
$\mathcal{A}$, a bounded reward function $r:\mathcal{S}\times\mathcal{A}\to[0,1]$, and a transition 
function $p:\mathcal{S}\times\mathcal{A}\times\mathcal{S}\to[0,1]$. 
Under a deterministic stationary
policy $\mu:\mathcal{S}\rightarrow \mathcal{A}$, at time $t$ with a 
current state $S^\mu_t$, the agent will take an action $A^\mu_t = \mu(S^\mu_t)$, 
receive a reward $R^\mu_t=r(S^\mu_t, A^\mu_t)$, and transition to
next state 
$S^\mu_{t+1}$ according to the probability distribution 
$p(\cdot|S^\mu_t, A^\mu_t)$. The resulting state sequence 
$\left\{S^\mu_t\right\}_{t\in \mathbb{N}}$ induced
by the policy $\mu$ forms a Markov chain with one‐step transition probabilities 
$
p^\mu\left(S^\mu_{t+1}| S^\mu_t\right)=p\left\lbrace S^\mu_{t+1} | S^\mu_t, A^\mu_t=\mu(S^\mu_t)\right\rbrace
$.\footnote{Since any Markov reward process arises from an MDP under a fixed policy, the general MDP setting covers the Markov reward process case.}
To simplify notation, let 
$
\mathcal{S} = \{1, 2, \cdots, |\mathcal{S}|\}
$ 
and define time-homogeneous transition probability matrix 
$
\boldsymbol{P}^\mu = \bigl[P^\mu_{ij}\bigr]_{i,j=1}^{|\mathcal{S}|}
$
with 
$
P^\mu_{ij} = p^\mu \left\lbrace S^\mu_{t+1}=j\mid S^\mu_{t}=i\right\rbrace
$. 
Likewise, let
$
\boldsymbol{r}^\mu = 
\left[r\{1, \mu(1)\}, \cdots, r\{|\mathcal{S}|, \mu(|\mathcal{S}|)\}\right]^\top
$
be the reward vector. 

One way to characterize the long-term performance of a given policy $\mu$ is via its average-reward, defined for each initial state $s\in\mathcal{S}$ as
$$
\omega^\mu(s):=\lim_{T\to\infty}\frac{1}{T}\mathbb{E}^{\mu}\left(\sum_{t=0}^{T-1} R^\mu_t \Bigm|S^\mu_0=s\right),
$$
where the expectation is taken over the randomness associated with the Markov chain $\left\{S^\mu_t\right\}_{t \in \mathbb{N}}$ induced by the policy $\mu$. Although the 
average-reward provides a natural evaluation criterion for $\mu$, the limit need not exist in general (see, e.g., Chapter 8 of \citep{puterman2014markov}). To guarantee the existence and uniqueness of the average-reward, it is common to make the following assumption.

\begin{assumption}\label{assumption:Markov}
The Markov chain $\left\{S^\mu_t\right\}_{t \in \mathbb{N}}$ is irreducible and aperiodic. 
\end{assumption}

Under Assumption \ref{assumption:Markov}, the chain has a unique stationary distribution 
$
\boldsymbol{\pi}^{\mu} = (\pi^{\mu}_i)_{i=1}^{|\mathcal{S}|}
$ 
satisfying 
$
{\boldsymbol{\pi}^\mu}^\top \boldsymbol{P}^\mu = {\boldsymbol{\pi}^\mu}^\top
$ 
with $\pi^{\mu}_i > 0$ for every $i \in \mathcal{S}$ \citep{levin2017markov}. Under the same assumption, one can further show that the average-reward is independent of the initial state \citep{bertsekas1996neuro}; that is,
$
\omega^\mu(s) = {\boldsymbol{\pi}^\mu}^\top \boldsymbol{r}^\mu, ~ \forall s \in \mathcal{S}.
$
Unlike its discounted counterpart, the average-reward criterion carries no information about the relative desirability of individual states. To quantify long-run, state-dependent performance under a stationary policy $\mu$, we introduce the basic differential value function $v^\mu:\mathcal{S}\to\mathbb{R}$,
$$
v^\mu(s):=\mathbb{E}^\mu\left\{\sum_{t=0}^{\infty}\left(R^\mu_t - \omega^\mu\right)\Bigm|S^\mu_0=s\right\},
$$
which measures the relative advantage (or disadvantage) of starting in state $s\in\mathcal{S}$. Accordingly, the quantities of interest are the 1) average reward: $\omega^\mu$ and 2) pairwise contrast: $v^\mu(s)-v^\mu(s')$ for any $s, s' \in \mathcal{S}$ 
which captures the comparative long-run performance of states $s,s'\in\mathcal{S}$.

In high-dimensional or continuous state spaces, a common strategy is to use linear function approximation, where we model the  pairwise difference 
 as 
$$
v^\mu(s)-v^\mu(s')\approx\{\boldsymbol{\phi}(s)-\boldsymbol{\phi}(s')\}^\top\boldsymbol{\theta}
$$ 
with a user-chosen feature map $\boldsymbol{\phi}(s)\in\mathbb{R}^d$ 
and weights $\boldsymbol{\theta}\in\mathbb{R}^d$. Because adding a 
constant to $v^\mu(s)$ leaves all differences unchanged, it suffices to 
learn $v^\mu(s)$ up to an additive constant. Let 
$\boldsymbol{\Phi}\in\mathbb{R}^{|\mathcal S|\times d}$ be the feature 
matrix whose $i^{\text{th}}$ row is $\boldsymbol{\phi}(i)^\top$ and
$\boldsymbol{M} := \mathrm{diag}(\boldsymbol{\pi}^\mu).$ Writing 
$
\boldsymbol{v}^\mu := [v^\mu(1),\ldots,v^\mu(|\mathcal{S}|)]^\top,
$ 
one has the series representation
$
\boldsymbol{v}^\mu = \sum_{t=0}^{\infty}(\boldsymbol{P}^\mu)^t\left(\boldsymbol{r}^\mu- \omega^\mu \boldsymbol{e}\right),
$
where $\boldsymbol{e}$ is the all-ones vector. With the weighted norm $\|\boldsymbol{x}\|_{\boldsymbol M}:=(\boldsymbol{x}^\top \boldsymbol M \boldsymbol{x})^{1/2}$, our second goal translates to finding $\boldsymbol{\theta}^*$ such that the weighted discrepancy
$$
\inf_{c\in\mathbb{R}} \bigl\|\boldsymbol{\Phi}\boldsymbol{\theta}^* - (\boldsymbol{v}^\mu + c\boldsymbol{e})\bigr\|_{\boldsymbol M}
$$
is small. This quantity is zero when the feature space contains a constant shift of the basic differential value function (i.e., the differential value function). When the weighted discrepancy is small, it indicates that $\boldsymbol{\phi}(s)^\top\boldsymbol{\theta}^*$ approximates $v^\mu(s)$ up to an additive constant, so the estimation of the contrasts $v^\mu(s)-v^\mu(s')$ is correspondingly accurate, improving as the discrepancy decreases.

\paragraph{Average-Reward TD($\lambda$) with Linear Approximation.}
The average-reward TD$(\lambda)$ algorithm 
\citep{tsitsiklis1999average} is a widely used stochastic-approximation 
method to achieve the aforementioned goals. At the $t^{\text{th}}$ 
iteration, the average‐reward TD($\lambda$) algorithm maintains both 
$\widehat{\omega}_t$, an estimate of the average-reward $\omega^\mu$, and an 
estimate $\widehat{\pmb{\theta}}_t$ of the optimal weight 
$\boldsymbol{\theta}^*$. With non-increasing positive step-sizes 
$\alpha_t,\beta_t$ and exponential weighting parameter 
$\lambda\in[0,1)$, the update rules are given by
\begin{equation}
\label{eqn:param_update_1}
\begin{aligned}
\widehat{\omega}_{t+1} &= \widehat{\omega}_t + \alpha_t\left(R^\mu_t - \widehat{\omega}_t\right),\\
\widehat{\pmb{\theta}}_{t+1} &= \widehat{\pmb{\theta}}_t + \beta_t \delta_t \boldsymbol{z}_t,
\end{aligned}
\end{equation}
where the eligibility trace $\boldsymbol{z}_t$ and TD error $\delta_t$ are
$$
\boldsymbol{z}_t=\sum_{i=0}^t\lambda^{t-i}\boldsymbol{\phi}\left(S^\mu_i\right),~~
\delta_t = R^\mu_t - \widehat{\omega}_t + {\widehat{\pmb{\theta}}_t}^\top\left\{\boldsymbol{\phi}\left(S^\mu_{t+1}\right)-\boldsymbol{\phi}\left(S^\mu_t\right)\right\},
$$
each respectively representing the geometrically weighted average of past feature vectors at visited states and the one‐step TD error, which measures how the reward (after subtracting the current average‐reward estimate $\widehat{\omega}_t$) plus the estimated value of the next state differs from the current value estimate. 
We restrict our attention to a single-time-scale average-reward TD($\lambda$) algorithm by assuming $\alpha_t=c_\alpha\beta_t$ with fixed $c_\alpha>0$ \citep{tsitsiklis1999average,zhang2021average}. Exploring distinct decay rates, an instance of the two-time-scale stochastic approximation framework \citep{borkar2008stochastic}, is interesting but outside our scope. 

\paragraph{Step-Size Sensitivity.}
Despite its foundational role in RL, standard average-reward TD($\lambda$) suffers from numerical 
sensitivity to step-size selection. To illustrate this issue, we present a simple numerical example. 
We consider a Markov reward process (MRP) with $|\mathcal{S}| = 100$ states and evaluate the 
performance of average-reward TD($\lambda$) learning with hyperparameter configuration 
$(c_\alpha, \lambda) = (1.0, 0.25)$ with a predetermined constant step-size 
$\beta_t = \beta_0 \in (0, 2), ~\forall t \in \mathbb{N}$. The objective is to estimate both 
the optimal weight $\boldsymbol{\theta}^*$ and the average-reward $\omega^\mu$. Detailed descriptions of 
the evaluation criterion (loss function) and experimental setting are provided in 
Sections~\ref{sec:theoretical_analysis} and \ref{SEC:NUMERICS}, respectively. 

Figure~\ref{fig:demonstration} illustrates the instability induced by step-size choices. The left panel shows a non-monotonic trend in 
performance: overly small step-sizes (e.g., $\beta_0 < 0.20$) lead to slow convergence, while a modest increase in step-size causes 
the loss function values to grow rapidly. The right panel presents the result for a moderately large step-size ($\beta_0 = 1.8$), 
where the average-reward TD($\lambda$) iterates exhibit oscillatory behavior. These empirical findings highlight the sensitivity of 
the TD learning to step-size selection and motivate the need for algorithms that are robust to such choices. In the following 
sections, we propose and analyze one such approach.

\begin{figure}[htp]
\centering
\caption{\small Sensitivity of average-reward TD($\lambda$) to the choice of step-size with exponential weighting parameter $\lambda = 0.25$ and step-size ratio $c_\alpha = 1.0$. The solid line denotes the mean, and the shaded region indicates the 95\% confidence interval. }
\label{fig:demonstration}
\begin{subfigure}[t]{0.4\linewidth}
\includegraphics[width=\linewidth]{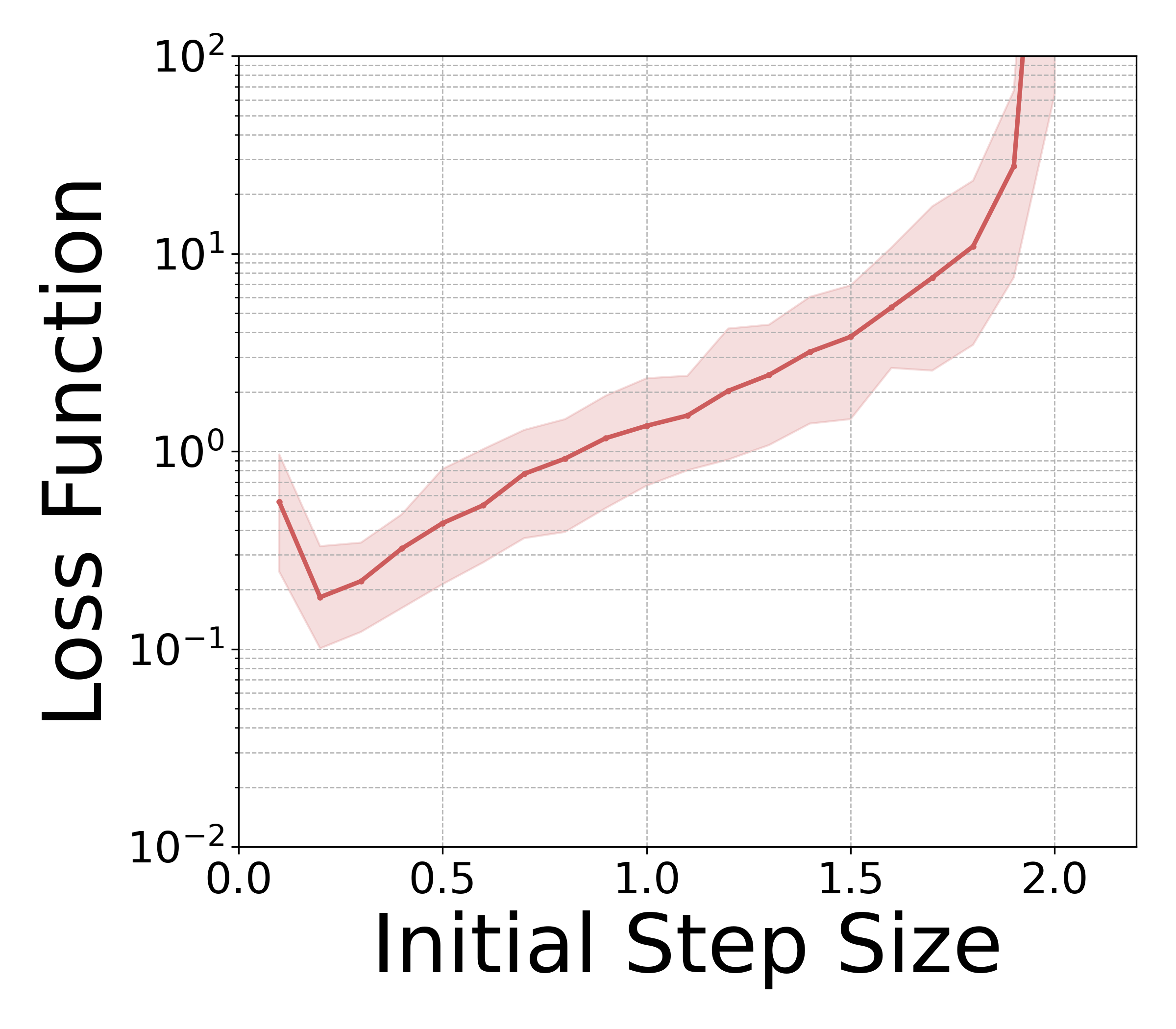}
\caption{\small Performance of average-reward TD($\lambda$) for step-sizes $\beta_0$ ranging from $0.1$ to $2.0$.}
\end{subfigure}
\begin{subfigure}[t]{0.4\linewidth}
\includegraphics[width=\linewidth]{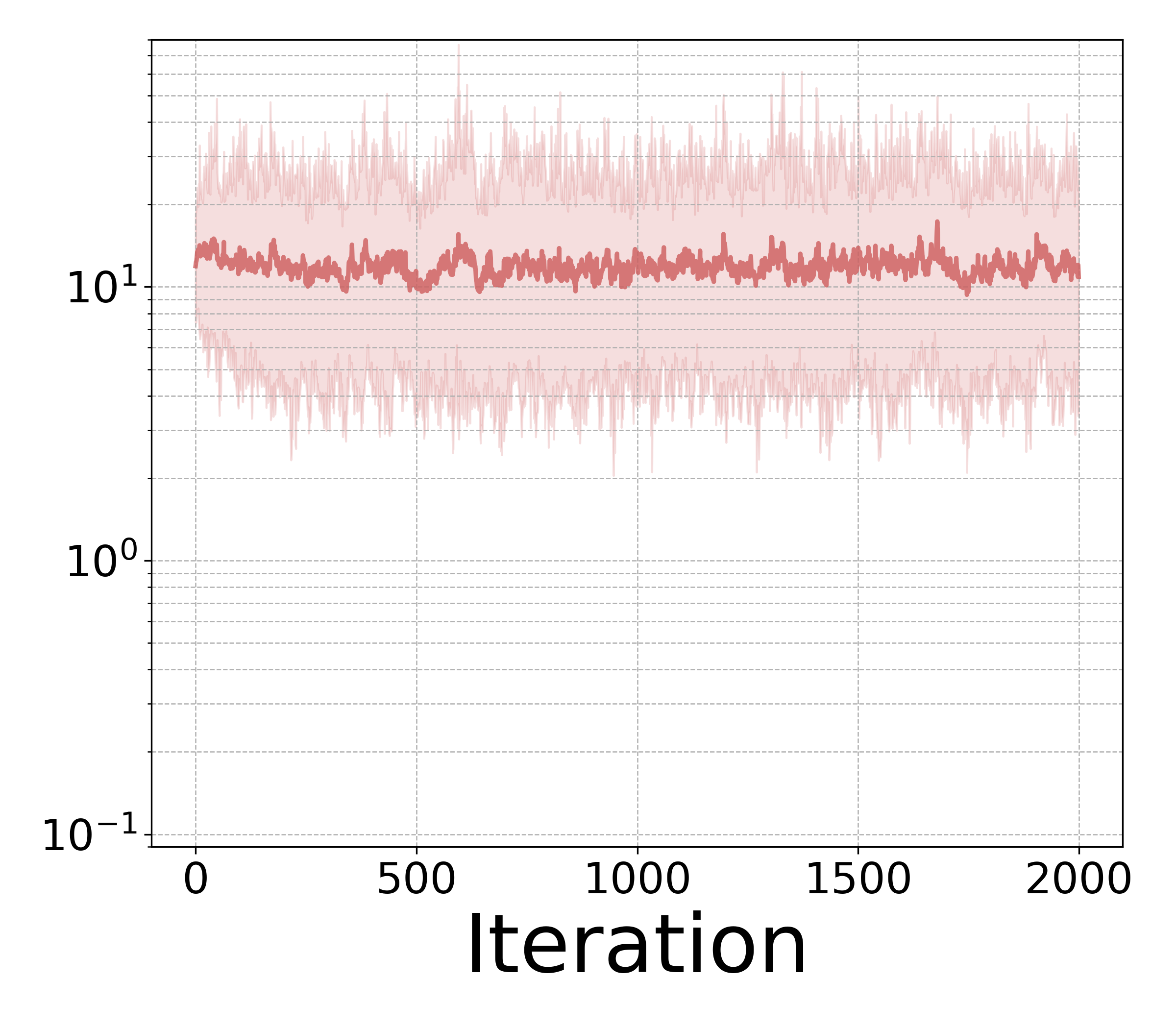}
\caption{\small Performance over iterations with step-size $\beta_0 = 1.8$, showing no improvement.}
\end{subfigure}
\end{figure}

\section{Average-Reward Implicit TD($\lambda$)}
\label{sec:implicit_TD_without_projection}

\noindent As we have seen in the previous section, average‐reward TD($\lambda$) demands carefully tuned step-sizes for stability. Implicit stochastic recursions, developed for stochastic gradient descent \citep{Toulis2014-mq, Toulis2014-bb, Toulis2021-ma, chee2023plus} and more recently extended to discounted‐reward on- and off‐policy TD \citep{kim2025stabilizing}, rewrite each update as a fixed‐point equation by allowing the gradient or TD error to depend on the new iterate. This reformulation automatically induces an adaptive shrinkage in the effective step-size, vastly improving numerical stability without increasing computational complexity. Building on this idea, we introduce the average‐reward implicit TD($\lambda$), which retains the simplicity and efficiency of standard average‐reward TD($\lambda$) while supporting more flexible step-size choices with finite-time error guarantees. In this section, we propose a novel average‐reward TD($\lambda$) algorithm that incorporates implicit updates into its recursive structure.

To derive implicit average-reward TD($\lambda$) updates, recall the update rule for $\widehat{\pmb{\theta}}_{t+1}$ given in \eqref{eqn:param_update_1}:  
\begin{equation*}
\begin{split}
\widehat{\pmb{\theta}}_{t+1} &= \widehat{\pmb{\theta}}_t+ \beta_t \left(R^\mu_t - \widehat{\omega}_t +  \boldsymbol{\phi}_{t+1}^\top \widehat{\pmb{\theta}}_t- \boldsymbol{\phi}_t^\top \widehat{\pmb{\theta}}_t\right) \boldsymbol{z}_t \\ &= \widehat{\pmb{\theta}}_t+ \beta_t \left\lbrace R^\mu_t - \widehat{\omega}_t +  \boldsymbol{\phi}_{t+1}^\top \widehat{\pmb{\theta}}_t- (\boldsymbol{z}_t - \lambda \boldsymbol{z}_{t-1})^\top \boldsymbol{\theta}_{t} \right\rbrace \boldsymbol{z}_t \\
&= \widehat{\pmb{\theta}}_t+ \beta_t \left(R^\mu_t - \widehat{\omega}_t +  \boldsymbol{\phi}_{t+1}^\top \widehat{\pmb{\theta}}_t + \lambda \boldsymbol{z}_{t-1}^\top \widehat{\pmb{\theta}}_t-  \boldsymbol{z}_t^\top \widehat{\pmb{\theta}}_t\right) \boldsymbol{z}_t,
\end{split}
\end{equation*}
where we have used the identity $\boldsymbol{\phi}_t = \boldsymbol{z}_t - \lambda \boldsymbol{z}_{t-1}$ in the second equality. At time $t$, we update $\widehat{\omega}_t$ and $\widehat{\pmb{\theta}}_t$ using a rule that depends on both the current iterate $(\widehat{\omega}_t, \widehat{\pmb{\theta}}_t)$ and their updated values $(\widehat{\omega}_{t+1}, \widehat{\pmb{\theta}}_{t+1})$:
\begin{align}
\widehat{\omega}_{t+1}
&= \widehat{\omega}_t + c_\alpha \beta_t (R^\mu_t - \textcolor{red}{\widehat{\omega}_{t+1}}),
\label{eqn:param_update_im2}\\
\widehat{\pmb{\theta}}_{t+1} 
&= \widehat{\pmb{\theta}}_t+ \beta_t (R^\mu_t - \widehat{\omega}_t +  \boldsymbol{\phi}_{t+1}^\top \widehat{\pmb{\theta}}_t + \lambda \boldsymbol{z}_{t-1}^\top \widehat{\pmb{\theta}}_t-  \boldsymbol{z}_t^\top \textcolor{red}{\widehat{\pmb{\theta}}_{t+1}}) \boldsymbol{z}_t.
\label{eqn:param_update_im1}
\end{align}
Solving the above recursions \eqref{eqn:param_update_im2} and \eqref{eqn:param_update_im1} admits the update rule for the average-reward implicit TD($\lambda$) algorithm. Lemma \ref{LEMMA:im_update_lemma} below characterizes the average-reward implicit TD($\lambda$) algorithm, and its proof is given in the supplementary materials. 

\begin{lemma} 
\label{LEMMA:im_update_lemma}
    Average-reward implicit TD($\lambda$) updates given in \eqref{eqn:param_update_im2} and \eqref{eqn:param_update_im1} can be expressed as 
    \begin{align*}
        \widehat{\omega}_{t+1}&=\widehat{\omega}_t + \frac{c_\alpha \beta_t}{1+c_\alpha \beta_t}( R^\mu_t -\widehat{\omega}_t) \\
        \widehat{\pmb{\theta}}_{t+1}&= \widehat{\pmb{\theta}}_t + \frac{\beta_t}{1 + \beta_t \|\boldsymbol{z}_t\|_2^2} \left( R^\mu_t - \widehat{\omega}_t + \boldsymbol{\phi}_{t+1}^\top \widehat{\pmb{\theta}}_t - \boldsymbol{\phi}_t^\top \widehat{\pmb{\theta}}_t \right) \boldsymbol{z}_t.
    \end{align*}
\end{lemma}
The update rule in Lemma~\ref{LEMMA:im_update_lemma} highlights a key mechanism underlying the robustness of the average-reward implicit TD($\lambda$) learning. Unlike the standard average-reward TD($\lambda$) methods, at each step $t \in \mathbb{N}$, the implicit updates dynamically rescale the step-size based on the magnitude of the eligibility trace as well as the step-size ratio parameter. Such shrinkage arises naturally from the implicit update mechanism, reducing the burden of laborious tuning. Importantly, the implicit algorithm has the same space and time complexity as the standard method in the average-reward setting, making it a practical replacement without additional computational burden or implementation difficulty. The benefits of this mechanism will be further clarified in the forthcoming theoretical analysis and subsequently illustrated
through a suite of numerical examples. 

To further enhance numerical stability and facilitate  theoretical analysis of the average‑reward implicit TD$(\lambda)$ algorithm, we introduce a projection step that forces each iterate $\widehat{\boldsymbol{\Theta}}_t:= [\widehat{\omega}_t,\widehat{\pmb{\theta}}_t]^\top$ to lie within the Euclidean ball of radius $R_{\boldsymbol{\Theta}}$ by enforcing the constraints
$\|\boldsymbol{\Theta}\|_2 \le R_{\boldsymbol{\Theta}}$. Such projection‐based stabilization is well‐studied in both the stochastic optimization and reinforcement learning literatures, and numerous theoretical results have been established \citep{nemirovski2009robust, bubeck2015convex, bhandari2018finite, xu2019two, zou2019finite, zhang2023convergence}. In practice, one can choose $R_{\boldsymbol{\Theta}}$ sufficiently large to ensure the limit point of $\widehat{\boldsymbol{\Theta}}_t$ is contained in the Euclidean ball of radius $R_{\boldsymbol{\Theta}}$. A complete algorithmic description of the average-reward implicit TD($\lambda$) is provided in Algorithm \ref{algorithm:with_projection}.

\begin{algorithm}[h]
\caption{Average-reward implicit TD($\lambda$) (with projection)}
\label{algorithm:with_projection}
\begin{algorithmic}[1]
\STATE \textbf{Input:} exponential weighting parameter \( \lambda \in [0, 1) \), basis functions \( \{\boldsymbol{\phi}_k\}_{k=1}^d \), step-size \( \{\beta_t\}_{t \in \mathbb{N}} \), step-size ratio parameter \( c_\alpha \), projection radius $R_{\boldsymbol{\Theta}}$ 
\STATE Initialize: \( \widehat{\omega}_0 \), \( \widehat{\pmb{\theta}}_0\), $S^\mu_0$ and eligibility trace \( \boldsymbol{z}_{-1} = 0 \). 
\FOR{ \( t = 0, 1, 2, \dots \) }
    \STATE Receive data: \(  (S^\mu_t, R^\mu_t, S^\mu_{t+1}) \)
    \STATE Get TD error: \vspace{-2mm}
    \begin{align*}
    \delta_t = R^\mu_t - \widehat{\omega}_t + \boldsymbol{\phi}(S^\mu_{t+1})^\top \widehat{\pmb{\theta}}_t - \boldsymbol{\phi}(S^\mu_t)^\top \widehat{\pmb{\theta}}_t
    \end{align*}\vspace{-6mm}
    \STATE Update eligibility trace: \( \boldsymbol{z}_t = \lambda \boldsymbol{z}_{t-1} + \boldsymbol{\phi}(S^\mu_t) \)
    \STATE Update parameters: \vspace{-2mm}
    \begin{align*}
    \widehat{\omega}_{t+1} &= \widehat{\omega}_t + \frac{c_\alpha \beta_t}{1 + c_\alpha \beta_t} \left(R^\mu_t - \widehat{\omega}_t\right), \\
    \widehat{\pmb{\theta}}_{t+1}& = \widehat{\pmb{\theta}}_t + \frac{\beta_t}{1 + \beta_t \|\boldsymbol{z}_t\|_2^2} \delta_t \boldsymbol{z}_t 
    \end{align*}\vspace{-2mm}
    \STATE For projected average-reward implicit TD($\lambda$):
    if $(\widehat{\omega}_{t+1})^2 + \|\widehat{\pmb{\theta}}_{t+1}\|^2 \ge R_{\boldsymbol{\Theta}}^2$, 
    \\ \vspace{-3mm}
    \begin{align*}
        \widehat{\omega}_{t+1} &= \frac{R_{\boldsymbol{\Theta}}}{\sqrt{(\widehat{\omega}_{t+1})^2 + \|\widehat{\pmb{\theta}}_{t+1}\|_2^2}}\widehat{\omega}_{t+1}, \\ \widehat{\pmb{\theta}}_{t+1}& = \frac{R_{\boldsymbol{\Theta}}}{\sqrt{(\widehat{\omega}_{t+1})^2 + \|\widehat{\pmb{\theta}}_{t+1}\|_2^2}}\widehat{\pmb{\theta}}_{t+1}
    \end{align*}
\ENDFOR
\end{algorithmic}
\end{algorithm}

\section{Theoretical Analysis}
\label{sec:theoretical_analysis}
In this section, we provide a theoretical analysis of the average-reward implicit TD($\lambda$) algorithm incorporating a projection 
step. We first assume that the columns of \( \boldsymbol{\Phi} \) are linearly independent, which implies that the basis functions 
span a \( d \)-dimensional feature space. Such an assumption ensures that redundant basis functions can be removed without loss of 
expressiveness. Hereafter, we use $\|\cdot\|$ to denote the Euclidean norm for vectors and operator norm for matrices. We assume that 
the feature vectors are normalized so that \( \|\boldsymbol{\phi}(i)\| \le 1 \) for all \( i \in \mathcal{S} \). Lastly, 
$\mathbb{E}^\mu$ denotes expectation with respect to the Markov chain $\{S^\mu_t\}_{t\in\mathbb{N}}$ under policy $\mu$ with a fixed $S^\mu_0$, and 
$\mathbb{E}^{\boldsymbol{\pi}^\mu}$ denotes expectation with respect to the stationary distribution of this chain.

To facilitate our analysis, we formulate the average-reward TD($\lambda$) update into a matrix notation form, given by
\begin{align*}
\begin{bmatrix}
\widehat{\omega}_{t+1} \\
\widehat{\pmb{\theta}}_{t+1}
\end{bmatrix}
&=
\begin{bmatrix}
\widehat{\omega}_t\\
\widehat{\pmb{\theta}}_t
\end{bmatrix}
+
\beta_t
\begin{bmatrix}
- c_{\alpha} & 0 \\
- {\boldsymbol{z}_t}  & {\boldsymbol{z}_t}(\boldsymbol{\phi}_{t+1}^{\top}-\boldsymbol{\phi}_{t}^{\top}) 
\end{bmatrix}
\begin{bmatrix}
\widehat{\omega}_t \\
\widehat{\pmb{\theta}}_t
\end{bmatrix}
+
\begin{bmatrix}
c_{\alpha} R^\mu_t \\
{R^\mu_t \boldsymbol{z}_t}
\end{bmatrix},
\end{align*}
which can be succinctly written as
\begin{equation}
\label{eqn:update_rule_mat}
    \widehat{\boldsymbol{\Theta}}_{t+1}  = \widehat{\boldsymbol{\Theta}}_t + \beta_t \{\boldsymbol{A}(\boldsymbol{X}_t) \widehat{\boldsymbol{\Theta}}_t + \boldsymbol{b}(\boldsymbol{X}_t)\}
\end{equation}
and its implicit version is given by
\begin{equation}
\label{eqn:update_rule_imp_mat}
    \widehat{\boldsymbol{\Theta}}_{t+1}  = \widehat{\boldsymbol{\Theta}}_t + \boldsymbol{D}_t \{\boldsymbol{A}(\boldsymbol{X}_t) \widehat{\boldsymbol{\Theta}}_t + \boldsymbol{b}(\boldsymbol{X}_t)\}
\end{equation}
where 
\begin{align*}
&\widehat{\boldsymbol{\Theta}}_t := \begin{bmatrix}
\widehat{\omega}_t \\
\widehat{\pmb{\theta}}_t
\end{bmatrix}, 
\quad \boldsymbol{A}(\boldsymbol{X}_t) := \begin{bmatrix}
- c_{\alpha} & 0 \\
- {\boldsymbol{z}_t}  & {\boldsymbol{z}_t}(\boldsymbol{\phi}_{t+1}^{\top}-\boldsymbol{\phi}_{t}^{\top})
\end{bmatrix},\\
&\boldsymbol{b}(\boldsymbol{X}_t) := \begin{bmatrix}
c_{\alpha} R^\mu_t \\
{R^\mu_t \boldsymbol{z}_t}
\end{bmatrix}, \quad 
\boldsymbol{D}_t := \begin{bmatrix}
        \frac{1}{1+c_{\alpha}\beta_t} & 0 \\
        0 & \frac{1}{1+\beta_t ||\boldsymbol{z}_t||^2} \boldsymbol{I}_{d}
    \end{bmatrix}
\end{align*}
with $\boldsymbol{X}_t := (S^\mu_t ,S^\mu_{t+1}, \boldsymbol{z}_t)$. 

Under suitable technical conditions, if $\boldsymbol{A}:=\mathbb{E}^{\boldsymbol{\pi}^\mu}[\boldsymbol{A}(\boldsymbol{X}_t)]$ is negative definite, results from stochastic approximation \citep{benveniste2012adaptive} imply that the iterate $\widehat{\boldsymbol{\Theta}}_t$ converges almost surely to $\boldsymbol{\Theta}^*=[ \omega^\mu, \boldsymbol{\theta}^* ]^\top$, which solves $\boldsymbol{A}\boldsymbol{\Theta}^*+\boldsymbol{b}=0$ with $\boldsymbol{b}:=\mathbb{E}^{\boldsymbol{\pi}^\mu}[\boldsymbol{b}(\boldsymbol{X}_t)]$. Earlier work \citep{tsitsiklis1999average} established almost sure convergence $\widehat{\boldsymbol{\Theta}}_t$ to $\boldsymbol{\Theta}^*$ by assuming $\boldsymbol{A}$ is negative definite (up to left multiplication by a diagonal matrix). However, such an assumption excludes cases where the feature matrix $\boldsymbol{\Phi}$ can yield value predictions that are constant across all states, i.e., when $\boldsymbol{e}$ lies in the column space of $\boldsymbol{\Phi}$.

To relax the aforementioned assumption, \citep{zhang2021finite} considered an auxiliary iterate $[\widehat{\omega}_t, \Pi_{\mathbb O}\widehat{\pmb{\theta}}_t]^\top$, where $\Pi_{\mathbb O}$ denotes projection onto $\mathbb O$, defined as the orthogonal complement of 
$
\mathbb{S}_{\boldsymbol{\Phi},\boldsymbol e}:=\operatorname{span}\{\boldsymbol{\theta}: \boldsymbol{\Phi}\boldsymbol{\theta}=\boldsymbol e\}.
$
Projecting onto $\mathbb O$ thus removes the constant-shift direction, i.e., the component of $\boldsymbol{\Phi}\boldsymbol{\theta}$ aligned with the all-ones vector direction. Such component adds the same constant to every state's value prediction and does not affect estimates of value contrasts $v^\mu(s)-v^\mu(s')$. Accordingly, it is natural to assess performance using the projected iterate $\Pi_{\mathbb O}\widehat{\pmb{\theta}}_t$ since any change in $\widehat{\pmb{\theta}}_t$ along $\operatorname{span}\{\boldsymbol e\}$ is not identifiable in the average-reward setting and can be ignored. Furthermore, on $\mathbb R\times\mathbb O$, one can restore strict negative definiteness of the matrix $\boldsymbol A$, formalized as Lemma \ref{LEMMA:NEG_DEF} below.

\begin{lemma}[Lemma 2 of \citep{zhang2021finite}]\label{LEMMA:NEG_DEF}
For $\lambda \in (0,1)$, let $\boldsymbol{M} = \text{diag}\left(\pi^\mu_1, \cdots, \pi^\mu_{|\mathcal{S}|}\right)$ and $\boldsymbol{P}^{(\lambda)} = (1 - \lambda) \sum_{m=0}^{\infty} \lambda^m (\boldsymbol{P}^\mu)^{m+1}$. Under Assumption \ref{assumption:Markov}, we have
$$
\Delta := \hspace{-3mm}\min_{\|\boldsymbol{\theta}\| = 1, \boldsymbol{\theta} \in \mathbb{O}} \boldsymbol{\theta}^\top \boldsymbol{\boldsymbol{\Phi}}^\top \boldsymbol{M} \left( \boldsymbol{I} -  \boldsymbol{P}^{(\lambda)}  \right) \boldsymbol{\boldsymbol{\Phi}} \boldsymbol{\theta} >0.
$$
In addition, for $c_\alpha \ge \Delta + \sqrt{\frac{1}{\Delta^2 (1-\lambda)^4} - \frac{1}{(1-\lambda)^2}}$, 
$$
\boldsymbol{\Theta}^\top \boldsymbol{A} \boldsymbol{\Theta} \le -\frac{\Delta}{2} \|\boldsymbol{\Theta}\|^2, ~~\text{for any} ~~\boldsymbol{\Theta} \in \mathbb{R} \times \mathbb{O}.
$$
\end{lemma}
With the negative definiteness of the matrix $\boldsymbol{A}$, the limit point $\boldsymbol{\Theta}^*$ is assured to be unique and one can then ask how far the auxiliary iterates are from the limit point. Specifically non-asymptotic bounds on $\left( \widehat{\omega}_t - \omega^\mu \right)^2 + \left\| \Pi_{\mathbb O}\left(\widehat{\pmb{\theta}}_t - \boldsymbol{\theta}^*\right) \right\|^2$ were established both for constant and decreasing step-size schedules \citep{zhang2021finite}. In addition to the finite-time error bounds, the approximation quality of $\boldsymbol{\theta}^*$ within the chosen feature class is captured by the $\boldsymbol{M}$-weighted discrepancy
\[
\inf_{c\in\mathbb{R}}\|\boldsymbol{\Phi}\boldsymbol{\theta}^*-(\boldsymbol{v}^\mu+c\boldsymbol e)\|_{\boldsymbol M}
\;\le\;\,
\tfrac{\inf_{\boldsymbol{\theta}\in\mathbb{R}^d,\,c\in\mathbb{R}}
\|\boldsymbol{\Phi}\boldsymbol{\theta}-(\boldsymbol{v}^\mu+c\boldsymbol e)\|_{\boldsymbol M}}{\sqrt{1-c_\lambda^{2}}}.
\]
with $c_\lambda\in[0,1)$ and $c_\lambda\to 0$ as $\lambda\to 1$ \citep{zhang2021finite}. Note that the right-hand term involves the best error achievable within the feature class. Thus $\boldsymbol{\theta}^*$ is optimal up to a multiplicative factor, which approaches $1$ as $\lambda\to1$. In particular, if the feature class is rich enough to represent any one differential value function, the best achievable error is zero; that is,
$
\inf_{c\in \mathbb{R}}\big\|\boldsymbol{\Phi}\boldsymbol{\theta}^*-(\boldsymbol v^\mu+c\boldsymbol e)\big\|_{\boldsymbol M}=0.
$

\paragraph{Non-asymptotic Analysis of Average-Reward Implicit TD($\lambda$).}
We are now ready to present finite-time error bounds for average-reward implicit TD($\lambda$) with the projection step, formally stated in Theorems \ref{THM:FIN_TIME_BOUND_CONST} and  \ref{THM:FIN_TIME_BOUND_DECR}. Results are provided for both constant and decreasing step-sizes. The bounds are expressed in terms of the negative-definiteness margin $\Delta$ from Lemma \ref{LEMMA:NEG_DEF}, the step-size parameters $c_\alpha$ (ratio: $\alpha_t/\beta_t$), $\beta_0$ (initial step-size), and $s$ (decay rate) and the mixing time of the underlying Markov process $\{S^\mu_t\}_{t \in \mathbb{N}}$ whose formal definition is given below.

\begin{definition}[Mixing Time]
Let $\{S_t\}_{t \in \mathbb{N}}\subset \mathcal{S}$ be a Markov chain with stationary distribution $\pi$. For $\epsilon \in (0,1)$, its mixing time is the smallest positive integer $\tau_\epsilon \in \mathbb{N}$ such that for all $t \geq \tau_\epsilon$, 
\begin{align*}
\sup_{s\in \mathcal{S}} d_{\mathrm{TV}}\left\lbrace\mathbb{P}(S_t = \cdot  \mid S_0 = s),\pi(\cdot) \right \rbrace
\le \epsilon,
\end{align*} 
where $d_{\mathrm{TV}}$ denotes the total variation distance.
\end{definition}

\begin{remark}
Under Assumption \ref{assumption:Markov}, the Markov chain $\{S^\mu_t\}_{t\in \mathbb{N}}$ is uniformly geometrically ergodic: there exist $m > 0$ and $\rho \in (0,1)$ such that 
$
\sup_{s\in \mathcal{S}} d_{\mathrm{TV}}\left\lbrace p^\mu(S^\mu_t  = \cdot \mid S^\mu_0 = s),\pi^\mu_{\cdot} \right \rbrace
\le m \rho^{t}.
$
Consequently, its mixing time $\tau_{\epsilon}$ is of order $\mathcal{O}\left(\log \frac{1}{\epsilon}\right)$.
\end{remark}

\begin{theorem}\label{THM:FIN_TIME_BOUND_CONST}
Suppose the Markov chain $\{S^\mu_t\}_{t \in \mathbb{N}}$ is uniformly geometrically ergodic with a rate parameter $\rho \in (0,1)$, the step-size ratio parameter is chosen to satisfy $c_\alpha \ge \Delta + \sqrt{\frac{1}{\Delta^2 (1-\lambda)^4} - \frac{1}{(1-\lambda)^2}}$ and the optimal parameter $\|\boldsymbol{\Theta}^*\| \le R_{\boldsymbol{\Theta}}$. With $\lambda \in [0,1)$ and constant step-size $\beta_t = \beta$, the iterates of the projected average-reward implicit TD($\lambda$) algorithm satisfy the following finite-time error bound
\begin{align*}
\mathbb{E}^\mu\left\{\left( \widehat{\omega}_{t+1} - \omega^\mu \right)^2 + \left\| \Pi_{\mathbb O}\left(\widehat{\pmb{\theta}}_{t+1} - \boldsymbol{\theta}^*\right) \right\|^2 \right\}&\lesssim (1-\beta\gamma\Delta)^{t+1}\left\{\left( \widehat{\omega}_0^{\mu} - \omega^\mu \right)^2 + \left\|\widehat{\pmb{\theta}}_0 - \boldsymbol{\theta}^* \right\|^2 \right\} \\ &\quad+ 
\mathcal{O}\left(\beta\tau_{\beta} + h^{\tau_{\beta}} + \beta \tau_{\beta} t h^t \right), \quad t \ge 0
\end{align*}
where
$
h = \max\{1-\beta\gamma\Delta, \rho, \lambda \} ~\text{and}~\gamma = \min\left\{\frac{1}{1+c_\alpha \beta}, \frac{(1-\lambda)^2}{(1-\lambda)^2 + \beta}\right\}.
$
\end{theorem}

\begin{theorem}\label{THM:FIN_TIME_BOUND_DECR}
Suppose the Markov chain $\{S^\mu_t\}_{t \in \mathbb{N}}$ is uniformly geometrically ergodic with a rate parameter $\rho \in (0,1)$, the step-size ratio parameter is chosen to satisfy $c_\alpha \ge \Delta + \sqrt{\frac{1}{\Delta^2 (1-\lambda)^4} - \frac{1}{(1-\lambda)^2}}$ and the optimal parameter $\|\boldsymbol{\Theta}^*\| \le R_{\boldsymbol{\Theta}}$. With $\lambda \in [0,1)$ and decreasing step-sizes $\beta_t = \frac{\beta_0}{(t+1)^s}, ~s \in (0,1)$, the iterates of the projected average-reward implicit TD($\lambda$) algorithm satisfy the following finite-time error bound 
\begin{align*}
\mathbb{E}^\mu\left\{\left( \widehat{\omega}_{t+1} - \omega^\mu \right)^2 + \left\| \Pi_{\mathbb O}\left(\widehat{\pmb{\theta}}_{t+1} - \boldsymbol{\theta}^*\right) \right\|^2 \right\} &\lesssim \exp\left[-\frac{\Delta \gamma_{0} \beta_0}{1-s}\{(1+t)^{1-s}-1\} \right] \left\{\left( \widehat{\omega}_0^{\mu} - \omega^\mu \right)^2 + \left\|\widehat{\pmb{\theta}}_0 - \boldsymbol{\theta}^* \right\|^2 \right\} \\
&\quad + \mathcal{O}\left\lbrace\tau_{\beta_t}t \exp\left(-ct^{1-s}\right) + \tau_{\beta_t} t^{-s} + q^{\tau_{\beta_t}}\right\rbrace,\quad t \ge 0
\end{align*}
for some constant $c > 0$,
$q = \max\{\rho, \lambda \} ~\text{and}~\gamma_{0} = \min\left\{\frac{1}{1+c_\alpha \beta_0}, \frac{(1-\lambda)^2}{(1-\lambda)^2 + \beta_0}\right\}.
$
\end{theorem}

\begin{remark}
Our two theorems substantially relax the restrictive conditions required in \citep{zhang2021finite}.
\begin{itemize}
    \item For constant step-sizes, we establish a finite-time bound without any initial step-size requirements.  
    By contrast, previous analysis ties the initial step-size to a problem dependent quantity (e.g., $\Delta\beta<2$) and further 
    imposes restrictive upper bounds on the step-size as well as the mixing time \citep{zhang2021finite}.
    \item For decaying step-sizes, our theorem accommodates the polynomial schedule of the form $\beta_t=\beta_0/(t+1)^s$ for any 
    $s\in(0,1)$, not just the $1/t$ rate covered in \citep{zhang2021finite}. In addition, the bounds in \citep{zhang2021finite} hold 
    only under extra restrictions, for example, there must exist an index $t^* \in \mathbb{N}$ with bounded cumulative step-size up to 
    $t^*$, and all subsequent step-sizes must stay below a problem-dependent threshold.
\end{itemize}
These relaxations remove delicate initial step-size conditions and broaden the range of admissible step-size schedules, while still 
providing finite-time error guarantees. 
\end{remark}

\section{Numerical Experiments}\label{SEC:NUMERICS}
In this section, we demonstrate the effectiveness of the proposed average-reward implicit TD($\lambda$) relative to 
 standard average-reward TD($\lambda$) on both policy evaluation and control tasks. All experiments were carried out on Intel(R)
  Xeon(R) Gold 6152 CPUs at 2.10 GHz with 32 GB RAM.

\subsection{Policy Evaluation}
For policy evaluation, we use MRP and the Boyan chain examples; for policy learning, we consider the 
access-control and pendulum problems. Performance is quantified by the loss function given by
$
\bigl(\widehat{\omega}_t-
\omega^\mu\bigr)^2+\bigl\|\Pi_{\mathbb O}(\widehat{\pmb{\theta}}_t-\boldsymbol{\theta}^\ast)\bigr\|^2.
$
A detailed description of how the loss value is computed is provided in the supplementary materials. We 
consider both constant and decaying step-size schedules. For decaying schedules, the step-size is held 
fixed for the first 150 iterations to promote exploration and then decreased thereafter. Each 
configuration is run for $T=2000$ steps with $\widehat{\omega}_0^{\mu}=0$ and 
$\widehat{\pmb{\theta}}_0\sim \mathrm{Unif}([-1,1]^d)$, and results are averaged over 50 independent 
trials. We fix the step-size ratio $c_\alpha =1$ and the exponential weighting parameter $\lambda = 
0.25$. We compare four methods: (i) average-reward standard TD($\lambda$); (ii) average-reward implicit 
TD($\lambda$) without projection; and (iii–iv) average-reward implicit TD($\lambda$) with projection, 
using projection radius $R_{\boldsymbol{\Theta}}\in\{1000,\,5000\}$. Full implementation details and 
additional experimental results are provided in the supplementary materials.

\subsubsection{Markov Reward Process}
Here we study an MRP with $|\mathcal S|=100$ states; the transition matrix and reward vector are 
generated following \citep{zhang2021finite}. We first consider constant step-sizes $\beta_t\equiv\beta_0$ 
with $\beta_0\in\{0.1,\ldots,3.0\}$. Figure \ref{fig:MRP_constant_presentation} summarizes the results 
(solid line = mean, shaded band = 95$\%$ confidence interval). As shown in the left panel, the average 
loss across 50 independent experiments increases for all methods as $\beta_0$ becomes larger, around 
$\beta_0\approx 0.3$ or more. However, as $\beta_0\to 2$, standard average-reward  TD($\lambda$) becomes 
unstable and its loss explodes, whereas average-reward implicit TD($\lambda$) remains numerically stable 
with modest loss growth. To compare all four algorithms' behavior at a moderate step-size, the right 
panel fixes a step-size value $\beta_0 = 1.0$ and tracks performance over iterations: average-reward 
implicit TD($\lambda$) maintains relatively low error throughout, while standard average-reward 
TD($\lambda$) incurs substantially larger error over the horizon.

\begin{figure}[htp]
\centering
\caption{\small MRP experiment under constant step-size, with exponential weighting parameter and 
step-size ratio set to $(\lambda, c_{\alpha}) = (0.25, 1.0)$. The solid line represents the mean, and the 
shaded region denotes the 95\% confidence interval. (Left) Loss value from step-size 0.1 to 3.0. (Right) 
Loss value over iterations with initial step-size $\beta_0 = 1.0$.
}
\label{fig:MRP_constant_presentation}
    \centering
    \begin{subfigure}[t]{0.4\linewidth}
\includegraphics[width=\linewidth]{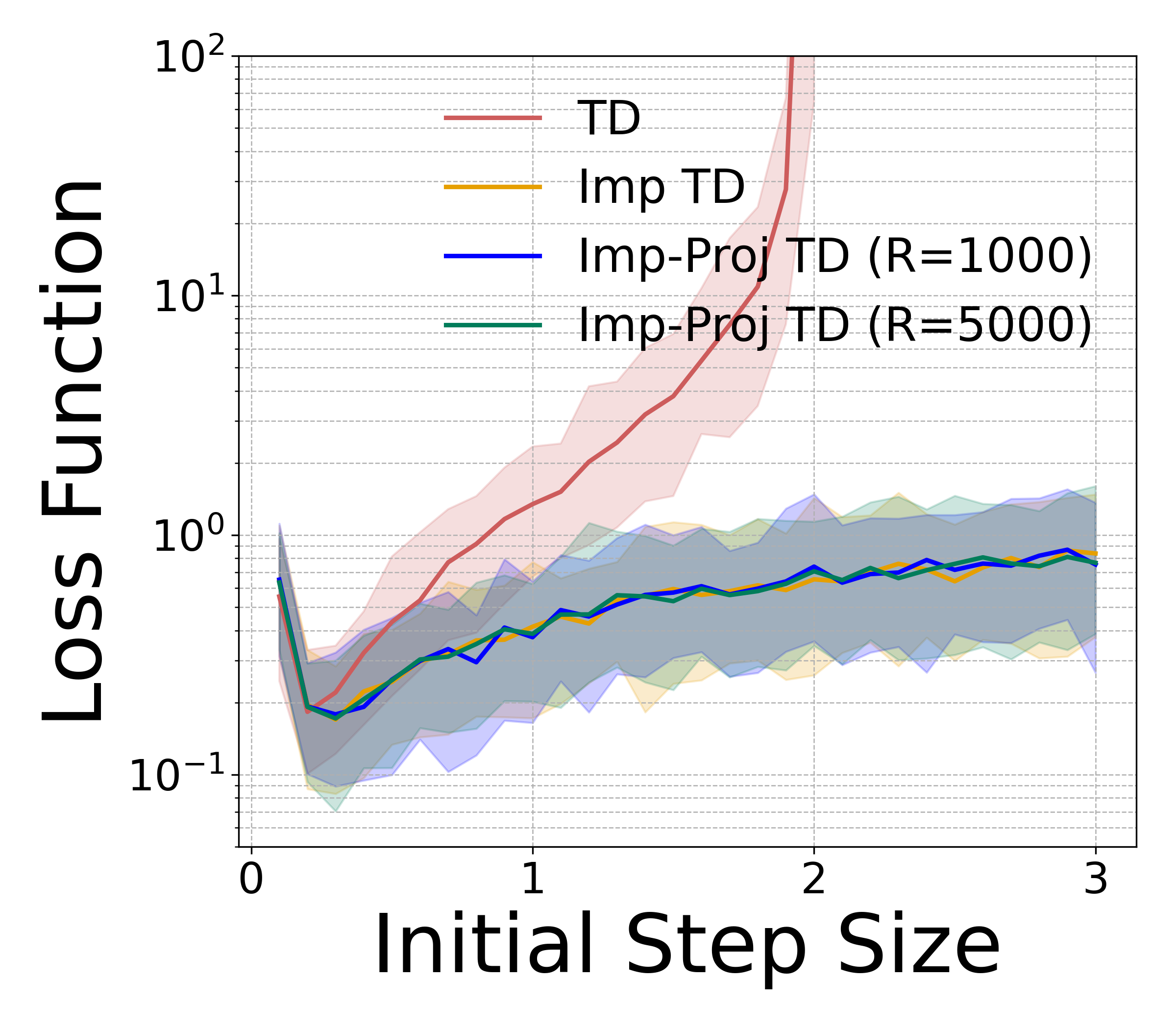}
    \end{subfigure}
\begin{subfigure}[t]{0.4\linewidth}\includegraphics[width=\linewidth]
{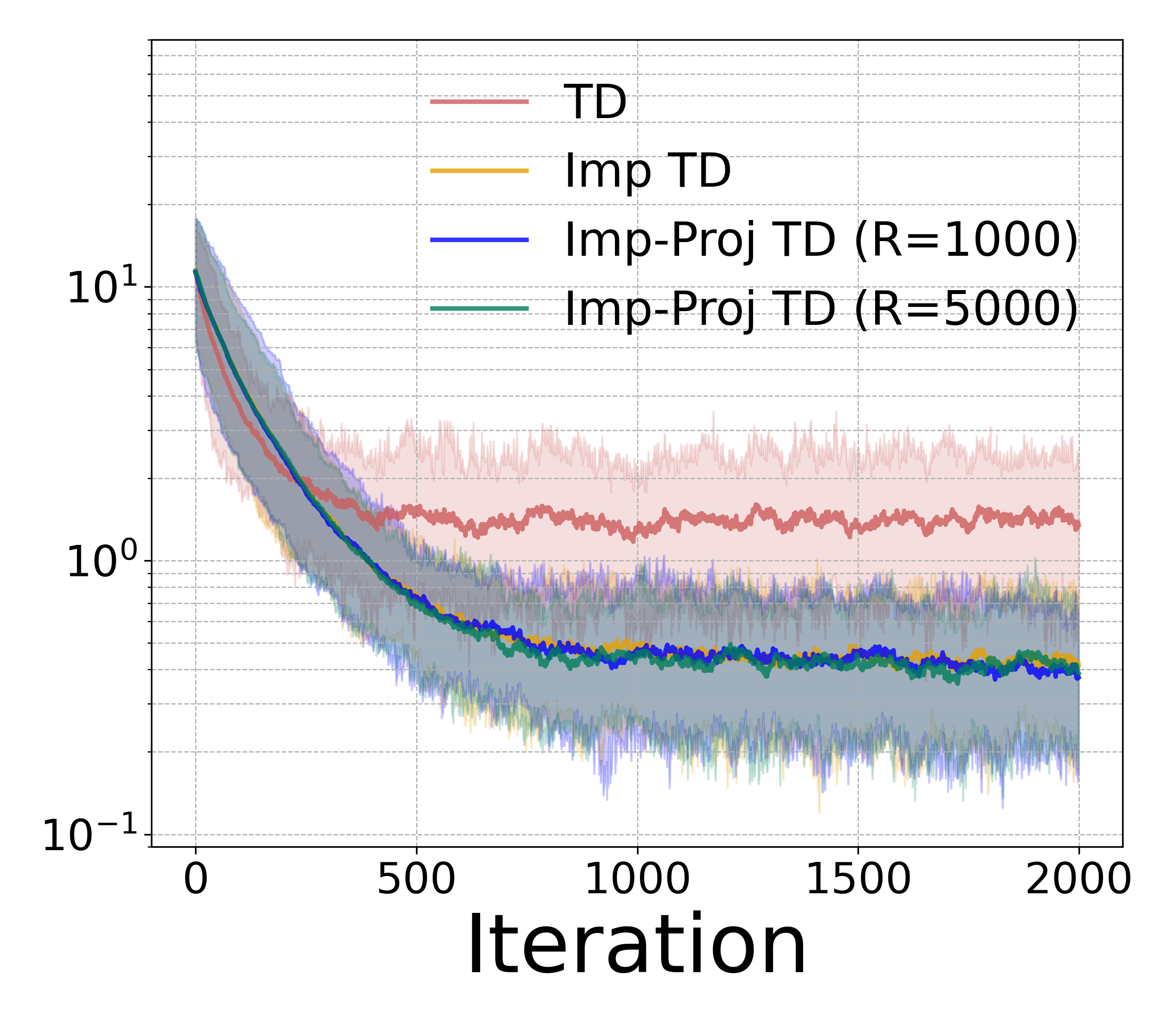}\end{subfigure}
\end{figure}

\subsubsection{Average-Reward Boyan Chain}
We next study the average-reward Boyan chain under deterministic policies. As in the MRP setting, the 
Boyan chain is a standard benchmark for TD learning \citep{boyan2002technical}. Because the original 
formulation is not average-reward, we use the variant proposed by \citep{zhang2021average}, which has 13 
states and 2 actions. In each experiment, we construct a deterministic policy by assigning an action to 
each state via independent $\mathrm{Bernoulli}(0.5)$ draws. As with the MRP experiments, we assess 
performance using the average loss across 50 independent runs. Figure \ref{fig:Boyan_selected_stepsizes} 
shows results on the Boyan chain example under the decaying step-size schedule $\beta_t = 
\beta_0/(t+1)^{0.99}$. In the left panel, average-reward implicit TD($\lambda$) methods remain stable 
across $\beta_0 \in [0.1, 3.0]$, whereas the standard average-reward TD($\lambda$) becomes unstable and 
its loss grows rapidly as $\beta_0$ approaches $1.5$. The right panel displays learning curves for 
$\beta_t = 1.5/(t+1)^{0.99}$ over 2000 iterations. The loss of standard average-reward TD($\lambda$) 
method consistently exceeds that of the average-reward implicit TD($\lambda$) methods, highlighting the 
latter's improved numerical stability and superior performance.

\begin{figure}[htp]
\centering
\caption{\small Boyan experiment with exponential weighting parameter and step-size ratio set to \( (\lambda, c_{\alpha}) = (0.25, 1.0) \) under decaying step-size schedule \( \beta_t = \beta_0 / (t+1)^{0.99} \). Solid lines denote the mean, and shaded regions represent 95\% confidence intervals. (Left) Loss value with initial step-sizes ranging from 0.1 to 3.0.
(Right) Loss value over iterations with initial step-size $\beta_0 = 1.5$. }
\label{fig:Boyan_selected_stepsizes}
    \centering
    \begin{subfigure}[t]{0.4\linewidth}
\includegraphics[width=\linewidth]{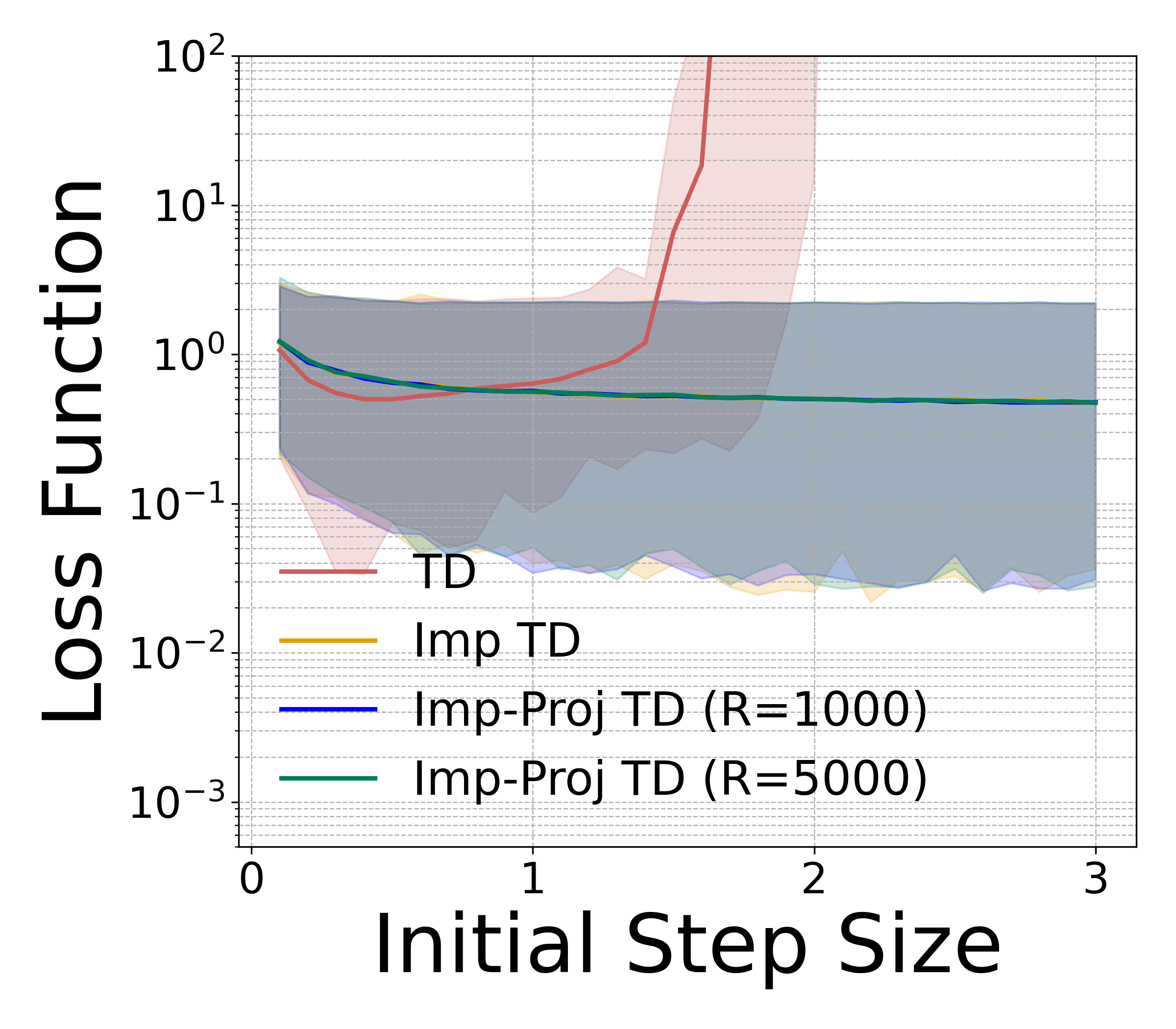}
    \end{subfigure}
\begin{subfigure}[t]{0.4\linewidth}\includegraphics[width=\linewidth]{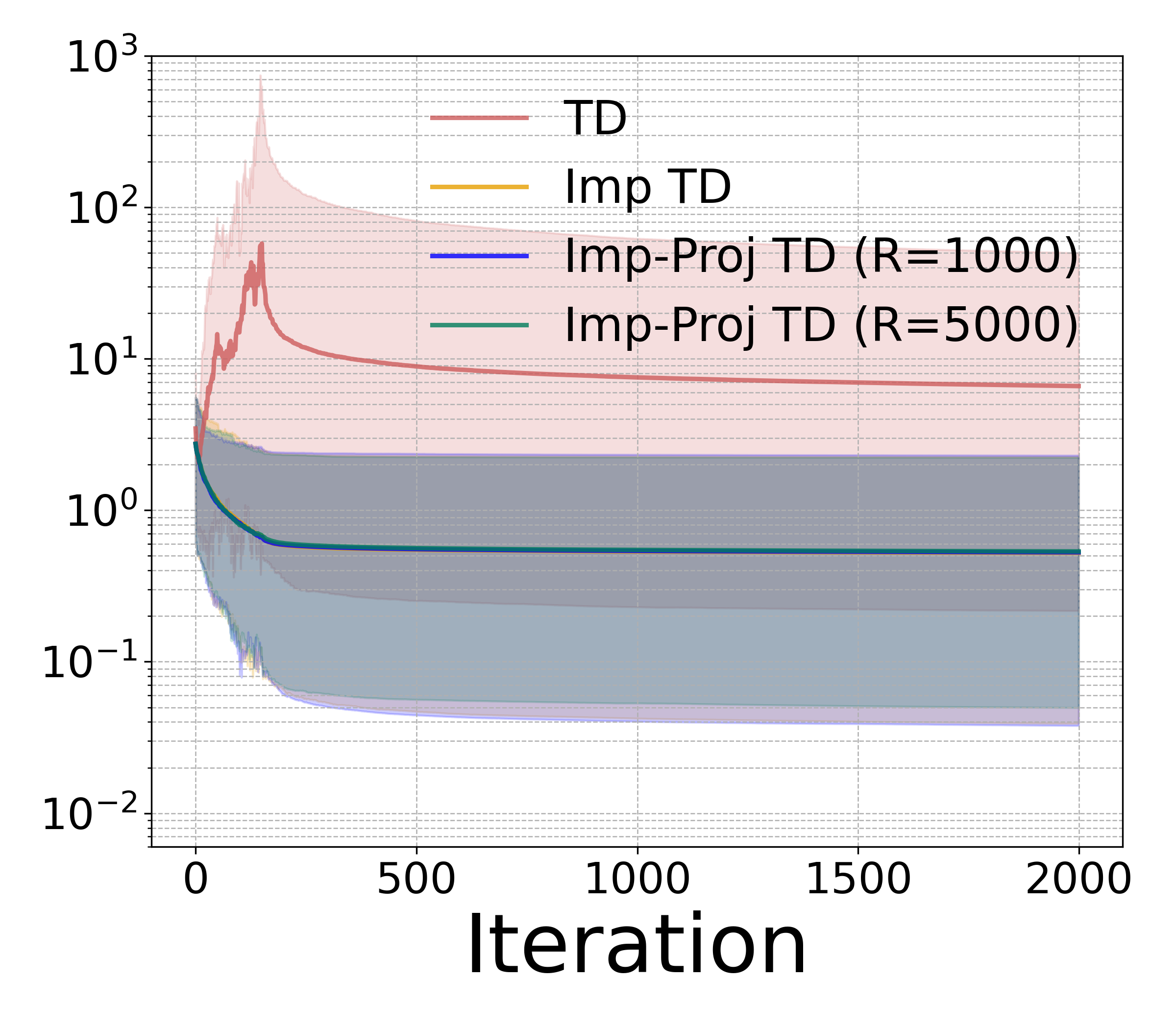}\end{subfigure}
\end{figure}

\subsection{Control Experiments}
In this section, we utilize the proposed average-reward implicit TD($\lambda$) on control tasks. We use state–action–reward–state–action (SARSA) with linear function approximation to estimate the action-value function. Each experiment comprises $T=15000$ time steps and is repeated over 30 independent runs. We employ a decaying step-size schedule $\beta_t=\beta_0/(t+400)^{0.99}$, holding $\beta_t$ constant for the first 150 iterations to encourage early exploration before gradually reducing it thereafter.

\subsubsection{Access-Control Queuing}
We study the canonical access-control queuing problem \citep{barto2021reinforcement} in the average-reward setting. At each decision epoch, an arriving customer belongs to one of four equiprobable classes, and the agent chooses whether to admit or reject the customer. There are ten identical servers; if admitted, the customer yields an immediate reward and occupies a server. Service completion occurs independently at each step with a fixed probability, inducing stochastic transitions in server availability. The goal is to learn an admission policy that optimally maps the current customer class and the number of available servers to an admit/reject decision. We illustrate average-reward learning results in the left panel of Figure~\ref{fig:combined_result_presentation}. The average-reward implicit TD($\lambda$) methods consistently outperform the average-reward TD($\lambda$) method across varying initial step-sizes in terms of average-reward. As the initial step-size increases, the implicit methods show a mild performance gain and remain stable whereas standard version deteriorates and fails to benefit from larger steps.

\begin{figure}[htp]
\centering
\caption{\small Control experiment with exponential weighting parameter and step-size ratio parameter \( (\lambda, c_{\alpha}) = (0.25, 1.0) \), under the decaying step-size schedule \( \beta_t = \beta_0 / (t+400)^{0.99} \). Initial step-size ranges from 0.25 to 1.5. Solid lines denote the mean, and shaded regions represent 95\% confidence intervals.}
\label{fig:combined_result_presentation}
    \centering
    \begin{subfigure}[t]{0.8\linewidth}
    \includegraphics[width=\linewidth]{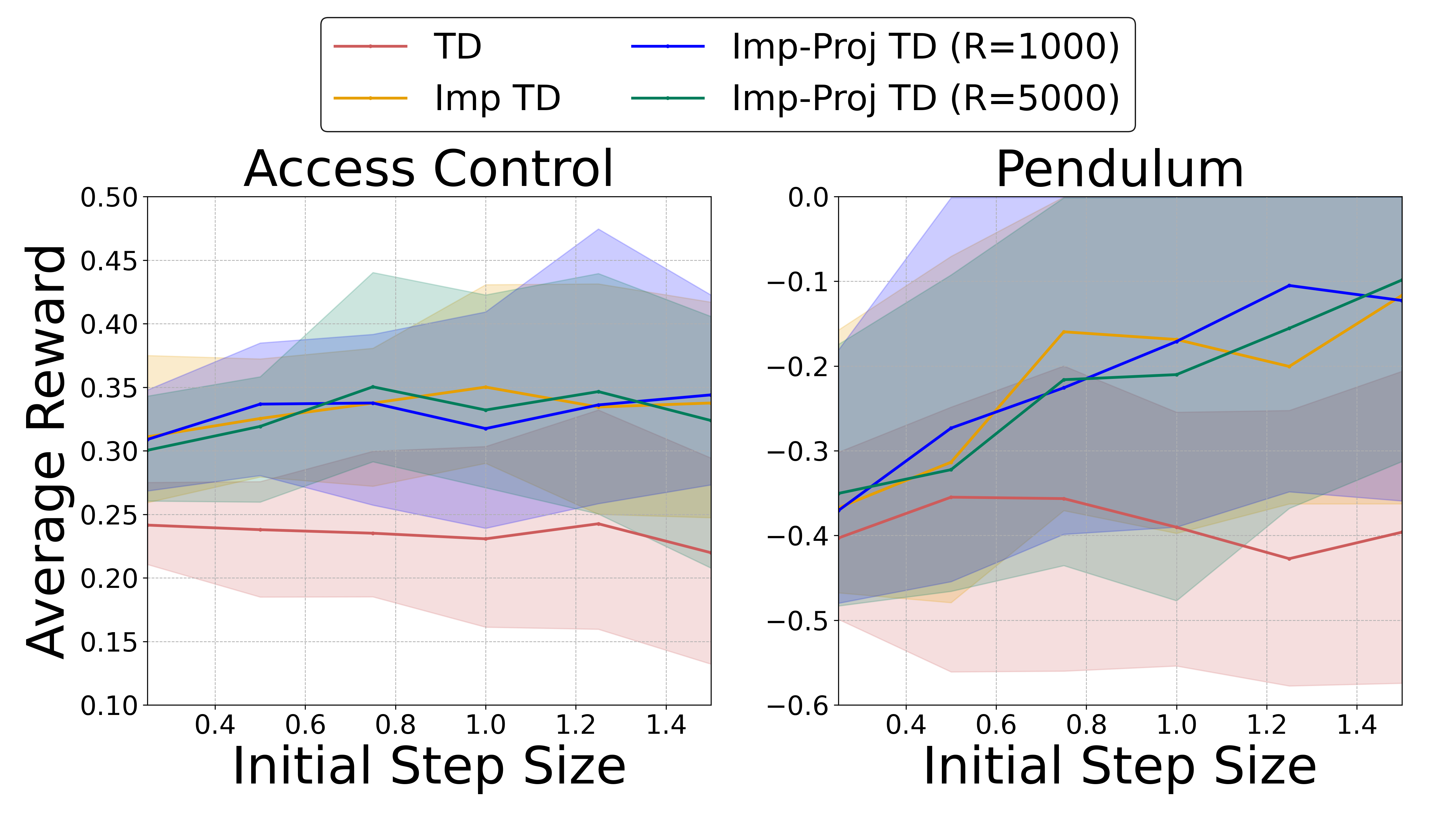}
    \end{subfigure}
\end{figure}

\subsubsection{Pendulum Environment}
We also apply the proposed average-reward implicit TD($\lambda$) to the \texttt{Pendulum-v1} environment. Because the environment is defined with episodic terminations \citep{towers2024gymnasium}, we modify it to match the infinite-horizon average-reward setting. The control objective is to keep the pendulum upright. The right panel of Figure~\ref{fig:combined_result_presentation} shows results for the pendulum environment with $(\lambda, c_\alpha) = (0.25, 1.0)$.
Mirroring the access-control task, larger initial step-sizes benefit the average-reward implicit TD($\lambda$) methods, which achieve higher average reward and remain stable, whereas standard average-reward TD($\lambda$) fails to benefit from larger step-sizes and exhibits no improvement.

\section{Conclusion}
We introduced average-reward implicit TD($\lambda$), a fixed-point variant of the average reward TD($\lambda$) that preserves per iteration complexity while markedly improving stability to step-size choices. Our theoretical guarantees provide explicit finite-time error bounds under both constant and decaying step-sizes, established via a projection-based analysis. Across policy evaluation and control examples, the implicit updates deliver robust performance over a wide range of step-sizes, demonstrating strong practical performance consistent with the theory. Looking ahead, promising directions include a full theoretical analysis of average-reward implicit SARSA and rigorous two-time-scale extensions of both standard and implicit average-reward TD($\lambda$). A related direction is to build implicit TD methods to estimate the asymptotic variance of cumulative reward in the average-reward regime.

\newpage
\appendix
\section{List of notations}
\noindent To provide rigorous details behind the established theoretical results, we provide a summary of the necessary assumptions, notations, and facts.
\begin{itemize}
    \item $\mathcal{S}= \{1, \cdots, |\mathcal{S}|\}$: state space 
    \item $\{S^\mu_t\}_{t \in \mathbb{N}}$: a sequence of states under policy $\mu$
    \item $(A^\mu_t)_{t \in \mathbb{N}}$: a sequence of actions under policy $\mu$ with $A^\mu_t := \mu(S^\mu_t)$
    \item $(R^\mu_t)_{t\in \mathbb{N}}$: a sequence of rewards under policy $\mu$ with $R^\mu_t=r(S^\mu_t, A^\mu_t)$
    \item
    $
    p^\mu\left(S^\mu_{t+1}| S^\mu_t\right):=p\left\lbrace S^\mu_{t+1} | S^\mu_t, A^\mu_t=\mu(S^\mu_t)\right\rbrace
    $: transition probabilities under policy $\mu$
    \item  
    $
    \boldsymbol{P}^\mu = \bigl[P^\mu_{ij}\bigr]_{i,j=1}^{|\mathcal{S}|}
    $: time-homogeneous transition probability matrix with 
    $
    P^\mu_{ij} = p^\mu \left\lbrace S^\mu_{t+1}=j\mid S^\mu_{t}=i\right\rbrace
    $
    \item
    $
    \boldsymbol{\pi}^{\mu} = (\pi^{\mu}_i)_{i=1}^{|\mathcal{S}|}
    $: a unique stationary distribution of $\{S^\mu_t\}_{t \in \mathbb{N}}$ 
    \item
        $
        \boldsymbol{r}^\mu = 
        \left[r\{1, \mu(1)\}, \cdots, r\{|\mathcal{S}|, \mu(|\mathcal{S}|)\}\right]^\top
        $: a reward vector under policy $\mu$
    \item $\mathbb{E}^\mu$: expectation with respect to the Markov chain $\{S^\mu_t\}_{t\in\mathbb{N}}$ with a fixed initial state $S^\mu_0$
    \item $\mathbb{E}^{\boldsymbol{\pi}^\mu}$: expectation with respect to the stationary distribution of the Markov chain $\{S^\mu_t\}_{t\in\mathbb{N}}$
    \item $\|\cdot\|$: Euclidean norm for vectors and the induced operator norm for matrices
    \item $\lesssim$: inequality up to a constant
    \item $\boldsymbol{X}_t := \left(S^\mu_t ,S^\mu_{t+1}, \boldsymbol{z}_t\right), \quad \boldsymbol{X}_{t-\tau:t} = (S^\mu_t, S^\mu_{t+1}, \boldsymbol{z}_{t-\tau:t})$ \item $\boldsymbol{\phi}_l = \boldsymbol{\phi}(S^\mu_l), \quad \|\boldsymbol{\phi}_l \| \le 1, \quad \boldsymbol{\Phi}\in\mathbb{R}^{|\mathcal S|\times d}:$ \text{feature 
matrix whose $i^{\text{th}}$ row is} $\boldsymbol{\phi}(i)^\top$ 
    \item $\boldsymbol{z}_t = \sum_{l=0}^t \lambda^{t-l}\boldsymbol{\phi}_l, \quad \boldsymbol{z}_{t-\tau:t} = \sum_{l=t-\tau}^t \lambda^{t-l} \boldsymbol{\phi}_l, \quad \omega^{\mu} = \lim_{T \to \infty} \frac{1}{T} \mathbb{E}^{\mu}\left(\sum_{t=0}^{T-1} R^\mu_t\right) = {\boldsymbol{\pi}^{\mu}}^\top \boldsymbol{r}^{\mu}$
    \item $\Pi_{\mathbb O}$: projection operator onto $\mathbb O := 
\mathbb{S}_{\boldsymbol{\Phi},\boldsymbol e}^\perp$ where
$\mathbb{S}_{\boldsymbol{\Phi},\boldsymbol e}:=\operatorname{span}\{\boldsymbol{\theta}: \boldsymbol{\Phi}\boldsymbol{\theta}=\boldsymbol e\}, \quad \boldsymbol{\Pi}:= \begin{bmatrix}
            1 & 0 \\
            0 & \Pi_{\mathbb{O}}
        \end{bmatrix}$
    \item $\boldsymbol{A}_t = \boldsymbol{A}(\boldsymbol{X}_t) := \begin{bmatrix}
    - c_{\alpha} & 0 \\
    - \boldsymbol{z}_t  &{\boldsymbol{z}_t}(\boldsymbol{\phi}_{t+1}^{\top}-\boldsymbol{\phi}_t^{\top}) 
    \end{bmatrix}, \quad \boldsymbol{A}:= \mathbb{E}^{\pi^\mu} \boldsymbol{A}_t$
    \item $\boldsymbol{A}_{t-\tau:t} = 
   \boldsymbol{A}(\boldsymbol{X}_{t-\tau:t}) := \begin{bmatrix}
    - c_{\alpha} & 0 \\
    - \boldsymbol{z}_{t-\tau:t}  &{\boldsymbol{z}_{t-\tau:t}}(\boldsymbol{\phi}_{t+1}^{\top}-\boldsymbol{\phi}_{t}^{\top}) 
    \end{bmatrix}$ 
    \item $\boldsymbol{b}_t = \boldsymbol{\boldsymbol{b}(\boldsymbol{X}_t)} := \begin{bmatrix}
        c_\alpha R^\mu_t \\
        R^\mu_t \boldsymbol{z}_t
    \end{bmatrix}, \quad \boldsymbol{b} = \mathbb{E}^{\boldsymbol{\pi}^\mu} \boldsymbol{b}_t, \quad \boldsymbol{b}_{t-\tau:t} = \boldsymbol{b}(\boldsymbol{X}_{t-\tau:t}):= \begin{bmatrix}
        c_\alpha R^\mu_t \\
        R^\mu_t \boldsymbol{z}_{t-\tau:t}
    \end{bmatrix}$
    \item $\widehat{\boldsymbol{\Theta}}_t := \begin{bmatrix}
        \widehat{\omega}_t\\
        \widehat{\pmb{\theta}}_t
        \end{bmatrix}, \quad \boldsymbol{\Theta}^* = \begin{bmatrix} \omega^{\mu}\\ \boldsymbol{\theta}^*\end{bmatrix}\in \mathbb{R}\times \mathbb{O}$ is the unique element satisfying $\boldsymbol{A}\boldsymbol{\Theta}^* = \boldsymbol{b}$.
    \item $\boldsymbol{D}_t := \begin{bmatrix}
                \frac{1}{1+c_{\alpha}\beta_t} & 0 \\
                0 & \frac{1}{1+\beta_t ||\boldsymbol{z}_t||^2} \boldsymbol{I}_{d}
            \end{bmatrix}, \quad \underline{\boldsymbol{D}_t} := \gamma_t \boldsymbol{I}_{d+1}, \quad \gamma_t = \min\left(\frac{1}{1+c_\alpha \beta_t},\frac{(1-\lambda)^2}{(1-\lambda)^2+\beta_t}\right)$
    \item $\zeta_t(\boldsymbol{\Theta}, \boldsymbol{X}_t) := \left\langle \left( \boldsymbol{A}(\boldsymbol{X}_t)-\boldsymbol{A}\right)\boldsymbol{\Theta^*}, \underline{\boldsymbol{D}_t}(\boldsymbol{\Theta^*}-\boldsymbol{\Theta})\right\rangle + \left\langle \boldsymbol{b}(\boldsymbol{X}_t)-\boldsymbol{b}, \underline{\boldsymbol{D}_t}(\boldsymbol{\Theta^*}-\boldsymbol{\Theta})\right\rangle$
    \item $\xi_t(\boldsymbol{\Theta}, \boldsymbol{X}_t) := (\boldsymbol{\Theta^*}-\boldsymbol{\Theta})^\top \left(\boldsymbol{A}(\boldsymbol{X}_t)^\top-\boldsymbol{A}^\top \right)\underline{\boldsymbol{D}_t}  \left( \boldsymbol{\Theta^*}-\boldsymbol{\Theta}\right)$
    \item $\boldsymbol{M} = \text{diag}\left(\pi^\mu_1, \cdots, \pi^\mu_{|\mathcal{S}|}\right), \quad
\Delta := \min_{\|\boldsymbol{\theta}\| = 1, \boldsymbol{\theta} \in \mathbb{O}} \boldsymbol{\theta}^\top \boldsymbol{\boldsymbol{\Phi}}^\top \boldsymbol{M} \left\lbrace \boldsymbol{I} -  (1 - \lambda) \sum_{m=0}^{\infty} \lambda^m (\boldsymbol{P}^\mu)^{m+1}  \right\rbrace \boldsymbol{\boldsymbol{\Phi}} \boldsymbol{\theta}$
\end{itemize}

\section{Theoretical results}
\begin{lemma}
    The implicit update rule is 
    \begin{align*}
        \widehat{\pmb{\theta}}_{t+1}&= \widehat{\pmb{\theta}}_t  + \frac{\beta_t}{1 + \beta_t \|\boldsymbol{z}_t\|^2} \left( R^\mu_t - \widehat{\omega}_t + \boldsymbol{\phi}_{t+1}^\top \widehat{\pmb{\theta}}_t  - \boldsymbol{\phi}_t^\top \widehat{\pmb{\theta}}_t \right) \boldsymbol{z}_t \\
        \widehat{\omega}_{t+1}&=\widehat{\omega}_t + \frac{c_\alpha \beta_t}{1+c_\alpha \beta_t}( R^\mu_t -\widehat{\omega}_t)
    \end{align*}
\end{lemma}
\begin{proof}
We first revisit the recursion formula in \eqref{eqn:param_update_im2} and \eqref{eqn:param_update_im1} : 
\begin{align*}
\widehat{\pmb{\theta}}_{t+1}
&= \widehat{\pmb{\theta}}_t + \beta_t (R^\mu_t - \widehat{\omega}_t +  \boldsymbol{\phi}_{t+1}^\top \widehat{\pmb{\theta}}_t  + \lambda \boldsymbol{z}_{t-1}^\top \widehat{\pmb{\theta}}_t -  \boldsymbol{z}_t^\top \textcolor{red}{\widehat{\pmb{\theta}}_{t+1}}) \boldsymbol{z}_t\\ 
\widehat{\omega}_{t+1} &= \widehat{\omega}_t + c_\alpha \beta_t (R^\mu_t - \textcolor{red}{\widehat{\omega}_{t+1}})
\end{align*}

\noindent Our goal is to derive the update rule for $\widehat{\pmb{\theta}}$ and $\widehat{\omega}$, which can be done by combining $\widehat{\pmb{\theta}}_{t+1}, \widehat{\omega}_{t+1}$. 

\paragraph{(1) Update Rule for $\widehat{\pmb{\theta}}$}
Let us first examine the update rule for the parameter $\widehat{\pmb{\theta}}$. Combining the rightmost term to the left, we have 
\begin{equation}
\begin{split}
\left(\boldsymbol{I}+\beta_t \boldsymbol{z}_t \boldsymbol{z}_t^\top\right)\widehat{\pmb{\theta}}_{t+1} 
&= \widehat{\pmb{\theta}}_t + \beta_t \left(R^\mu_t - \widehat{\omega}_t +  \boldsymbol{\phi}_{t+1}^\top \widehat{\pmb{\theta}}_t + \lambda \boldsymbol{z}_{t-1}^\top \widehat{\pmb{\theta}}_t\right) \boldsymbol{z}_t.
\end{split}
\end{equation}
From the Woodbury matrix identity $\left(\boldsymbol{I}+\beta_t \boldsymbol{z}_t \boldsymbol{z}_t^\top\right)^{-1}=\boldsymbol{I}-\frac{\beta_t}{1+\beta_t \|\boldsymbol{z}_t\|^2} \boldsymbol{z}_t \boldsymbol{z}_t^\top$, we have
\begin{align*}
\widehat{\pmb{\theta}}_{t+1}
&= \widehat{\pmb{\theta}}_t + \beta_t \left(R^\mu_t - \widehat{\omega}_t + \boldsymbol{\phi}_{t+1}^\top \widehat{\pmb{\theta}}_t + \lambda \boldsymbol{z}_{t-1}^\top \widehat{\pmb{\theta}}_t\right) \boldsymbol{z}_t  - \frac{\beta_t \boldsymbol{z}_t^\top \widehat{\pmb{\theta}}_t}{1 + \beta_t \|\boldsymbol{z}_t\|^2} \boldsymbol{z}_t 
- \beta_t \frac{\beta_t R^\mu_t \|\boldsymbol{z}_t\|^2 }{1 + \beta_t \|\boldsymbol{z}_t\|^2}  \boldsymbol{z}_t 
+ \beta_t \frac{\beta_t\widehat{\omega}_t \|\boldsymbol{z}_t\|^2 }{1 + \beta_t \|\boldsymbol{z}_t\|^2} \boldsymbol{z}_t \\
&\quad - \beta_t  \frac{\beta_t \boldsymbol{\phi}_{t+1}^\top \widehat{\pmb{\theta}}_t \|\boldsymbol{z}_t\|^2}{1 + \beta_t \|\boldsymbol{z}_t\|^2} \boldsymbol{z}_t 
- \beta_t  \frac{\beta_t\lambda \boldsymbol{z}_{t-1}^\top \widehat{\pmb{\theta}}_t \|\boldsymbol{z}_t\|^2 }{1 + \beta_t \|\boldsymbol{z}_t\|^2} \boldsymbol{z}_t \\
&= \widehat{\pmb{\theta}}_t + \beta_t R^\mu_t \left(1 - \frac{\beta_t \|\boldsymbol{z}_t\|^2}{1 + \beta_t \|\boldsymbol{z}_t\|^2}  \right) \boldsymbol{z}_t 
- \beta_t \widehat{\omega}_t \left(1 - \frac{\beta_t \|\boldsymbol{z}_t\|^2 }{1 + \beta_t \|\boldsymbol{z}_t\|^2} \right) \boldsymbol{z}_t \\
&\quad + \beta_t \boldsymbol{\phi}_{t+1}^\top \widehat{\pmb{\theta}}_t \left(1 - \frac{\beta_t \|\boldsymbol{z}_t\|^2 }{1 + \beta_t \|\boldsymbol{z}_t\|^2} \right) \boldsymbol{z}_t 
+ \beta_t \lambda \boldsymbol{z}_{t-1}^\top \widehat{\pmb{\theta}}_t \left(1 - \frac{\beta_t\|\boldsymbol{z}_t\|^2}{1 + \beta_t \|\boldsymbol{z}_t\|^2}  \right) \boldsymbol{z}_t 
- \frac{\beta_t \boldsymbol{z}_t^\top \widehat{\pmb{\theta}}_t}{1 + \beta_t \|\boldsymbol{z}_t\|^2} \boldsymbol{z}_t \\
&= \widehat{\pmb{\theta}}_t + \frac{\beta_t}{1 + \beta_t \|\boldsymbol{z}_t\|^2} \left( R^\mu_t - \widehat{\omega}_t + \boldsymbol{\phi}_{t+1}^\top \widehat{\pmb{\theta}}_t + \lambda \boldsymbol{z}_{t-1}^\top \widehat{\pmb{\theta}}_t - \boldsymbol{z}_t^\top \widehat{\pmb{\theta}}_t \right) \boldsymbol{z}_t \\
&= \widehat{\pmb{\theta}}_t + \frac{\beta_t}{1 + \beta_t \|\boldsymbol{z}_t\|^2} \left( R^\mu_t - \widehat{\omega}_t + \boldsymbol{\phi}_{t+1}^\top \widehat{\pmb{\theta}}_t - \boldsymbol{\phi}_t^\top \widehat{\pmb{\theta}}_t \right) \boldsymbol{z}_t 
\end{align*}

\paragraph{(2) Update Rule for $\widehat{\omega}$}
Similarly, for the update rule for the $\widehat{\omega}$, we combine the term to the left hand side, which yields 

\begin{equation*}
    (1+c_\alpha \beta_t)\widehat{\omega}_{t+1} = \widehat{\omega}_t + c_\alpha \beta_t R^\mu_t.
\end{equation*}

\noindent Now, dividing with $(1+c_\alpha \beta_t)$ we have

\begin{equation*}
\widehat{\omega}_{t+1} = \frac{1}{1+c_\alpha \beta_t}\widehat{\omega}_t + \frac{c_\alpha \beta_t}{1+c_\alpha \beta_t} R^\mu_t =\widehat{\omega}_t + \frac{c_\alpha \beta_t}{1+c_\alpha \beta_t}( R^\mu_t -\widehat{\omega}_t).
\end{equation*}
This completes the whole update rules. 
\end{proof}

\subsection{Proof of Main Theorems}
\noindent  In this section, we provide a proof of the finite-time error bounds of the projected average-reward projected TD($\lambda$) algorithm. To this end, recall that the implicit average-reward TD($\lambda$) update rule is given by 
\begin{align*}
    \widehat{\omega}_{t+1}&=\widehat{\omega}_t + \frac{c_\alpha \beta_t}{1+c_\alpha \beta_t}\left( R^\mu_t -\widehat{\omega}_t\right)\\
    \widehat{\pmb{\theta}}_{t+1}&= \widehat{\pmb{\theta}}_t + \frac{\beta_t}{1 + \beta_t \|\boldsymbol{z}_t\|^2} \left( R^\mu_t - \widehat{\omega}_t + \boldsymbol{\phi}_{t+1}^\top \widehat{\pmb{\theta}}_t - \boldsymbol{\phi}_t^\top \widehat{\pmb{\theta}}_t \right) \boldsymbol{z}_t .
\end{align*}
To gain numerical stability as well as to facilitate theoretical analysis, we impose additional projection steps to both the primary iterate $\widehat{\pmb{\theta}}_t$ and the reward iterate $\widehat{\omega}_t$. Recall that the projection operator with a radius $R>0$ is given by $$\Pi_R(\boldsymbol{u}) = \begin{cases} \boldsymbol{u}, \quad\quad \text{if}~ \|\boldsymbol{u}\| \le R\\ \frac{R}{\|\boldsymbol{u}\|}\boldsymbol{u}  \quad \text{otherwise}.\end{cases}$$ Then the projected average-reward TD($\lambda$) update is as follows
\begin{align*}
    \widehat{\omega}_{t+1}&=\Pi_{R_\omega} \left\lbrace\widehat{\omega}_t + \frac{c_\alpha \beta_t}{1+c_\alpha \beta_t}\left( R^\mu_t -\widehat{\omega}_t\right)\right\rbrace,\\
     \widehat{\pmb{\theta}}_{t+1}&= \Pi_{R_{\boldsymbol{\theta}}}\left\lbrace \widehat{\pmb{\theta}}_t + \frac{\beta_t}{1 + \beta_t \|\boldsymbol{z}_t\|^2} \left( R^\mu_t - \widehat{\omega}_t + \boldsymbol{\phi}_{t+1}^\top \widehat{\pmb{\theta}}_t - \boldsymbol{\phi}_t^\top \widehat{\pmb{\theta}}_t  \right) \boldsymbol{z}_t\right\rbrace ,
\end{align*}
where $R_\omega > 0$ and $R_{\boldsymbol{\theta}} > 0$  are large enough such that $R_\omega \ge \|\omega^{\mu}\|$ and $R_{\boldsymbol{\theta}} \ge \|\boldsymbol{\theta^{*}}\|$. Following the analysis of the average-reward TD($\lambda$) in \citep{zhang2021finite}, we consider the following auxiliary iterates
\begin{align}
    \widehat{\omega}_{t+1}&=\Pi_{R_\omega} \left\lbrace\widehat{\omega}_t + \frac{c_\alpha \beta_t}{1+c_\alpha \beta_t}\left( R^\mu_t -\widehat{\omega}_t\right)\right\rbrace, \label{APPEND_IMP_REWARD_UPDATE}\\
    \widehat{\pmb{\theta}}_{t+1}&= \Pi_{\mathbb{O}}\left[\Pi_{R_{\boldsymbol{\theta}}} \left\lbrace \widehat{\pmb{\theta}}_t + \frac{\beta_t}{1 + \beta_t \|\boldsymbol{z}_t\|^2} \left( R^\mu_t - \widehat{\omega}_t + \boldsymbol{\phi}_{t+1}^\top \widehat{\pmb{\theta}}_t - \boldsymbol{\phi}_t^\top \widehat{\pmb{\theta}}_t \right) \boldsymbol{z}_t\right\rbrace\right] \label{APPEND_IMP_PRIMARY_UPDATE},
\end{align}
to facilitate finite-time error analysis. Here, the space $\mathbb O$ is the orthogonal complement of the space generated by $\boldsymbol{\theta_e}$ where $\boldsymbol{\Phi}\boldsymbol{\theta_e} = \boldsymbol{e}$. In short, when considering the primary iterate, any deviance in the direction of $\boldsymbol{\theta_e}$ will be ignored under $\Pi_{\mathbb{O}}$. Using the matrix notations we introduced, we can now succinctly write both \eqref{APPEND_IMP_REWARD_UPDATE} and \eqref{APPEND_IMP_PRIMARY_UPDATE} as
$$
\widehat{\boldsymbol{\Theta}}_{t+1} = \boldsymbol{\Pi}\left[\Pi_{R_{\boldsymbol{\Theta}}}\left\lbrace\widehat{\boldsymbol{\Theta}}_{t} + \beta_t \boldsymbol{D}_t\left(\boldsymbol{ A_t}\widehat{\boldsymbol{\Theta}}_{t} + \boldsymbol{b}_t\right) \right\rbrace\right],
$$
where 
$
\Pi_{R_{\boldsymbol{\Theta}}}(\widehat{\boldsymbol{\Theta}}) := \left[
\Pi_{R_\omega} (\widehat{\omega}), 
\Pi_{R_{\boldsymbol{\theta}}}(\widehat{\pmb{\theta}})
 \right]^\top
$ for $\widehat{\boldsymbol{\Theta}} = [
\widehat{\omega}, 
\widehat{\pmb{\theta}}]^\top$, $R_{\boldsymbol{\Theta}} = \sqrt{R_\omega^2 + R^2_{\boldsymbol{\theta}}}$. Since $\boldsymbol{\Theta^*} = \boldsymbol{\Pi}\left(\Pi_{R_{\boldsymbol{\Theta}}}\boldsymbol{\Theta^*}\right)$, we have
\begin{align}
\|\boldsymbol{\Theta^*} - \widehat{\boldsymbol{\Theta}}_{t+1}\|^2&= \left\|\boldsymbol{\Pi}\left(\Pi_{R_{\boldsymbol{\Theta}}}\boldsymbol{\Theta^*}\right) - \boldsymbol{\Pi}\left[\Pi_{R_{\boldsymbol{\Theta}}}\left\lbrace\widehat{\boldsymbol{\Theta}}_{t} + \beta_t \boldsymbol{D}_t\left(\boldsymbol{A}_t \widehat{\boldsymbol{\Theta}}_{t}+ \boldsymbol{b}_t\right)\right\rbrace\right] \right\|^2 \nonumber \\ &\le \left\|\Pi_{R_{\boldsymbol{\Theta}}}\boldsymbol{\Theta^*} - \Pi_{R_{\boldsymbol{\Theta}}}\left\lbrace\widehat{\boldsymbol{\Theta}}_{t} + \beta_t \boldsymbol{D}_t\left(\boldsymbol{ A_t} \widehat{\boldsymbol{\Theta}}_{t}+ \boldsymbol{b}_t\right)\right\rbrace \right\|^2 \nonumber \\
&\le \left\|\boldsymbol{\Theta^*} - \left[\widehat{\boldsymbol{\Theta}}_{t} + \beta_t \boldsymbol{D}_t\left(\boldsymbol{A}_t\widehat{\boldsymbol{\Theta}}_{t} + \boldsymbol{b}_t\right)\right] \right\|^2 \nonumber \\
&\le \left\|\boldsymbol{\Theta^*} - \widehat{\boldsymbol{\Theta}}_{t} \right\|^2 \underbrace{- 2\beta_t \left[\boldsymbol{D}_t\left(\boldsymbol{A}_t\widehat{\boldsymbol{\Theta}}_{t} + \boldsymbol{b}_t\right) \right]^\top \left( \boldsymbol{\Theta^*} - \widehat{\boldsymbol{\Theta}}_{t}\right)}_{(*)} + \underbrace{\beta_t^2 \left\|\boldsymbol{D}_t\left(\boldsymbol{A}_t\widehat{\boldsymbol{\Theta}}_{t} + \boldsymbol{b}_t\right) \right\|^2}_{(**)}, \label{fin_upper_bound}
\end{align}
where the first inequality is due to non-expansiveness of the operator $\boldsymbol{\Pi}$ and the second inequality is due to non-expansiveness of the projection operator $\Pi_{R_{\boldsymbol{\Theta}}}$. We first obtain an upper bound of the expression in $(*)$. To this end, note that
\begin{align*}
(*) &= - 2\beta_t \left\lbrace \boldsymbol{D}_t\left(\boldsymbol{A}_t\widehat{\boldsymbol{\Theta}}_{t} - \boldsymbol{A}_t\boldsymbol{\Theta^*} + \boldsymbol{A}_t\boldsymbol{\Theta^*} - \boldsymbol{A}\boldsymbol{\Theta^*} +  \boldsymbol{A}\boldsymbol{\Theta^*} + \boldsymbol{b}_t\right) \right\rbrace^\top \left(\boldsymbol{\Theta^*} - \widehat{\boldsymbol{\Theta}}_{t}\right) \nonumber \\
& = 2\beta_t (\boldsymbol{\Theta^*}-\widehat{\boldsymbol{\Theta}}_{t})^\top \boldsymbol{A}_t^\top \boldsymbol{D}_t  \left( \boldsymbol{\Theta^*} - \widehat{\boldsymbol{\Theta}}_{t}\right) - 2\beta_t \left\{\left(\boldsymbol{A}_t - \boldsymbol{A}\right)\boldsymbol{\Theta^*} + \left( \boldsymbol{b}_t-\boldsymbol{b}\right)\right\}^\top \boldsymbol{D}_t(\boldsymbol{\Theta^*}-\widehat{\boldsymbol{\Theta}}_{t}) \nonumber \\
& = 2\beta_t (\boldsymbol{\Theta^*}-\widehat{\boldsymbol{\Theta}}_{t})^\top \boldsymbol{A}_t^\top \underline{\boldsymbol{D}_t}  \left( \boldsymbol{\Theta^*} - \widehat{\boldsymbol{\Theta}}_{t}\right) - 2\beta_t \left\{\left(\boldsymbol{A}_t - \boldsymbol{A}\right)\boldsymbol{\Theta^*} + \left(\boldsymbol{b}_t-\boldsymbol{b}\right)\right\}^\top \underline{\boldsymbol{D}_t}(\boldsymbol{\Theta^*}-\widehat{\boldsymbol{\Theta}}_{t}) \\
&\quad + 2\beta_t (\boldsymbol{\Theta^*}-\widehat{\boldsymbol{\Theta}}_{t})^\top \boldsymbol{A}_t^\top (\boldsymbol{D}_t - \underline{\boldsymbol{D}_t})  \left( \boldsymbol{\Theta^*} - \widehat{\boldsymbol{\Theta}}_{t}\right) - 2\beta_t \left\{\left(\boldsymbol{A}_t - \boldsymbol{A}\right)\boldsymbol{\Theta^*} + \left(\boldsymbol{b}_t-\boldsymbol{b}\right)\right\}^\top (\boldsymbol{D}_t-\underline{\boldsymbol{D}_t})(\boldsymbol{\Theta^*}-\widehat{\boldsymbol{\Theta}}_{t})
\end{align*}
where the second equality follows from $\boldsymbol{A} \boldsymbol{\Theta^*} =\boldsymbol{b}$. We bound each term in the last expression separately. For the first term, note that
\begin{align*}
(\boldsymbol{\Theta^*}-\widehat{\boldsymbol{\Theta}}_{t})^\top \boldsymbol{A}_t^\top  \underline{\boldsymbol{D}_t}\left(\boldsymbol{\Theta^*} - \widehat{\boldsymbol{\Theta}}_{t}\right) &= (\boldsymbol{\Theta^*}-\widehat{\boldsymbol{\Theta}}_{t})^\top \left(\boldsymbol{A}_t^\top-\boldsymbol{A}^\top \right)\underline{\boldsymbol{D}_t}\left( \boldsymbol{\Theta^*} - \widehat{\boldsymbol{\Theta}}_{t}\right) +  (\boldsymbol{\Theta^*}-\widehat{\boldsymbol{\Theta}}_{t})^\top \boldsymbol{A}^\top \underline{\boldsymbol{D}_t} \left( \boldsymbol{\Theta^*} -\widehat{\boldsymbol{\Theta}}_{t}\right)\\
&\le \left(\boldsymbol{\Theta^*}-\widehat{\boldsymbol{\Theta}}_{t}\right)^\top \left(\boldsymbol{A}_t^\top-\boldsymbol{A}^\top \right) \underline{\boldsymbol{D}_t} \left(\boldsymbol{\Theta^*} - \widehat{\boldsymbol{\Theta}}_{t}\right) -\frac{\Delta \gamma_t}{2} \|\boldsymbol{\Theta^*}-\widehat{\boldsymbol{\Theta}}_{t}\|^2,   
\end{align*}
where the first inequality is the direct consequence of Lemma \ref{LEMMA:NEG_DEF}, i.e., 
$$
(\boldsymbol{\Theta^*}-\widehat{\boldsymbol{\Theta}}_{t})^\top \boldsymbol{A}^\top \left( \boldsymbol{\Theta^*} -\widehat{\boldsymbol{\Theta}}_{t}\right)
\le -\frac{\Delta}{2} \|\boldsymbol{\Theta^*}-\widehat{\boldsymbol{\Theta}}_{t}\|^2, \quad (\boldsymbol{\Theta^*}-\widehat{\boldsymbol{\Theta}}_{t}) \in \mathbb{R}\times \mathbb{O},
$$
as long as $c_\alpha \ge \Delta + \sqrt{\frac{1}{\Delta^2 (1-\lambda)^4} - \frac{1}{(1-\lambda)^2}}$. Therefore, we obtain the following bound for the first term
\begin{equation}\label{First_First_Bound}
    2\beta_t  (\boldsymbol{\Theta^*}-\widehat{\boldsymbol{\Theta}}_{t})^\top \boldsymbol{A}_t^\top\underline{\boldsymbol{D}_t} \left(\boldsymbol{\Theta^*} - \widehat{\boldsymbol{\Theta}}_{t}\right) \le  2\beta_t\xi_{t}(\widehat{\boldsymbol{\Theta}}_{t}, \boldsymbol{X}_t)  -\beta_t\gamma_{t}\Delta \|\boldsymbol{\Theta^*}-\widehat{\boldsymbol{\Theta}}_{t}\|^2
\end{equation}
which holds almost surely. For the second term, notice that
\begin{equation}\label{First_Second_Bound}
    - 2\beta_t \left\{\left(\boldsymbol{A}_t -\boldsymbol{A}\right)\boldsymbol{\Theta^*} + \left(\boldsymbol{b}_t-\boldsymbol{b}\right)\right\}^\top\underline{\boldsymbol{D}_t}(\boldsymbol{\Theta^*}-\widehat{\boldsymbol{\Theta}}_{t}) = -2\beta_t \zeta_{t}(\widehat{\boldsymbol{\Theta}}_{t}, \boldsymbol{X}_t).
\end{equation}
For the last two terms, applying Cauchy-Schwarz inequality with $\|\boldsymbol{D}_t-\underline{\boldsymbol{D}_t}\| \le \frac{(1+c_\alpha)\beta_t}{(1-\lambda)^2}$ (see Lemma \ref{LEMMA:D_Dt_bound}) gives us
\begin{align}
   2\beta_t (\boldsymbol{\Theta^*}-\widehat{\boldsymbol{\Theta}}_{t})^\top \boldsymbol{A}_t^\top (\boldsymbol{D}_t - \underline{\boldsymbol{D}_t})  \left( \boldsymbol{\Theta^*} - \widehat{\boldsymbol{\Theta}}_{t}\right) &\le \frac{8 R_{\boldsymbol{\Theta}}^2 A_{\max}(1+c_\alpha)\beta_t^2}{(1-\lambda)^2} \label{First_Third_Bound} \\
   - 2\beta_t \left\{\left(\boldsymbol{A}_t - \boldsymbol{A}\right)\boldsymbol{\Theta^*} + \left(\boldsymbol{b}_t-\boldsymbol{b}\right)\right\}^\top (\boldsymbol{D}_t-\underline{\boldsymbol{D}_t})(\boldsymbol{\Theta^*}-\widehat{\boldsymbol{\Theta}}_{t}) 
    &\le \frac{8 R_{\boldsymbol{\Theta}}(A_{\max}R_{\boldsymbol{\Theta}} + b_{\max})(1+c_\alpha)\beta_t^2}{(1-\lambda)^2} \label{First_Fourth_Bound}
\end{align}
where $A_{\max}:= \sqrt{c_\alpha^2 + \frac{5}{(1-\lambda)^2}}$ and $b_{\max}:= \sqrt{c_\alpha^2 + \frac{1}{(1-\lambda)^2}}$, which respectively serves as a uniform bound on $\|\boldsymbol{A}_t\|$ and $\|\boldsymbol{b}_t\|$.
Combining \eqref{First_First_Bound}, \eqref{First_Second_Bound}, \eqref{First_Third_Bound} and \eqref{First_Fourth_Bound}, we get
\begin{equation}\label{First_Bound}
    (*) \le  -\beta_t\gamma_t\Delta \|\boldsymbol{\Theta^*}-\widehat{\boldsymbol{\Theta}}_{t}\|^2 + 2\beta_t \xi_t(\widehat{\boldsymbol{\Theta}}_{t}, \boldsymbol{X}_t) -2\beta_t \zeta_{t}(\widehat{\boldsymbol{\Theta}}_{t}, \boldsymbol{X}_t) + G_1 \beta_t^2,
\end{equation}
where $G_1 = \{8R_{\boldsymbol{\Theta}}(2A_{\max}R_{\boldsymbol{\Theta}} + b_{\max})(1+c_\alpha)\}/(1-\lambda)^2$. 

Next, we obtain an upper bound of the expression in $(**)$. Thanks to the fact $\|\boldsymbol{D}_t\| \le 1$, we get
\begin{equation}\label{Second_Bound}
    (**) \le \beta_t^2 \|\boldsymbol{\boldsymbol{A}_t} \widehat{\boldsymbol{\Theta}}_{t} + \boldsymbol{\boldsymbol{b}_t}\|^2\le 2\beta_t^2 \left( A_{\max}^2 R_{\boldsymbol{\Theta}}^2 + b_{\max}^2\right) =:G_2\beta_t^2.
\end{equation}
Combining \eqref{First_Bound} and \eqref{Second_Bound}, we have
\begin{align*}
\|\boldsymbol{\Theta^*} - \widehat{\boldsymbol{\Theta}}_{t+1}\|^2 \le  (1-\beta_t\gamma_t\Delta )\|\boldsymbol{\Theta^*}-\widehat{\boldsymbol{\Theta}}_{t}\|^2 +2\beta_t\xi_{t}(\widehat{\boldsymbol{\Theta}}_{t}, \boldsymbol{X}_t) -2\beta_t \zeta_{t}(\widehat{\boldsymbol{\Theta}}_{t}, \boldsymbol{X}_t)+G\beta_t^2,
\end{align*}
with $G = G_1 + G_2$. Taking expectations of both sides, we have
\begin{align}
\mathbb{E}^\mu\|\boldsymbol{\Theta^*} - \widehat{\boldsymbol{\Theta}}_{t+1}\|^2 &\le  (1-\beta_t\gamma_t\Delta )\mathbb{E}^\mu\|\boldsymbol{\Theta^*}-\widehat{\boldsymbol{\Theta}}_{t}\|^2 +2\beta_t \mathbb{E}^\mu\xi_t(\widehat{\boldsymbol{\Theta}}_{t}, \boldsymbol{X}_t) -2\beta_t \mathbb{E}^\mu\zeta_t(\widehat{\boldsymbol{\Theta}}_{t}, \boldsymbol{X}_t)+G\beta_t^2 \nonumber \\
&\le  (1-\beta_t\gamma_t\Delta )\mathbb{E}^\mu\|\boldsymbol{\Theta^*}-\widehat{\boldsymbol{\Theta}}_{t}\|^2 +2\beta_t \left|\mathbb{E}^\mu\xi_t(\widehat{\boldsymbol{\Theta}}_{t}, \boldsymbol{X}_t)\right|+2\beta_t \left|\mathbb{E}^\mu\zeta_t(\widehat{\boldsymbol{\Theta}}_{t}, \boldsymbol{X}_t)\right|+G\beta_t^2 \nonumber \\
&\le \left\lbrace\prod_{i=0}^t(1-\beta_i\gamma_i\Delta )\right\rbrace \|\boldsymbol{\Theta^*}-\widehat{\boldsymbol{\Theta}}_{0}\|^2 \nonumber \\ & \quad + 2\sum_{i=0}^t \left\lbrace\prod_{k=i+1}^t(1-\beta_k\gamma_k\Delta )\right\rbrace\beta_i \left\lbrace \left|\mathbb{E}^\mu\xi_i(\widehat{\boldsymbol{\Theta}}_{i}, \boldsymbol{X}_i)\right| + \left|\mathbb{E}^\mu\zeta_i(\widehat{\boldsymbol{\Theta}}_{i}, \boldsymbol{X}_i)\right|\right\rbrace \nonumber \\ & \quad + G \sum_{i=0}^t \left\lbrace\prod_{k=i+1}^t(1-\beta_k\gamma_k\Delta )\right\rbrace\beta_i^2 \label{Theta_iterate_prelim_bound}
\end{align}
where we have used the identity $1-\beta_i\gamma_i\Delta \in (0,1)$ for all $i \in \mathbb{N}$, which is true by the assumption $c_\alpha \ge \Delta + \sqrt{\frac{1}{\Delta^2 (1-\lambda)^4} - \frac{1}{(1-\lambda)^2}}$.

\begin{theorem}\label{THM:FIN_TIME_BOUND_CONST_APPEND}
Suppose the Markov chain $\{S^\mu_t\}_{t \in \mathbb{N}}$ is uniformly geometrically ergodic with a rate parameter $\rho \in (0,1)$ and the step-size ratio parameter is chosen to satisfy $c_\alpha \ge \Delta + \sqrt{\frac{1}{\Delta^2 (1-\lambda)^4} - \frac{1}{(1-\lambda)^2}}$. With $\lambda \in [0,1)$ and constant step-size $\beta_t = \beta$, the iterates of the projected average-reward implicit TD($\lambda$) algorithm satisfy the following finite-time error bound
\begin{align*}
\mathbb{E}^\mu\|\boldsymbol{\Theta^*} - \widehat{\boldsymbol{\Theta}}_{t+1}\|^2 \lesssim (1-\beta\gamma\Delta)^{t+1} \|\boldsymbol{\Theta^*}-\widehat{\boldsymbol{\Theta}}_{0}\|^2  + \mathcal{O}\left(\beta\tau_{\beta} + h^{\tau_{\beta}} + \beta t \tau_{\beta} h^t \right), \quad t \ge 0
\end{align*}
where
$
\gamma = \min\left\{\frac{1}{1+c_\alpha \beta}, \frac{(1-\lambda)^2}{(1-\lambda)^2 + \beta}\right\}~\text{and}~ h = \max\left\{1-\beta\gamma\Delta, ~\rho, ~\lambda \right\}.
$
\end{theorem}
\begin{proof}
Starting from \eqref{Theta_iterate_prelim_bound}, we have
\begin{align*}
\mathbb{E}^\mu\|\boldsymbol{\Theta^*} - \widehat{\boldsymbol{\Theta}}_{t+1}\|^2 &\le (1-\beta\gamma\Delta )^{t+1} \|\boldsymbol{\Theta^*}-\widehat{\boldsymbol{\Theta}}_{0}\|^2 \\ &+ 2\beta\sum_{i=0}^t (1-\beta\gamma\Delta )^{t-i} \left\lbrace \left|\mathbb{E}^\mu\xi_i(\widehat{\boldsymbol{\Theta}}_{i}, \boldsymbol{X}_i)\right| + \left|\mathbb{E}^\mu\zeta_i(\widehat{\boldsymbol{\Theta}}_{i}, \boldsymbol{X}_i)\right|\right\rbrace  
\\ &+ G\beta^2 \sum_{i=0}^t (1-\beta\gamma\Delta )^{t-i}   \\
&\lesssim (1-\beta\gamma\Delta )^{t+1} \|\boldsymbol{\Theta^*}-\widehat{\boldsymbol{\Theta}}_{0}\|^2  \\&\quad + \beta^2\tau_{\beta} \sum_{i=0}^{2\tau_{\beta}} (1-\beta\gamma\Delta )^{t-i}  + \beta\tau_\beta\sum_{i=0}^{2\tau_{\beta}} (1-\beta\gamma\Delta )^{t-i} q^i \\
&\quad + \beta(\beta\tau_{\beta} + q^{\tau_{\beta}})\sum_{i=2\tau_{\beta}+1}^{t} (1-\beta\gamma\Delta )^{t-i} + \beta\tau_{\beta} \sum_{i=2\tau_{\beta}+1}^{t} (1-\beta\gamma\Delta )^{t-i}  q^i  \\& \quad +\beta^2 \sum_{i=0}^t (1-\beta\gamma\Delta )^{t-i} \\
&\lesssim (1-\beta\gamma\Delta)^{t+1} \|\boldsymbol{\Theta^*}-\widehat{\boldsymbol{\Theta}}_{0}\|^2  + \frac{\beta\tau_{\beta} }{\gamma\Delta}  + \beta\tau_\beta^2 h^t + \frac{\beta\tau_{\beta} + q^{\tau_{\beta}}}{\gamma\Delta} + \beta \tau_{\beta} t  h^t   +\frac{\beta}{\gamma\Delta} 
\end{align*}
where in the second inequality, we used Lemma \ref{LEMMA:expect_zeta_bound} and Lemma \ref{LEMMA:expect_xi_bound} with $q = \max\{\rho, \lambda\}$. In the last inequality, we used $h = \max\left\{1-\beta\gamma\Delta, ~\rho, ~\lambda \right\}$. Summarizing the terms, we get
\begin{align*}
\mathbb{E}^\mu\|\boldsymbol{\Theta^*} - \widehat{\boldsymbol{\Theta}}_{t+1}\|^2 \lesssim (1-\beta\gamma\Delta)^{t+1}\|\boldsymbol{\Theta^*}-\widehat{\boldsymbol{\Theta}}_{0}\|^2 +\mathcal{O}\left(\beta\tau_{\beta} + h^{\tau_{\beta}} + \beta \tau_{\beta} t h^t \right).
\end{align*}
\end{proof}

\begin{theorem}\label{THM:FIN_TIME_BOUND_DECR_APPEND}
Suppose the Markov chain $\{S^\mu_t\}_{t \in \mathbb{N}}$ is uniformly geometrically ergodic with a rate parameter $\rho \in (0,1)$ and the step-size ratio parameter is chosen to satisfy $c_\alpha \ge \Delta + \sqrt{\frac{1}{\Delta^2 (1-\lambda)^4} - \frac{1}{(1-\lambda)^2}}$. With $\lambda \in [0,1)$ and decreasing step-sizes $\beta_t = \frac{\beta_0}{(t+1)^s}, ~s \in (0,1)$, the iterates of the projected average-reward implicit TD($\lambda$) algorithm satisfy the following finite-time error bound, for $t \ge 0$, 
\begin{align*}
\mathbb{E}^\mu\|\boldsymbol{\Theta^*} - \widehat{\boldsymbol{\Theta}}_{t+1}\|^2 \lesssim \exp\left[-\frac{\Delta \gamma_0 \beta_0}{1-s}\{(1+t)^{1-s}-1\} \right] \|\boldsymbol{\Theta^*} - \widehat{\boldsymbol{\Theta}}_{0}\|^2 + \mathcal{O}\left\lbrace\tau_{\beta_t}t \exp\left(-ct^{1-s}\right) + \tau_{\beta_t} t^{-s} + q^{\tau_{\beta_t}}\right\rbrace,
\end{align*}
for some constant $c > 0$,
$
\gamma_{0} = \min\left\{\frac{1}{1+c_\alpha \beta_0}, \frac{(1-\lambda)^2}{(1-\lambda)^2 + \beta_0}\right\}~\text{and}~ q = \max\{\rho, \lambda \}.
$
\end{theorem}
\begin{proof}
Starting from \eqref{Theta_iterate_prelim_bound} with the identity $1-\beta_i\gamma_i\Delta \le \exp(-\beta_i\gamma_i\Delta)$ for all $i \in \mathbb{N}$, we have
\begin{align*}
\mathbb{E}^\mu\|\boldsymbol{\Theta^*} - \widehat{\boldsymbol{\Theta}}_{t+1}\|^2 &\le \exp\left(-\Delta\sum_{i=0}^t\beta_i\gamma_i\right) \|\boldsymbol{\Theta^*} - \widehat{\boldsymbol{\Theta}}_{0}\|^2 \\
& \quad+ 2\sum_{i=0}^t \exp\left(-\Delta\sum_{k=i+1}^t\beta_k\gamma_k\right) \beta_i \left|\mathbb{E}^\mu\xi_i(\widehat{\boldsymbol{\Theta}}_{i}, \boldsymbol{X}_i)+\mathbb{E}^\mu\zeta_i(\widehat{\boldsymbol{\Theta}}_{i}, \boldsymbol{X}_i)\right| \\
& \quad + G\sum_{i=0}^t \exp\left(-\Delta\sum_{k=i+1}^t\beta_k\gamma_k\right) \beta^2_i \\
&\le \exp\left(-\Delta \gamma_0 \sum_{i=0}^t\beta_i\right) \|\boldsymbol{\Theta^*} - \widehat{\boldsymbol{\Theta}}_{0}\|^2 \\
&\quad + 2\sum_{i=0}^t \exp\left(-\Delta\gamma_0\sum_{k=i+1}^t\beta_k\right) \beta_i \left|\mathbb{E}^\mu\xi_i(\widehat{\boldsymbol{\Theta}}_{i}, \boldsymbol{X}_i)+\mathbb{E}^\mu\zeta_i(\widehat{\boldsymbol{\Theta}}_{i}, \boldsymbol{X}_i)\right| \\
& \quad + G\sum_{i=0}^t \exp\left(-\Delta\gamma_0\sum_{k=i+1}^t\beta_k\right) \beta^2_i\\
\end{align*}
where the second inequality follows from the fact that $\gamma_i$ is increasing. The final expression can be re-expressed as
\begin{align}
 &\exp\left\lbrace-\Delta \gamma_0 \beta_0 \sum_{i=0}^t \frac{1}{(i+1)^s}\right\rbrace \|\boldsymbol{\Theta^*} - \widehat{\boldsymbol{\Theta}}_{0}\|^2 \label{decr_error_first_term}\\ &\quad + 2\sum_{i=0}^t \exp\left\lbrace-\Delta\gamma_0\beta_0\sum_{k=i+1}^t\frac{1}{(k+1)^s}\right\rbrace \beta_i \left|\mathbb{E}^\mu \xi_i(\widehat{\boldsymbol{\Theta}}_{i}, \boldsymbol{X}_i)+\mathbb{E}^\mu \zeta_i(\widehat{\boldsymbol{\Theta}}_{i}, \boldsymbol{X}_i)\right| \label{decr_error_second_term} \\
& \quad + G\sum_{i=0}^t \exp\left\lbrace-\Delta\gamma_0\beta_0\sum_{k=i+1}^t\frac{1}{(k+1)^s}\right\rbrace \beta^2_i.\label{decr_error_third_term}
\end{align}
The first term in \eqref{decr_error_first_term} admits the following bound
\begin{equation}\label{decreasing_first_bound}
 \exp\left\lbrace-\Delta \gamma_0 \beta_0 \sum_{i=0}^t \frac{1}{(i+1)^s}\right\rbrace \|\boldsymbol{\Theta^*} - \widehat{\boldsymbol{\Theta}}_{0}\|^2   \le \exp\left[-\frac{\Delta \gamma_0 \beta_0}{1-s}\{(1+t)^{1-s}-1\} \right] \|\boldsymbol{\Theta^{*}} - \widehat{\boldsymbol{\Theta}}_{0}\|^2.     
\end{equation}

\noindent For the second term in \eqref{decr_error_second_term}, we first note that 
\begin{align}
&2\sum_{i=0}^t \exp\left\lbrace-\Delta\gamma_0\beta_0\sum_{k=i+1}^t\frac{1}{(k+1)^s}\right\rbrace \beta_i \left|\mathbb{E}^\mu\xi_i(\widehat{\boldsymbol{\Theta}}_{i}, \boldsymbol{X}_i)+\mathbb{E}^\mu\zeta_i(\widehat{\boldsymbol{\Theta}}_{i}, \boldsymbol{X}_i)\right| \nonumber \\&\lesssim \sum_{i=0}^{2\tau_{\beta_t}} \exp\left\lbrace-\Delta\gamma_0\beta_0\sum_{k=i+1}^t\frac{1}{(k+1)^s}\right\rbrace \beta_i \left(\tau_{\beta_t}\beta_0 + iq^i\right) \nonumber \\ &\quad + \sum_{i=2\tau_{\beta_t}+1}^t \exp\left\lbrace-\Delta\gamma_0\beta_0\sum_{k=i+1}^t\frac{1}{(k+1)^s}\right\rbrace \beta_i \left(\tau_{\beta_t}\beta_{i-2\tau_{\beta_t}} + \tau_{\beta_t}q^i + q^{\tau_{\beta_t}} \right) \nonumber  \\
&\lesssim \tau_{\beta_t}\beta_0 \sum_{i=0}^{2\tau_{\beta_t}} \exp\left\lbrace-\Delta\gamma_0\beta_0\sum_{k=i+1}^t\frac{1}{(k+1)^s}\right\rbrace \beta_i \label{error_bd_decr_first} \\ 
&\quad + \tau_{\beta_t}\sum_{i=0}^t \exp\left\lbrace-\Delta\gamma_0\beta_0\sum_{k=i+1}^t\frac{1}{(k+1)^s}\right\rbrace \beta_i q^i \label{error_bd_decr_second}  \\ 
&\quad + \tau_{\beta_t}\sum_{i=2\tau_{\beta_t}+1}^t \exp\left\lbrace-\Delta\gamma_0\beta_0\sum_{k=i+1}^t\frac{1}{(k+1)^s}\right\rbrace \beta_i \beta_{i-2\tau_{\beta_t}} \label{error_bd_decr_third} \\
&\quad + q^{\tau_{\beta_t}} \sum_{i=2\tau_{\beta_t}+1}^t \exp\left\lbrace-\Delta\gamma_0\beta_0\sum_{k=i+1}^t\frac{1}{(k+1)^s}\right\rbrace \beta_i   \label{error_bd_decr_fourth} 
\end{align}
where we used Lemma \ref{LEMMA:expect_zeta_bound} and Lemma \ref{LEMMA:expect_xi_bound} in the first inequality. We first establish an upper bound on \eqref{error_bd_decr_first} using Lemma \ref{LEMMA:prelim_exp_bound}. 
\begin{align*}
    \tau_{\beta_t}\beta_0\sum_{i=0}^{2\tau_{\beta_t}} \exp\left\lbrace-\Delta\gamma_0\beta_0\sum_{k=i+1}^t\frac{1}{(k+1)^s}\right\rbrace \beta_i &\le \frac{\tau_{\beta_t}\beta_0  e^{\Delta \gamma_0\beta_0}}{\Delta \gamma_0} \exp\left[-\frac{\Delta \gamma_0 \beta_0}{(1-s)}\left\lbrace(1+t)^{1-s} - (1+2\tau_{\beta_t})^{1-s}\right\rbrace \right] \\
    &\lesssim \tau_{\beta_t}\exp\left[-\frac{\Delta \gamma_0 \beta_0}{(1-s)}\left\lbrace(1+t)^{1-s} - (1+2\tau_{\beta_t})^{1-s}\right\rbrace \right]
\end{align*}

\noindent Next, we obtain an upper bound on \eqref{error_bd_decr_second}. To this end, consider
\begin{align}
    \tau_{\beta_t}\sum_{i=0}^{t} \exp\left\lbrace-\Delta\gamma_0\beta_0\sum_{k=i+1}^t\frac{1}{(k+1)^s}\right\rbrace \beta_i q^i &= \tau_{\beta_t}\sum_{i=0}^{\floor{t/2}} \exp\left\lbrace-\Delta\gamma_0\beta_0\sum_{k=i+1}^t\frac{1}{(k+1)^s}\right\rbrace \beta_i q^i \label{decr_err_second_first}\\ 
    &\quad + \tau_{\beta_t}\sum_{i=\floor{t/2}+1}^{t} \exp\left\lbrace-\Delta\gamma_0\beta_0\sum_{k=i+1}^t\frac{1}{(k+1)^s}\right\rbrace \beta_i q^i. \label{decr_err_second_second}
\end{align}
The term in \eqref{decr_err_second_first} admits the following bound
\begin{align*}
   \eqref{decr_err_second_first}
   \le \tau_{\beta_t}\beta_0\sum_{i=0}^{\floor{t/2}} \exp\left\lbrace-\Delta\gamma_0\beta_0\sum_{k=\floor{t/2}+1}^t\frac{1}{(k+1)^s}\right\rbrace. 
\end{align*}
Since 
$$
\sum_{k=\floor{t/2}+1}^t\frac{1}{(k+1)^s} \ge \int_{\floor{t/2}+1}^t x^{-s}dx \ge \frac{t^{1-s} - (t/2 + 1)^{1-s}}{1-s} = \Omega\left(t^{1-s}\right),
$$
we have
$$
\eqref{decr_err_second_first} \lesssim \tau_{\beta_t} t \exp\left(-c t^{1-s} \right),
$$
for some constant $c > 0$. For the term in \eqref{decr_err_second_second}, we have
$$
\eqref{decr_err_second_second} \le \tau_{\beta_t}\sum_{i=\floor{t/2}+1}^{t}  \beta_i q^i \le \tau_{\beta_t}\sum_{i=\floor{t/2}+1}^{t} \frac{\beta_0}{(t/2)^s}q^i \lesssim \frac{\tau_{\beta_t}}{t^s}. 
$$
And hence, we get the bound for \eqref{error_bd_decr_second}, given by
$$
\tau_{\beta_t}\sum_{i=0}^{t} \exp\left\lbrace-\Delta\gamma_0\beta_0\sum_{k=i+1}^t\frac{1}{(k+1)^s}\right\rbrace \beta_i q^i \lesssim \tau_{\beta_t} t \exp\left(-c t^{1-s} \right)+ \frac{\tau_{\beta_t}}{t^s}.
$$
Next, we obtain an upper bound on \eqref{error_bd_decr_third}. From Lemma \ref{LEMMA:prelim_exp_bound},
\begin{align*}
 \tau_{\beta_t}\sum_{i=2\tau_{\beta_t}+1}^t \exp\left\lbrace-\Delta\gamma_0\beta_0\sum_{k=i+1}^t\frac{1}{(k+1)^s}\right\rbrace \beta_i \beta_{i-2\tau_{\beta_t}} \lesssim \tau_{\beta_t}\left(\exp\left[-\frac{\Delta \gamma_0\beta_0}{2(1-s)}\left\lbrace(t+1)^{1-s} - 1\right\rbrace \right] + \beta_{t-2\tau_{\beta_t}} \right).
\end{align*}

\noindent The only remaining term is the one in \eqref{error_bd_decr_fourth}, whose bound is given by 
\begin{align*}
q^{\tau_{\beta_t}} \sum_{i=2\tau_{\beta_t}+1}^t \exp\left\lbrace-\Delta\gamma_0\beta_0\sum_{k=i+1}^t\frac{1}{(k+1)^s}\right\rbrace \beta_i \lesssim q^{\tau_{\beta_t}},
\end{align*}
where we have used \eqref{prelim_second_pre}. Combining altogether to obtain a bound of \eqref{decr_error_second_term}, we get
\begin{align}
    \eqref{decr_error_second_term} &\lesssim \tau_{\beta_t}\exp\left[-\frac{\Delta \gamma_0 \beta_0}{(1-s)}\left\lbrace(1+t)^{1-s} - (1+2\tau_{\beta_t})^{1-s}\right\rbrace \right] \nonumber\\ &\quad + \tau_{\beta_t} t \exp\left(-c t^{1-s} \right)+ \frac{\tau_{\beta_t}}{t^s} \nonumber\\ &\quad + \tau_{\beta_t}\left(\exp\left[-\frac{\Delta \gamma_0\beta_0}{2(1-s)}\left\lbrace(1+t)^{1-s} - 1\right\rbrace \right] + \beta_{t-2\tau_{\beta_t}} \right) + q^{\tau_{\beta_t}} \nonumber\\
    & \quad \lesssim \tau_{\beta_t}t \exp\left(-ct^{1-s}\right) + \tau_{\beta_t} t^{-s} + q^{\tau_{\beta_t}} \label{decreasing_second_bound},
\end{align}
for some constant $c > 0$.\\

\noindent Lastly, from Lemma \ref{LEMMA:prelim_all_time_exp_bound}, the last term in \eqref{decr_error_third_term} is upper bounded by
\begin{align}
 G\sum_{i=0}^t \exp\left\lbrace-\Delta\gamma_0\beta_0\sum_{k=i+1}^t\frac{1}{(k+1)^s}\right\rbrace \beta^2_i &\lesssim \exp\left\lbrace-\frac{\Delta \gamma_0}{2}\beta_0\sum_{k=0}^t\frac{1}{(k+1)^s}\right\rbrace + \beta_t \nonumber \\
 &\lesssim  \exp\left[-\frac{\Delta \gamma_0 \beta_0}{1-s}\{(1+t)^{1-s}-1\} \right] + \beta_t .\label{decreasing_last_bound}
\end{align}
Combining \eqref{decreasing_first_bound}, \eqref{decreasing_second_bound} and \eqref{decreasing_last_bound}, we have
\begin{align*}
\mathbb{E}^\mu\|\boldsymbol{\Theta^*} - \widehat{\boldsymbol{\Theta}}_{t+1}\|^2&\lesssim \exp\left[-\frac{\Delta \gamma_0 \beta_0}{1-s}\{(1+t)^{1-s}-1\} \right] \|\boldsymbol{\Theta^*} - \widehat{\boldsymbol{\Theta}}_{0}\|^2 \\ &\quad+ \tau_{\beta_t}t \exp\left(-ct^{1-s}\right) + \tau_{\beta_t} t^{-s} + q^{\tau_{\beta_t}} \\ &\quad+ \exp\left[-\frac{\Delta \gamma_0 \beta_0}{1-s}\{(1+t)^{1-s}-1\} \right] + \beta_t,
\end{align*}
which can be further succinctly written as
$$
\mathbb{E}^\mu\|\boldsymbol{\Theta^*} - \widehat{\boldsymbol{\Theta}}_{t+1}\|^2 \le \exp\left[-\frac{\Delta \gamma_0 \beta_0}{1-s}\{(1+t)^{1-s}-1\} \right] \|\boldsymbol{\Theta^*} - \widehat{\boldsymbol{\Theta}}_{0}\|^2 + \mathcal{O}\left\lbrace\tau_{\beta_t}t \exp\left(-ct^{1-s}\right) + \tau_{\beta_t} t^{-s} + q^{\tau_{\beta_t}}\right\rbrace
$$
for some constant $c > 0$.
\end{proof}
\section{Supporting lemmas}
\noindent Our goal is to establish a finite-time error bound on $\|\boldsymbol{\Theta^*}-\widehat{\boldsymbol{\Theta}}_{t}\|$. To this end, we state the preliminary lemmas and provide their proofs. The first lemma establishes a norm bound on the eligibility trace.

\begin{lemma}
\label{LEMMA:eligibility_trace_bound}
Given an exponential weighting parameter \( \lambda \in [0,1) \), \( \|\boldsymbol{z}_t\| \leq \frac{1}{1 - \lambda} \) for all \( t \in \mathbb{N} \).
\end{lemma}
\begin{proof}
From the definition of the eligibility trace, \( \boldsymbol{z}_t = \sum_{i=0}^{t} \lambda^{t - i} \boldsymbol{\phi_i} \), we observe
$
\|\boldsymbol{z}_t\| = \left\|\sum_{i=0}^{t} \lambda^{t - i} \boldsymbol{\phi_i}\right\| \leq \sum_{i=0}^{t} \lambda^{t-i}  \leq  \frac{1}{1 - \lambda }.
$
The first inequality comes from the triangle inequality, given $\|\boldsymbol{\phi_i}\| \leq 1$. The latter is true by the sum of a geometric series. 
\end{proof}

\noindent Next lemma provides explicit bounds on the error incurred by replacing the true, steady-state eligibility trace with its $\tau$‑step truncated version.
\begin{lemma}\label{LEMMA:EL_discrepancy_bound}
Given an initial state $s_0 \in \mathcal{S}$, suppose $\{S^\mu_t\}_{t\in \mathbb{N}}$ is uniformly geometrically ergodic with a rate parameter $\rho \in (0,1)$. For all $t \in \mathbb{N}$, let $\tau \in \{0,\cdots, t\}$, then
\begin{enumerate}
    \item $\left\|\mathbb{E}^{\boldsymbol{\pi}^\mu} \left(\boldsymbol{z}_t\right) - \mathbb{E}^\mu\left(\boldsymbol{z}_{t-\tau:t}\right)\right\| \lesssim \tau q^t + \lambda^\tau $
    \item $\left\|\mathbb{E}^{\boldsymbol{\pi}^\mu} \left(\boldsymbol{z}_t \boldsymbol{\phi}_t^\top \right) - \mathbb{E}^\mu\left(\boldsymbol{z}_{t-\tau:t}\boldsymbol{\phi}_t^\top\right)\right\| \lesssim \tau q^t + \lambda^\tau $
    \item $\left\|\mathbb{E}^{\boldsymbol{\pi}^\mu} \left(\boldsymbol{z}_t\boldsymbol{\phi}_{t+1}^\top\right) - \mathbb{E}^\mu\left(\boldsymbol{z}_{t-\tau:t}\boldsymbol{\phi}_{t+1}^\top\right)\right\| \lesssim \tau q^t + \lambda^\tau$
\end{enumerate}
where $q:= \max\{\lambda, \rho\}$.
\end{lemma}
\begin{proof} 
We leverage the steady-state expression of the eligibility trace $\boldsymbol{z}_t = \sum_{l=-\infty}^t \lambda^{t-l}\boldsymbol{\phi}_l$ whenever the expectation is with respect to the stationary distribution induced by the policy $\mu$. \\

\noindent Proof of the first statement: 
\begin{align*}
    \mathbb{E}^{\boldsymbol{\pi}^\mu} \left(\boldsymbol{z}_t\right) -\mathbb{E}^\mu\left(\boldsymbol{z}_{t-\tau:t}\right)
    &=\mathbb{E}^{\boldsymbol{\pi}^\mu}\left(\sum_{l=-\infty}^t \lambda^{t-l}\boldsymbol{\phi}_l \right) - \mathbb{E}^\mu\left(\sum_{l=t-\tau}^t \lambda^{t-l}\boldsymbol{\phi}_l \right)\\
    &= \sum_{l=-\infty}^t \lambda^{t-l}\mathbb{E}^{\boldsymbol{\pi}^\mu}\left(\boldsymbol{\phi}_l \right) - \sum_{l=t-\tau}^t \lambda^{t-l}\mathbb{E}^\mu\left(\boldsymbol{\phi}_l\right)\\
    &= \sum_{l=t-\tau}^t \lambda^{t-l}\left\{\mathbb{E}^{\boldsymbol{\pi}^\mu}\left(\boldsymbol{\phi}_l \right) - \mathbb{E}^\mu\left(\boldsymbol{\phi}_l \right)\right\}+ \sum_{l=-\infty}^{t-\tau-1} \lambda^{t-l}\mathbb{E}^{\boldsymbol{\pi}^\mu} \left(\boldsymbol{\phi}_l\right)\\
    & = \sum_{l=t-\tau}^t \lambda^{t-l}\left[\sum_{s \in \mathcal{S}} \left\{\pi_s^\mu\boldsymbol{\phi}(s) - p^\mu(S^\mu_l = s|S^\mu_0 = s_0)\boldsymbol{\phi}(s)\right\}\right]+ \sum_{l=-\infty}^{t-\tau-1} \lambda^{t-l}\mathbb{E}^{\boldsymbol{\pi}^\mu} \left(\boldsymbol{\phi}_l\right)
\end{align*}
Note that $\sup_{s \in \mathcal{S}} |\pi_s^\mu - p^\mu(S^\mu_l = s|S^\mu_0 = s_0)| \lesssim \rho^{l}$ for some $\rho \in (0,1)$ follows from the geometric ergodicity of the Markov chain $\{S^\mu_t\}_{t \in \mathbb{N}}$. From the finiteness of $\mathcal{S}$ and normalized features, we have
\begin{align*}
 \left\|\mathbb{E}^{\boldsymbol{\pi}^\mu} \left(\boldsymbol{z}_t\right) - \mathbb{E}^\mu\left(\boldsymbol{z}_{t-\tau:t}\right)\right\| &\lesssim \sum_{l = t-\tau}^t \lambda^{t-l}\rho^l  + \sum_{l = -\infty}^{t-\tau-1} \lambda^{t-l} \\ 
 &\lesssim \sum_{l=t-\tau}^t q^t + \sum_{l = -\infty}^{t-\tau-1} \lambda^{t-l} \quad \text{where} ~q = \max\{\lambda, \rho\} \in (0,1) \\
 &\lesssim \tau q^t + \lambda^\tau.
\end{align*}

\noindent Proof of the second statement: 
\begin{align*}
&\mathbb{E}^{\boldsymbol{\pi}^\mu} \left(\boldsymbol{z}_t\boldsymbol{\phi}_t^\top\right) - \mathbb{E}^\mu\left(\boldsymbol{z}_{t-\tau:t}\boldsymbol{\phi}_t^\top \right)\ =
    \mathbb{E}^{\boldsymbol{\pi}^\mu}\left(\sum_{l=-\infty}^t \lambda^{t-l}\boldsymbol{\phi}_l \boldsymbol{\phi}_t^\top \right) - \mathbb{E}^\mu\left(\sum_{l=t-\tau}^t \lambda^{t-l}\boldsymbol{\phi}_l \boldsymbol{\phi}_t^\top  \right)\\
    &= \sum_{l=-\infty}^t \lambda^{t-l}\mathbb{E}^{\boldsymbol{\pi}^\mu}\left(\boldsymbol{\phi}_l \boldsymbol{\phi}_t^\top \right) - \sum_{l=t-\tau}^t \lambda^{t-l}\mathbb{E}^\mu\left(\boldsymbol{\phi}_l \boldsymbol{\phi}_t^\top \right)\\
    &= \sum_{l=t-\tau}^t \lambda^{t-l}\left\{\mathbb{E}^{\boldsymbol{\pi}^\mu}\left(\boldsymbol{\phi}_l  \boldsymbol{\phi}_t^\top \right) - \mathbb{E}^\mu\left(\boldsymbol{\phi}_l \boldsymbol{\phi}_t^\top \right)\right\}+ \sum_{l=-\infty}^{t-\tau-1} \lambda^{t-l}\mathbb{E}^{\boldsymbol{\pi}^\mu} \left(\boldsymbol{\phi}_l \boldsymbol{\phi}_t^\top \right)\\
    & = \sum_{l=t-\tau}^t \lambda^{t-l}\left[\sum_{s \in \mathcal{S}} \left\lbrace\pi_s^\mu\boldsymbol{\phi}(s)\mathbb{E}^{\mu}\left(\boldsymbol{\phi}_t^\top |S^\mu_l = s \right) - p^\mu(S^\mu_l = s|S^\mu_0 =s_0)\boldsymbol{\phi}(s)\mathbb{E}^{\mu}\left(\boldsymbol{\phi}_t^\top |S^\mu_l = s \right)\right\rbrace\right] 
    \\ &\quad + \sum_{l=-\infty}^{t-\tau-1} \lambda^{t-l}\mathbb{E}^{\boldsymbol{\pi}^\mu} \left(\boldsymbol{\phi}_l\boldsymbol{\phi}_l^\top\right)
\end{align*}
By the same logic as in the first statement, we have
\begin{align*}
 \left\|\mathbb{E}^{\boldsymbol{\pi}^\mu} \left(\boldsymbol{z}_t\boldsymbol{\phi}_t^\top\right) - \mathbb{E}^\mu\left(\boldsymbol{z}_{t-\tau:t}\boldsymbol{\phi}_t^\top\right)\right\| \lesssim \sum_{l = t-\tau}^t \lambda^{t-l}\rho^l  + \sum_{l = -\infty}^{t-\tau-1} \lambda^{t-l} \lesssim \tau q^t + \lambda^\tau, 
\end{align*}
where $q = \max\{\lambda, \rho\} \in (0,1)$. \\

\noindent Proof of the third statement:
\begin{align*}
&\mathbb{E}^{\boldsymbol{\pi}^\mu} \left(\boldsymbol{z}_t\boldsymbol{\phi}_{t+1}^\top\right) - \mathbb{E}^\mu\left(\boldsymbol{z}_{t-\tau:t}\boldsymbol{\phi}_{t+1}^\top \right)\ =
    \mathbb{E}^{\boldsymbol{\pi}^\mu}\left(\sum_{l=-\infty}^t \lambda^{t-l}\boldsymbol{\phi}_l \boldsymbol{\phi}_{t+1}^\top \right) - \mathbb{E}^\mu\left(\sum_{l=t-\tau}^t \lambda^{t-l}\boldsymbol{\phi}_l \boldsymbol{\phi}_{t+1}^\top  \right)\\
    &= \sum_{l=-\infty}^t \lambda^{t-l}\mathbb{E}^{\boldsymbol{\pi}^\mu}\left(\boldsymbol{\phi}_l \boldsymbol{\phi}_{t+1}^\top \right) - \sum_{l=t-\tau}^t \lambda^{t-l}\mathbb{E}^\mu\left(\boldsymbol{\phi}_l \boldsymbol{\phi}_{t+1}^\top  \right)\\
    &= \sum_{l=t-\tau}^t \lambda^{t-l}\left\{\mathbb{E}^{\boldsymbol{\pi}^\mu}\left(\boldsymbol{\phi}_l  \boldsymbol{\phi}_{t+1}^\top \right) - \mathbb{E}^\mu\left(\boldsymbol{\phi}_l \boldsymbol{\phi}_{t+1}^\top \right)\right\}+ \sum_{l=-\infty}^{t-\tau-1} \lambda^{t-l}\mathbb{E}^{\boldsymbol{\pi}^\mu} \left(\boldsymbol{\phi}_l \boldsymbol{\phi}_{t+1}^\top \right)\\
    & = \sum_{l=t-\tau}^t \lambda^{t-l}\left[\sum_{s \in \mathcal{S}} \left\lbrace\pi_s^\mu\boldsymbol{\phi}(s)\mathbb{E}^{\mu}\left(\boldsymbol{\phi}_{t+1}^\top |S^\mu_l = s \right) - p^\mu\left(S^\mu_l = s|S^\mu_0=s_0\right)\boldsymbol{\phi}(s)\mathbb{E}^{\mu}\left(\boldsymbol{\phi}_{t+1}^\top |S^\mu_l = s \right)\right\rbrace\right]\\&\qquad+ \sum_{l=-\infty}^{t-\tau-1} \lambda^{t-l}\mathbb{E}^{\boldsymbol{\pi}^\mu} \left(\boldsymbol{\phi}_l\boldsymbol{\phi}_{l+1}^\top\right),
\end{align*}
where in the last equality, we use the law of iterated expectations. Following the proof of the first statement, we have
\begin{align*}
 \left\|\mathbb{E}^{\boldsymbol{\pi}^\mu} \left(\boldsymbol{z}_t\boldsymbol{\phi}_{t+1}^\top\right) - \mathbb{E}^\mu\left(\boldsymbol{z}_{t-\tau:t}\boldsymbol{\phi}_{t+1}^\top\right)\right\| \lesssim \sum_{l = t-\tau}^t \lambda^{t-l}\rho^l  + \sum_{l = -\infty}^{t-\tau-1} \lambda^{t-l} \lesssim \tau q^t + \lambda^\tau,
\end{align*}
where $q = \max\{\lambda, \rho\} \in (0,1)$.
\end{proof}

\noindent Subsequent lemma provides explicit bounds on the error incurred by replacing the steady-state reward expectation with its non steady-state version.

\begin{lemma}\label{LEMMA:ER_discrepancy_bound}
Given an initial state $s_0 \in \mathcal{S}$, suppose $\{S^\mu_t\}_{t\in \mathbb{N}}$ is uniformly geometrically ergodic with a rate parameter $\rho \in (0,1)$. For $\tau \in \{0,\cdots, t\}$,
\begin{enumerate}
    \item $|\mathbb{E}^{\boldsymbol{\pi}^\mu}(R^\mu_t) - \mathbb{E}^\mu(R^\mu_t)| \lesssim \rho^t$
    \item $\|\mathbb{E}^{\boldsymbol{\pi}^\mu}(R^\mu_t \boldsymbol{z}_t) - \mathbb{E}^\mu(R^\mu_t \boldsymbol{z}_{t-\tau:t}) \|\lesssim \tau q^t + \lambda^\tau$
\end{enumerate}
where $q:= \max\{\lambda, \rho\}$.
\end{lemma}
\begin{proof}
For the first statement, notice that
\begin{align*}
\mathbb{E}^{\boldsymbol{\pi}^\mu}(R^\mu_t) - \mathbb{E}^\mu(R^\mu_t) = \sum_{s \in \mathcal{S}} r\{s, \mu(s)\} \left\{\pi_s^\mu -  p^\mu(S^\mu_t = s |S^\mu_0 = s_0)\right\} \Longrightarrow |\mathbb{E}^{\boldsymbol{\pi}^\mu}(R^\mu_t) - \mathbb{E}^\mu(R^\mu_t) |\lesssim \rho^k,
\end{align*}
where the last inequality follows from the uniform geometric ergodicity of the Markov chain $\{S^\mu_t\}_{t \in \mathbb{N}}$. For the second statement, we again leverage the steady-state expression of the eligibility trace $\boldsymbol{z}_t = \sum_{l=-\infty}^t \lambda^{t-l}\boldsymbol{\phi}_l$ whenever the expectation is with respect to the steady-state distribution induced by the policy $\mu$. 
\begin{align*}
\mathbb{E}^{\boldsymbol{\pi}^\mu} (R^\mu_t \boldsymbol{z}_t) - \mathbb{E}^\mu(R^\mu_t \boldsymbol{z}_{t-\tau:t}) &= \mathbb{E}^{\boldsymbol{\pi}^\mu}\left(R^\mu_t\sum_{l=-\infty}^t \lambda^{t-l}\boldsymbol{\phi}_l\right) - \mathbb{E}^\mu\left(R^\mu_t\sum_{l=t-\tau}^t \lambda^{t-l}\boldsymbol{\phi}_l\right) \\
    &= \sum_{l=-\infty}^t \lambda^{t-l}\mathbb{E}^{\boldsymbol{\pi}^\mu}\left(R^\mu_t\boldsymbol{\phi}_l\right) - \sum_{l=t-\tau}^t \lambda^{t-l}\mathbb{E}^\mu\left(R^\mu_t\boldsymbol{\phi}_l\right) \\
    &= \sum_{l=t-\tau}^t \lambda^{t-l}\left\{\mathbb{E}^{\boldsymbol{\pi}^\mu}\left(R^\mu_t\boldsymbol{\phi}_l\right) - \mathbb{E}^\mu\left(R^\mu_t\boldsymbol{\phi}_l\right)\right\} + \sum_{l=-\infty}^{t-\tau-1} \lambda^{t-l}\mathbb{E}^{\boldsymbol{\pi}^\mu}\left(R^\mu_t\boldsymbol{\phi}_l\right) \\
    &= \sum_{l=t-\tau}^t \lambda^{t-l}\left[\sum_{s \in \mathcal{S}}\{\pi^\mu_s-p^\mu(S^\mu_l=s|S^\mu_0=s_0)\}\boldsymbol{\phi}(s) \mathbb{E}^\mu(R^\mu_t |S^\mu_l = s) \right] \\
    &\quad +\sum_{l=-\infty}^{t-\tau-1} \lambda^{t-l}\mathbb{E}^{\boldsymbol{\pi}^\mu}\left(R^\mu_t\boldsymbol{\phi}_l\right) .
\end{align*}
Taking the norm on both sides with the uniform geometric ergodicity of the Markov chain $\{S^\mu_t\}_{t \in \mathbb{N}}$, we have
\begin{align*}
    \| \mathbb{E}^{\boldsymbol{\pi}^\mu} (R^\mu_t \boldsymbol{z}_t) - \mathbb{E}^\mu(R^\mu_t \boldsymbol{z}_{t-\tau:t}) \| &\lesssim \sum_{l=t-\tau}^t \lambda^{t-l} \rho^l + \sum_{l=-\infty}^{t-\tau-1} \lambda^{t-l}  \\
    &\lesssim \tau q^t + \lambda^\tau,
\end{align*}
where the last inequality follows from $q = \max\{\lambda, \rho\}$.
\end{proof}

\noindent Lemma below establishes uniform bounds on the norm of $\boldsymbol{A}_t$ and $\boldsymbol{b}_t$, which will appear in the finite-time error bounds.

\begin{lemma}\label{LEMMA:A_b_bound}
For all $t \in \mathbb{N}$, $\|\boldsymbol{A}_t\| \le A_{\max}$ and $\|\boldsymbol{b}_t\| \le b_{\max}$ for some constants $A_{\max}, b_{\max} \in \mathbb{R}_{>0}$    
\end{lemma}
\begin{proof}
\begin{align*}
\|\boldsymbol{A}_t\|^2 &\le \|\boldsymbol{A}_t\|^2_F \le c_\alpha^2 + \|\boldsymbol{z}_t\|^2 + \|\boldsymbol{z}_t\|^2 \|\boldsymbol{\phi}_{t+1}^\top - \boldsymbol{\phi}_t^\top\|^2 \le c_\alpha^2 + \frac{5}{\left(1-\lambda \right)^2}  =: A^2_{\max} 
\\
\|\boldsymbol{b}_t\|^2 &\le c_\alpha^2 + \frac{1}{(1-\lambda)^2} =: b^2_{\max},
\end{align*}
where in the equality for $\|\boldsymbol{b}_t\|^2$, we used the fact $|R^\mu_t| \le 1$ for all $t \in \mathbb{N}.$
\end{proof}

\begin{lemma}\label{LEMMA:D_Dt_bound}
For all $t \in \mathbb{N}$, $\|\boldsymbol{D}_t-\underline{\boldsymbol{D}_t}\| \le \frac{(1+c_\alpha)\beta_t}{(1-\lambda)^2}$.    
\end{lemma}
\begin{proof}
By the definition of the operator norm, we know
\begin{align*}
\|\boldsymbol{D}_t-\underline{\boldsymbol{D}_t}\| \le \max \left\{\left|\frac{1}{1+c_\alpha \beta_t} - \gamma_t \right|,  \left|\frac{1}{1+\beta_t\|\boldsymbol{z}_t\|^2} - \gamma_t \right|\right\}.
\end{align*}
Note that 
\begin{align*}
 \left|\frac{1}{1+c_\alpha \beta_t} - \gamma_t  \right| \le \left |\frac{1}{1+c_\alpha \beta_t} -  \frac{(1-\lambda)^2}{(1-\lambda)^2+ \beta_t} \right| \le \left|\frac{\{1-(1-\lambda)^2 c_\alpha\}\beta_t}{(1-\lambda)^2} \right| \le \frac{(1+c_\alpha)\beta_t}{(1-\lambda)^2}.
\end{align*}
Since 
\begin{align*}
\left |\frac{1}{1+\beta_t\|\boldsymbol{z}_t\|^2} - \frac{1}{1+c_\alpha \beta_t} \right| &\le \left|\frac{c_\alpha\beta_t-\beta_t \|\boldsymbol{z}_t\|^2}{(1+\beta_t\|\boldsymbol{z}_t\|^2)(1+c_\alpha\beta_t)} \right| \le \left\{c_\alpha+\frac{1}{(1-\lambda)^2}\right\}\beta_t \\
\left |\frac{1}{1+\beta_t\|\boldsymbol{z}_t\|^2}  -  \frac{(1-\lambda)^2}{(1-\lambda)^2+ \beta_t} \right| &\le \left|\frac{\{1-(1-\lambda)^2 \|\boldsymbol{z}_t\|^2\}\beta_t}{(1-\lambda)^2} \right| \le \frac{\beta_t}{(1-\lambda)^2} 
\end{align*}
we have
$$
\left|\frac{1}{1+\beta_t\|\boldsymbol{z}_t\|^2} - \gamma_t  \right| \le\frac{(1+c_\alpha)\beta_t}{(1-\lambda)^2},
$$
which yields the desired bound.
\end{proof}

\noindent The following lemma establishes a uniform bound on the norm of the function $\zeta$. 

\begin{lemma}\label{LEMMA:zeta_bound}
For all $t \in \mathbb{N}$, $\boldsymbol{\Theta} \in \{\boldsymbol{\Xi}:\|\boldsymbol{\Xi}\| \le R_{\boldsymbol{\Theta}}\}$, 
$\|\zeta_t(\boldsymbol{\Theta}, \boldsymbol{X}_t) \| \le  C_\zeta$ for some constant $C_\zeta > 0$. 
\end{lemma}
\begin{proof}
Note that $\|\underline{\boldsymbol{D}_t}\| \le 1$, $\|\boldsymbol{A}_t\| \le A_{\max}$ and $\|\boldsymbol{b}_t\| \le b_{\max}$,
\begin{align*}
\|\zeta_t(\boldsymbol{\Theta}, \boldsymbol{X}_t) \| &\le \left\|\boldsymbol{A}_t -\boldsymbol{A}\right\|\left\|\boldsymbol{\Theta^*} \right\| \left\|\underline{\boldsymbol{D}_t}\right\| \|\boldsymbol{\Theta^*}-\boldsymbol{\Theta}\| +\left\|\boldsymbol{\boldsymbol{b}_t} -\boldsymbol{b}\right\| \left\|\underline{\boldsymbol{D}_t}\right\| \|\boldsymbol{\Theta^*}-\boldsymbol{\Theta}\|  \\
&\le 2\left\|\boldsymbol{A}_t - \boldsymbol{A}\right\|R_{\boldsymbol{\Theta}}^2 + 2\left\|\boldsymbol{b}_t - \boldsymbol{b}\right\| R_{\boldsymbol{\Theta}}\\
&\le 4A_{\max} R_{\boldsymbol{\Theta}}^2 + 4 b_{\max} R_{\boldsymbol{\Theta}} =: C_\zeta
\end{align*}
\end{proof}

\noindent Two lemmas below respectively establish Lipschitzness of the function $\zeta$ with respect to $\boldsymbol{\Theta}$ component and the deviance bound with respect to $\boldsymbol{X}$ component. 

\begin{lemma}\label{LEMMA:zeta_lipschitz}
For all $t \in \mathbb{N}$, $\boldsymbol{\Theta_1}, \boldsymbol{\Theta_2} \in \{\boldsymbol{\Xi}:\|\boldsymbol{\Xi}\| \le R_{\boldsymbol{\Theta}}\}$, 
$$
\left|\zeta_t(\boldsymbol{\Theta_1}, \boldsymbol{X}_t) - \zeta_t(\boldsymbol{\Theta_2}, \boldsymbol{X}_t)\right| \le L_\zeta \|\boldsymbol{\Theta_1} - \boldsymbol{\Theta_2}\|,
$$ 
for some constant $L_\zeta > 0$.
\end{lemma}
\begin{proof}
\begin{align*}
    \left|\zeta_t(\boldsymbol{\Theta_1}, \boldsymbol{X}_t) - \zeta_t(\boldsymbol{\Theta_2}, \boldsymbol{X}_t)\right| &= \left\langle \left(\boldsymbol{A}_t -\boldsymbol{A}\right)\boldsymbol{\Theta^*}, \underline{\boldsymbol{D}_t}(\boldsymbol{\Theta_2} - \boldsymbol{\Theta_1}) \right\rangle + \left\langle\boldsymbol{b}_t -\boldsymbol{b}, \underline{\boldsymbol{D}_t}(\boldsymbol{\Theta_2} - \boldsymbol{\Theta_1}) \right\rangle \\
    &\le \left\|\boldsymbol{A}_t -\boldsymbol{A}\right\| \|\boldsymbol{\Theta^*}\| \|\underline{\boldsymbol{D}_t}\| \|\boldsymbol{\Theta_2} - \boldsymbol{\Theta_1}\| + \left\|\boldsymbol{b}_t -\boldsymbol{b}\right\| \|\underline{\boldsymbol{D}_t}\| \|\boldsymbol{\Theta_2} - \boldsymbol{\Theta_1}\| \\
    &\le \left\|\boldsymbol{A}_t - \boldsymbol{A}\right\|R_{\boldsymbol{\Theta}} \|\boldsymbol{\Theta_2} - \boldsymbol{\Theta_1}\| + \left\|\boldsymbol{b}_t - \boldsymbol{b}\right\| \|\boldsymbol{\Theta_2} - \boldsymbol{\Theta_1}\|\\
    &\le (2A_{\max} R_{\boldsymbol{\Theta}} + 2 b_{\max })\|\boldsymbol{\Theta_2} - \boldsymbol{\Theta_1}\| \\
    & =: L_\zeta \|\boldsymbol{\Theta_2} - \boldsymbol{\Theta_1}\| 
\end{align*}
where in the first inequality, we used the Cauchy-Schwarz inequality. The second and third inequalities follow from $\|\underline{\boldsymbol{D}_t}\| \le 1$, $\|\boldsymbol{A}_t\| \le A_{\max}$ and $\|\boldsymbol{b}_t\| \le b_{\max}$.
\end{proof}

\begin{lemma}\label{LEMMA:zeta_x_deviance_bound}
For all $t \in \mathbb{N}$, $\boldsymbol{\Theta} \in \{\boldsymbol{\Xi}:\|\boldsymbol{\Xi}\| \le R_{\boldsymbol{\Theta}}\}$, let $\tau \in \{0,\cdots, t\}$, then we have
$$
\left | \zeta_t(\boldsymbol{\Theta}, \boldsymbol{X}_t) - \zeta_t(\boldsymbol{\Theta}, \boldsymbol{X}_{t-\tau:t} )\right|\lesssim \lambda^\tau.
$$  
\end{lemma}
\begin{proof}
Note that
\begin{align}
\left | \zeta_t(\boldsymbol{\Theta}, \boldsymbol{X}_t) - \zeta_t(\boldsymbol{\Theta}, \boldsymbol{X}_{t-\tau:t} )\right| &= \left\langle \left(\boldsymbol{A}_t -\boldsymbol{A}_{t-\tau:t}\right)\boldsymbol{\Theta^*}, \underline{\boldsymbol{D}_t}(\boldsymbol{\Theta^*} - \boldsymbol{\Theta}) \right\rangle + \left\langle \boldsymbol{b}_t-\boldsymbol{b}_{t-\tau:t}, \underline{\boldsymbol{D}_t}(\boldsymbol{\Theta^*} - \boldsymbol{\Theta}) \right\rangle \nonumber \\
&\le 2\left\|\boldsymbol{A}_t - \boldsymbol{A}_{t-\tau:t}\right\| R_{\boldsymbol{\Theta}}^2 + 2\left\| \boldsymbol{b}_t-\boldsymbol{b}_{t-\tau:t} \right\|R_{\boldsymbol{\Theta}}. \label{zeta_x_discrep_bound_prelim}\end{align}
where we used $\|\boldsymbol{\Theta^*}\| \le R_{\boldsymbol{\Theta}}$, $\|\boldsymbol{\Theta}\| \le R_{\boldsymbol{\Theta}}$, $\|\underline{\boldsymbol{D}_t}\| \le 1$. To obtain a bound on $\left\|\boldsymbol{A}_t - \boldsymbol{A}_{t-\tau:t}\right\|$ and $\left\| \boldsymbol{b}_t-\boldsymbol{b}_{t-\tau:t}\right\|$, note that 
$$
\|\boldsymbol{z}_t - \boldsymbol{z}_{t-\tau:t}\| \le \sum_{l=0}^{t-\tau} \lambda^{t-l} \le \lambda^\tau / (1-\lambda).
$$
Now consider
$$
\boldsymbol{A}_t - \boldsymbol{A}_{t-\tau:t} = \begin{bmatrix}
    0 & 0 \\
    -(\boldsymbol{z}_t-\boldsymbol{z}_{t-\tau:t}) & (\boldsymbol{z}_t-\boldsymbol{z}_{t-\tau:t})(\boldsymbol{\phi}_{t+1}^\top-\boldsymbol{\phi}_t^\top),
\end{bmatrix}
$$
whose operator norm satisfies the following bound
\begin{equation}\label{A_norm_bound}
    \|\boldsymbol{A}_t- \boldsymbol{A}_{t-\tau:t}\|^2 \le 
    \|\boldsymbol{A}_t- \boldsymbol{A}_{t-\tau:t}\|_F^2 \le \|\boldsymbol{z}_t-\boldsymbol{z}_{t-\tau:t}\|^2 +\|\boldsymbol{z}_t-\boldsymbol{z}_{t-\tau:t}\|^2 \left\|\boldsymbol{\phi}_{t+1}^\top -\boldsymbol{\phi}_t^\top\right\|^2 \le 5\lambda^{2\tau}/(1-\lambda)^2.
\end{equation}
Similarly, consider
$$
\boldsymbol{b}_t- \boldsymbol{b}_{t-\tau:t} = \begin{bmatrix}
    0  \\
    R^\mu_t(\boldsymbol{z}_t-\boldsymbol{z}_{t-\tau:t})
\end{bmatrix},
$$
for which the Euclidean norm is bounded as follows
\begin{equation}\label{b_norm_bound}
    \|\boldsymbol{b}_t- \boldsymbol{b}_{t-\tau:t}\| \le  \|\boldsymbol{z}_t-\boldsymbol{z}_{t-\tau:t}\| \le \lambda^{\tau}/(1-\lambda).
\end{equation}
Plugging \eqref{A_norm_bound} and \eqref{b_norm_bound} into \eqref{zeta_x_discrep_bound_prelim}, we have
$$
\left | \zeta_t(\boldsymbol{\Theta}, \boldsymbol{X}_t) - \zeta_t(\boldsymbol{\Theta}, \boldsymbol{X}_{t-\tau:t} )\right| \lesssim \lambda^\tau.
$$
\end{proof}

\noindent As we did for the function $\zeta_t$, we next establish boundedness, Lipschitzness of the function $\xi_t$ with respect to $\boldsymbol{\Theta}$ as well as the deviance bound with respect to $\boldsymbol{X}$ component.
\begin{lemma}\label{LEMMA:xi_bound}
For all $t \in \mathbb{N}, ~\boldsymbol{\Theta} \in \{\boldsymbol{\Xi}:\|\boldsymbol{\Xi}\| \le R_{\boldsymbol{\Theta}}\}$, 
$|\xi_t(\boldsymbol{\Theta}, \boldsymbol{X}_t) | \le  C_\xi$ for some constant $C_\xi > 0$. 
\end{lemma}
\begin{proof}
\begin{align*}
|\xi_t(\boldsymbol{\Theta}, \boldsymbol{X}_t) | &= \left|(\boldsymbol{\Theta^*}-\boldsymbol{\Theta})\left(\boldsymbol{A}_t -\boldsymbol{A}\right)^\top \underline{\boldsymbol{D}_t}(\boldsymbol{\Theta^*}-\boldsymbol{\Theta}) \right|\\ &\le \left\|\boldsymbol{\Theta^*}-\boldsymbol{\Theta}\right\| \left\|\boldsymbol{A}_t -\boldsymbol{A}\right\| \left\|\underline{\boldsymbol{D}_t}\right\| \left\|\boldsymbol{\Theta^*}-\boldsymbol{\Theta} \right\|\\
&\le 8R_{\boldsymbol{\Theta}}^2 A_{\max} = : C_\xi
\end{align*}
where in the first inequality, we used Cauchy-Schwarz inequality and in the second inequality, we used $\|\underline{\boldsymbol{D}_t}\| \le 1$, $\|\boldsymbol{A}_t\| \le A_{\max}$ and $\|\boldsymbol{A}\| \le A_{\max}$.  
\end{proof}

\noindent Two lemmas below respectively establish Lipschitzness of the function $\xi_t$ with respect to $\boldsymbol{\Theta}$ component and the deviance bound with respect to $\boldsymbol{X}$ component. 

\begin{lemma}\label{LEMMA:xi_lipschitz}
For all $t \in \mathbb{N}, ~\boldsymbol{\Theta_1}, \boldsymbol{\Theta_2} \in \{\boldsymbol{\Xi}:\|\boldsymbol{\Xi}\| \le R_{\boldsymbol{\Theta}}\}$, 
$$
\left|\xi_t(\boldsymbol{\Theta_1}, \boldsymbol{X}_t) - \xi_t(\boldsymbol{\Theta_2}, \boldsymbol{X}_t)\right| \le L_\xi \|\boldsymbol{\Theta_1} - \boldsymbol{\Theta_2}\|,
$$ 
for some constant $L_\xi > 0$.
\end{lemma}
\begin{proof}
\begin{align*}
    \left|\xi_t(\boldsymbol{\Theta_1}, \boldsymbol{X}_t) - \xi_t(\boldsymbol{\Theta_2}, \boldsymbol{X}_t)\right| &= \left|(\boldsymbol{\Theta^*}-\boldsymbol{\Theta_1})^\top \left(\boldsymbol{A}_t-\boldsymbol{A} \right)^\top \underline{\boldsymbol{D}_t}(\boldsymbol{\Theta^*}-\boldsymbol{\Theta_1})-(\boldsymbol{\Theta^*}-\boldsymbol{\Theta_2})^\top \left(\boldsymbol{A}_t-\boldsymbol{A} \right)^\top \underline{\boldsymbol{D}_t}(\boldsymbol{\Theta^*}-\boldsymbol{\Theta_2})\right| \\
    &\le \left|(\boldsymbol{\Theta^*}-\boldsymbol{\Theta_1})^\top \left(\boldsymbol{A}_t-\boldsymbol{A} \right)^\top \underline{\boldsymbol{D}_t}(\boldsymbol{\Theta^*}-\boldsymbol{\Theta_1})-(\boldsymbol{\Theta^*}-\boldsymbol{\Theta_2})^\top \left(\boldsymbol{A}_t-\boldsymbol{A} \right)^\top \underline{\boldsymbol{D}_t}(\boldsymbol{\Theta^*}-\boldsymbol{\Theta_1})\right| \\
    &\quad + \left|(\boldsymbol{\Theta^*}-\boldsymbol{\Theta_2})^\top \left(\boldsymbol{A}_t-\boldsymbol{A} \right)^\top \underline{\boldsymbol{D}_t}(\boldsymbol{\Theta^*}-\boldsymbol{\Theta_1})-(\boldsymbol{\Theta^*}-\boldsymbol{\Theta_2})^\top \left(\boldsymbol{A}_t-\boldsymbol{A} \right)^\top \underline{\boldsymbol{D}_t}(\boldsymbol{\Theta^*}-\boldsymbol{\Theta_2})\right| \\
    &= \left|(\boldsymbol{\Theta_2}-\boldsymbol{\Theta_1})^\top \left(\boldsymbol{A}_t-\boldsymbol{A} \right)^\top \underline{\boldsymbol{D}_t}(\boldsymbol{\Theta^*}-\boldsymbol{\Theta_1})\right| + \left|(\boldsymbol{\Theta^*}-\boldsymbol{\Theta_2})^\top \left(\boldsymbol{A}_t-\boldsymbol{A} \right)^\top \underline{\boldsymbol{D}_t}(\boldsymbol{\Theta_2}-\boldsymbol{\Theta_1})\right| \\
    &\le \left\|\boldsymbol{\Theta_2}-\boldsymbol{\Theta_1}\right\| \left\|\boldsymbol{A}_t-\boldsymbol{A} \right\| \|\underline{\boldsymbol{D}_t}\|\left\|\boldsymbol{\Theta^*}-\boldsymbol{\Theta_1}\right\| + \left\|\boldsymbol{\Theta^*}-\boldsymbol{\Theta_2}\right\| \left\|\boldsymbol{A}_t-\boldsymbol{A} \right\|\|\underline{\boldsymbol{D}_t}\|\left\|\boldsymbol{\Theta_2}-\boldsymbol{\Theta_1}\right\| \\
    &\le 4R_{\boldsymbol{\Theta}} A_{\max} \left\|\boldsymbol{\Theta_2}-\boldsymbol{\Theta_1}\right\| \\
    &=: L_{\xi}\left\|\boldsymbol{\Theta_1}-\boldsymbol{\Theta_2}\right\| 
\end{align*}
where in the first inequality, we used the triangle inequality. The rest of the inequalities follow from Cauchy-Schwarz inequality, $\|\underline{\boldsymbol{D}_t}\| \le 1$, $\|\boldsymbol{A}_t\| \le A_{\max}$ and $\|\boldsymbol{A}\| \le A_{\max}$.
\end{proof}

\begin{lemma}\label{LEMMA:xi_x_deviance_bound}
For all $t \in \mathbb{N}, ~\boldsymbol{\Theta} \in \{\boldsymbol{\Xi}:\|\boldsymbol{\Xi}\| \le R_{\boldsymbol{\Theta}}\}$, let $\tau \in \{0,\cdots, t\}$, then we have
$$
\left | \xi_t(\boldsymbol{\Theta}, \boldsymbol{X}_t) - \xi_t(\boldsymbol{\Theta}, \boldsymbol{X}_{t-\tau:t} )\right|\lesssim \lambda^\tau.
$$  
\end{lemma}
\begin{proof}
Note that
\begin{align*}
\left | \xi_t(\boldsymbol{\Theta}, \boldsymbol{X}_t) - \xi_t(\boldsymbol{\Theta}, \boldsymbol{X}_{t-\tau:t} )\right| &= \left|(\boldsymbol{\Theta^*}-\boldsymbol{\Theta})^\top (\boldsymbol{A}_t -\boldsymbol{A})^\top \underline{\boldsymbol{D}_t} (\boldsymbol{\Theta^*}-\boldsymbol{\Theta})-(\boldsymbol{\Theta^*}-\boldsymbol{\Theta})^\top (\boldsymbol{A}_{t-\tau:t} -\boldsymbol{A})^\top \underline{\boldsymbol{D}_t}(\boldsymbol{\Theta^*}-\boldsymbol{\Theta}) \right|\\&= \left|(\boldsymbol{\Theta^*}-\boldsymbol{\Theta})^\top (\boldsymbol{A}_t -\boldsymbol{A}_{t-\tau:t})^\top \underline{\boldsymbol{D}_t} (\boldsymbol{\Theta^*}-\boldsymbol{\Theta}) \right| \\
& \le  \left\|\boldsymbol{\Theta^*}-\boldsymbol{\Theta}\right\| \left\|\boldsymbol{A}_t -\boldsymbol{A}_{t-\tau:t}\right\| \|\underline{\boldsymbol{D}_t}\| \left\|\boldsymbol{\Theta^*}-\boldsymbol{\Theta}\right\| \\
&\lesssim \lambda^\tau
\end{align*}
where the first inequality follows from the Cauchy-Schwarz. The last inequality follows from \eqref{A_norm_bound} with the fact $\|\underline{\boldsymbol{D}_t}\| \le 1.$
\end{proof}

\noindent The next two lemmas will serve as key ingredients in establishing bounds on the non-steady state expectations of $\zeta_t$ and $\xi_t$ terms.

\begin{lemma}\label{LEMMA:exp_A_bound}
For all $t \in \mathbb{N}$, let $\tau \in \{0, \cdots, t\}$, then we have
$$
\|\mathbb{E}^\mu\left(\boldsymbol{A}_{t-\tau:t} \right) - \boldsymbol{A}\| \lesssim \tau q^t + \lambda^\tau
$$ 
where $q:= \max\{\lambda, \rho\}.$
\end{lemma}
\begin{proof}
Note that
\begin{align*}
   \mathbb{E}^\mu\left(\boldsymbol{A}_{t-\tau:t} \right) - \boldsymbol{A} &= \begin{bmatrix}
       -c_\alpha & 0 \\
       -\mathbb{E}^\mu(\boldsymbol{z}_{t-\tau:t}) & \mathbb{E}^\mu\{\boldsymbol{z}_t (\boldsymbol{\phi}_{t+1}^\top- \boldsymbol{\phi}_t^\top)\} 
   \end{bmatrix} - \begin{bmatrix}
       -c_\alpha & 0 \\
       -\mathbb{E}^{\boldsymbol{\pi}^\mu}(\boldsymbol{z}_t) & \mathbb{E}^{\boldsymbol{\pi}^\mu}\{\boldsymbol{z}_t (\boldsymbol{\phi}_{t+1}^\top- \boldsymbol{\phi}_t^\top)\} 
   \end{bmatrix}\\
   &= \begin{bmatrix}
       0 & 0 \\
      \mathbb{E}^{\boldsymbol{\pi}^\mu}(\boldsymbol{z}_t) -\mathbb{E}^\mu(\boldsymbol{z}_{t-\tau:t}) & \mathbb{E}^\mu\{\boldsymbol{z}_t (\boldsymbol{\phi}_{t+1}^\top- \boldsymbol{\phi}_t^\top)\}-\mathbb{E}^{\boldsymbol{\pi}^\mu}\{\boldsymbol{z}_t (\boldsymbol{\phi}_{t+1}^\top- \boldsymbol{\phi}_t^\top)\} 
   \end{bmatrix}.
\end{align*}
Therefore,
\begin{align*}
\left\|\mathbb{E}^\mu(\boldsymbol{A}_{t-\tau:t}) - \boldsymbol{A}\right\|^2 &\le \left\|\mathbb{E}^\mu(\boldsymbol{A}_{t-\tau:t}) - \boldsymbol{A}\right\|_F^2 \\ &\le \left\|\mathbb{E}^{\boldsymbol{\pi}^\mu}(\boldsymbol{z}_t) - \mathbb{E}^\mu(\boldsymbol{z}_{t-\tau:t})\right\|^2 + \left\|\mathbb{E}^\mu\{\boldsymbol{z}_t (\boldsymbol{\phi}_{t+1}^\top- \boldsymbol{\phi}_t^\top)\}-\mathbb{E}^{\boldsymbol{\pi}^\mu}\{\boldsymbol{z}_t (\boldsymbol{\phi}_{t+1}^\top- \boldsymbol{\phi}_t^\top)\}\right\|^2  \\
    &\le \left\|\mathbb{E}^{\boldsymbol{\pi}^\mu}(\boldsymbol{z}_t) - \mathbb{E}^\mu(\boldsymbol{z}_{t-\tau:t})\right\|^2 + 2\left\|\mathbb{E}^\mu(\boldsymbol{z}_t \boldsymbol{\phi}_{t+1}^\top)-\mathbb{E}^{\boldsymbol{\pi}^\mu}(\boldsymbol{z}_t \boldsymbol{\phi}_{t+1}^\top) \right\|^2 + 2\left\|\mathbb{E}^{\boldsymbol{\pi}^\mu}(\boldsymbol{z}_t \boldsymbol{\phi}_t^\top)-\mathbb{E}^\mu(\boldsymbol{z}_t \boldsymbol{\phi}_t^\top) \right\|^2 \\
    &\lesssim \left(\tau q^t + \lambda^\tau \right)^2,
\end{align*}
where the last line follows from Lemma \ref{LEMMA:EL_discrepancy_bound}.
\end{proof}

\begin{lemma}\label{LEMMA:exp_b_bound}
For all $t \in \mathbb{N}$, let $\tau \in \{0, \cdots, t\}$, then we have
$$\|\mathbb{E}^\mu\left(\boldsymbol{b}_{t-\tau:t} \right) - \boldsymbol{b}\| \lesssim \tau q^t + \lambda^\tau,
$$
where $q= \max\{\lambda, \rho\}.$
\end{lemma}
\begin{proof}
    Note that
\begin{align*}
   \mathbb{E}^\mu\left(\boldsymbol{b}_{t-\tau:t} \right) - \boldsymbol{b} &= \begin{bmatrix}
       c_\alpha \mathbb{E}^\mu(R^\mu_t) - c_\alpha \mathbb{E}^{\boldsymbol{\pi}^\mu}(R^\mu_t)\\
       \mathbb{E}^\mu(R^\mu_t \boldsymbol{z}_{t-\tau:t}) - \mathbb{E}^{\boldsymbol{\pi}^\mu}(R^\mu_t \boldsymbol{z}_t)
   \end{bmatrix}
\end{align*}
Therefore,
\begin{align*}
\left\|\mathbb{E}^\mu\left(\boldsymbol{b}_{t-\tau:t} \right) - \boldsymbol{b}\right\|^2 &\le c_\alpha^2 \left|\mathbb{E}^\mu(R_{t}) - \mathbb{E}^{\boldsymbol{\pi}^\mu}(R^\mu_t)\right|^2 + \left\|\mathbb{E}^\mu(R_{t}\boldsymbol{z}_{t-\tau:t}) - \mathbb{E}^{\boldsymbol{\pi}^\mu}(R^\mu_t\boldsymbol{z}_t)\right\|^2  \\
    &\lesssim c_\alpha^2 \rho^{2t}+ \left(\tau q^t + \lambda^\tau \right)^2 \\
    &\lesssim \left(\tau q^t + \lambda^\tau \right)^2
\end{align*}
where the second inequality follows from Lemma \ref{LEMMA:ER_discrepancy_bound}.
\end{proof}

\noindent We now establish bounds on the expectation of $\zeta_t$.

\begin{lemma}\label{LEMMA:expect_zeta_bound}
Suppose $(\beta_t)_{t \in \mathbb{N}}$ is a non-increasing sequence and $q= \max\{\lambda, \rho\}$. Given $t \in \mathbb{N}$, suppose $i > 2\tau_{\beta_t}$ then,
\begin{align*}
    \left| \mathbb{E}^\mu\zeta_i(\widehat{\boldsymbol{\Theta}}_i, \boldsymbol{X}_i) \right| \lesssim \tau_{\beta_t}\beta_{i-2\tau_{\beta_t}} + \tau_{\beta_t} q^i + q^{\tau_{\beta_t}}.
\end{align*}
Otherwise, 
\begin{align*}
    \left| \mathbb{E}^\mu\zeta_i(\widehat{\boldsymbol{\Theta}}_i, \boldsymbol{X}_i) \right| \lesssim \tau_{\beta_t} \beta_0 + iq^i.
\end{align*}
\end{lemma}
\begin{proof}
We begin by considering the case $i > 2\tau_{\beta_t}$. From the triangle inequality, we have
\begin{align}
    \left| \mathbb{E}^\mu\zeta_i(\widehat{\boldsymbol{\Theta}}_i, \boldsymbol{X}_i) \right| &\le \left| \mathbb{E}^\mu\zeta_i(\widehat{\boldsymbol{\Theta}}_i, \boldsymbol{X}_i) - \mathbb{E}^\mu\zeta_i(\widehat{\boldsymbol{\Theta}}_{i-2\tau_{\beta_t}}, \boldsymbol{X}_i)\right| \label{zeta_first_case_first_term}\\ &\quad + \left| \mathbb{E}^\mu\zeta_i(\widehat{\boldsymbol{\Theta}}_{i-2\tau_{\beta_t}}, \boldsymbol{X}_i) - \mathbb{E}^\mu\zeta_i(\widehat{\boldsymbol{\Theta}}_{i-2\tau_{\beta_t}}, \boldsymbol{X}_{i-\tau_{\beta_t}:i}) \right| \label{zeta_first_case_second_term}\\ &\quad + \left |\mathbb{E}^\mu\zeta_i(\widehat{\boldsymbol{\Theta}}_{i-2\tau_{\beta_t}}, \boldsymbol{X}_{i-\tau_{\beta_t}:i}) \right| \label{zeta_first_case_third_term}
\end{align}
To obtain an upper bound of \eqref{zeta_first_case_first_term}, note that
\begin{align}
    &\left| \mathbb{E}^\mu\zeta_i(\widehat{\boldsymbol{\Theta}}_i, \boldsymbol{X}_i) - \mathbb{E}^\mu\zeta_i(\widehat{\boldsymbol{\Theta}}_{i-2\tau_{\beta_t}}, \boldsymbol{X}_i)\right|\nonumber\\ &\le \left| \mathbb{E}^\mu\zeta_i(\widehat{\boldsymbol{\Theta}}_i, \boldsymbol{X}_i) - \mathbb{E}^\mu\zeta_i(\boldsymbol{\Theta}_{i-1}, \boldsymbol{X}_i)\right|  + \cdots +  \left| \mathbb{E}^\mu\zeta_i(\widehat{\boldsymbol{\Theta}}_{i-2\tau_{\beta_t}+1}, \boldsymbol{X}_i) - \mathbb{E}^\mu\zeta_i(\widehat{\boldsymbol{\Theta}}_{i-2\tau_{\beta_t}}, \boldsymbol{X}_i)\right|\nonumber \\
    &\le L_\zeta \sum_{l=i-2\tau_{\beta_t}}^{i-1} \left\|\widehat{\boldsymbol{\Theta}}_{l+1} - \widehat{\boldsymbol{\Theta}}_{l}\right\| \label{zeta_consec_bound} 
\end{align}
where we used Lemma \ref{LEMMA:zeta_lipschitz} in the second inequality. For $\left\|\widehat{\boldsymbol{\Theta}}_{l+1} - \widehat{\boldsymbol{\Theta}}_{l}\right\|$, we observe the following inequality
\begin{align}
\left\|\widehat{\boldsymbol{\Theta}}_{l+1} - \widehat{\boldsymbol{\Theta}}_{l}\right\| &= \left\|\boldsymbol{\Pi}\left[\Pi_{R_{\boldsymbol{\Theta}}}\left\{\widehat{\boldsymbol{\Theta}}_l + \beta_l \underline{\boldsymbol{D}_l}\left(\boldsymbol{A}_l\widehat{\boldsymbol{\Theta}}_l + \boldsymbol{b}_l\right) \right\}\right] - \boldsymbol{\Pi}\left(\Pi_{R_{\boldsymbol{\Theta}}}\widehat{\boldsymbol{\Theta}}_l\right)\right\| \nonumber \\&\le \left\|\Pi_{R_{\boldsymbol{\Theta}}}\left\{\widehat{\boldsymbol{\Theta}}_l + \beta_l \underline{\boldsymbol{D}_l}\left(\boldsymbol{A}_l\widehat{\boldsymbol{\Theta}}_l + \boldsymbol{b}_l\right) \right\} - \Pi_{R_{\boldsymbol{\Theta}}}\widehat{\boldsymbol{\Theta}}_l  \right\| \nonumber \\
 &\le \left\|\widehat{\boldsymbol{\Theta}}_l + \beta_l \underline{\boldsymbol{D}_l}\left\{\boldsymbol{A}_l\widehat{\boldsymbol{\Theta}}_l + \boldsymbol{b}_l\right\} - \widehat{\boldsymbol{\Theta}}_l  \right\| \nonumber \\
 &\le \beta_l\left\|\boldsymbol{A}_l\widehat{\boldsymbol{\Theta}}_l + \boldsymbol{b}_l \right\| \nonumber \\
 &\le \beta_l (A_{\max} R_{\boldsymbol{\Theta}} + b_{\max}). \label{Theta_consec_bound}
\end{align}
The first and second inequalities follow from the non-expansiveness of the projection operators, while the third inequality follows from the bound $\|\underline{\boldsymbol{D}_l}\| \le 1.$ Plugging \eqref{Theta_consec_bound} back to \eqref{zeta_consec_bound}, we have 
\begin{equation}\label{case1_first_bound}
     \left| \mathbb{E}^\mu\zeta_i(\widehat{\boldsymbol{\Theta}}_i, \boldsymbol{X}_i) - \mathbb{E}^\mu\zeta_i(\widehat{\boldsymbol{\Theta}}_{i-2\tau_{\beta_t}}, \boldsymbol{X}_i)\right| \lesssim \sum_{l=i-2\tau_{\beta_t}}^{i-1}\beta_l 
\end{equation}
Next, by applying Lemma \ref{LEMMA:zeta_x_deviance_bound} together with Jensen’s inequality, we immediately observe the upper bound of the term in \eqref{zeta_first_case_second_term}, namely,
\begin{equation}\label{case1_second_bound}
\left| \mathbb{E}^\mu\zeta_i(\widehat{\boldsymbol{\Theta}}_{i-2\tau_{\beta_t}}, \boldsymbol{X}_i) - \mathbb{E}^\mu\zeta_i(\widehat{\boldsymbol{\Theta}}_{i-2\tau_{\beta_t}}, \boldsymbol{X}_{i-\tau_{\beta_t}:i}) \right| \lesssim \lambda^{\tau_{\beta_t}}.
\end{equation}
We now obtain an upper bound of the term in  \eqref{zeta_first_case_third_term}. To this end, let us set
$$
f_i(\widehat{\boldsymbol{\Theta}}_{i-2\tau_{\beta_t}}, \boldsymbol{Y}_{i-\tau_{\beta_t}:i}) := \zeta_i(\widehat{\boldsymbol{\Theta}}_{i-2\tau_{\beta_t}}, \boldsymbol{X}_{i-\tau_{\beta_t}:i}),
$$
where $\boldsymbol{Y}_{i-\tau_{\beta_t}:i} = \left(S^\mu_{i-\tau_{\beta_t}},S^\mu_{i-\tau_{\beta_t}+1}, \cdots, S^\mu_{i-1}, \boldsymbol{X}_i\right)$. We further define $\boldsymbol{\Theta'}_{i-2\tau_{\beta_t}}$ and $\boldsymbol{Y'}_{i-\tau_{\beta_t}:i}$ as random variables drawn independently from the marginal distributions of $\widehat{\boldsymbol{\Theta}}_{i-2\tau_{\beta_t}}$ and $\boldsymbol{Y}_{i-\tau_{\beta_t}:i}$ respectively. Since
$$
\widehat{\boldsymbol{\Theta}}_{i-2\tau_{\beta_t}} \to S^\mu_{i-2\tau_{\beta_t}} \to S^\mu_{i-\tau_{\beta_t}} \to S^\mu_i \to \boldsymbol{X}_i = (S^\mu_i, S^\mu_{i+1}, \boldsymbol{z}_i)
$$
forms a Markov chain, an application of Lemma 9 in \citep{bhandari2018finite} results in
$$
\left|\mathbb{E}^\mu f_i(\widehat{\boldsymbol{\Theta}}_{i-2\tau_{\beta_t}}, \boldsymbol{Y}_{i-\tau_{\beta_t}:i}) \right| \lesssim \left|\mathbb{E}f_i(\boldsymbol{\Theta'}_{i-2\tau_{\beta_t}}, \boldsymbol{Y'}_{i-\tau_{\beta_t}:i}) \right| + \rho^{\tau_{\beta_t}},
$$
for all $i > 2\tau_{\beta_t}$. Since $\boldsymbol{\Theta'}_{i-2\tau_{\beta_t}}$ and $\boldsymbol{Y'}_{i-\tau_{\beta_t}:i}$ are independent to each other, we get
\begin{align*}
\mathbb{E}f_i(\boldsymbol{\Theta'}_{i-2\tau_{\beta_t}}, \boldsymbol{Y'}_{i-\tau_{\beta_t}:i}) &= \boldsymbol{\Theta^*}^\top \mathbb{E}^\mu\left(\boldsymbol{A}_{i-\tau_{\beta_t}:i} - \boldsymbol{A}\right)^\top \underline{\boldsymbol{D}_i}\mathbb{E}^\mu\left(\boldsymbol{\Theta^*}- \boldsymbol{\Theta'}_{i-2\tau_{\beta_t}} \right) \\ &\quad + \mathbb{E}^\mu\left(\boldsymbol{b}_{i-\tau_{\beta_t}:i} - \boldsymbol{b} \right)^\top \underline{\boldsymbol{D}_i} \mathbb{E}^\mu\left(\boldsymbol{\Theta^*}- \boldsymbol{\Theta'}_{i-2\tau_{\beta_t}} \right).
\end{align*}
From the Cauchy-Schwarz inequality coupled with the Jensen's inequality, we get
\begin{align*}
\left|\mathbb{E}f_i(\boldsymbol{\Theta'}_{i-2\tau_{\beta_t}}, \boldsymbol{Y'}_{i-\tau_{\beta_t}:i})\right| &\le 2 \left\|\mathbb{E}^\mu\left(\boldsymbol{A}_{i-\tau_{\beta_t}:i} -\boldsymbol{A}  \right)\right\| R_{\boldsymbol{\Theta}}^2 + 2\left\| \mathbb{E}^\mu\left(\boldsymbol{b}_{i-\tau_{\beta_t}:i} -\boldsymbol{b} \right)\right\| R_{\boldsymbol{\Theta}}  
\lesssim \tau_{\beta_t} q^i + \lambda^{\tau_{\beta_t}}
\end{align*}
where the second inequality is due to Lemma \ref{LEMMA:exp_A_bound} and Lemma \ref{LEMMA:exp_b_bound}. Therefore, we get
\begin{equation}\label{case1_third_bound}
\left|\mathbb{E}^\mu f_i(\widehat{\boldsymbol{\Theta}}_{i-2\tau_{\beta_t}}, \boldsymbol{Y}_{i-\tau_{\beta_t}:i}) \right| \lesssim \tau_{\beta_t} q^i + \lambda^{\tau_{\beta_t}} + \rho^{\tau_{\beta_t}} \lesssim \tau_{\beta_t} q^i + q^{\tau_{\beta_t}},
\end{equation}
with $q = \max \{\lambda, \rho \}$. Combining \eqref{case1_first_bound}, \eqref{case1_second_bound} and \eqref{case1_third_bound}, we get
\begin{align*}
    \left| \mathbb{E}^\mu \zeta_i(\widehat{\boldsymbol{\Theta}}_i, \boldsymbol{X}_i) \right| \lesssim \sum_{l=i-2\tau_{\beta_t}}^{i-1}\beta_l + \tau_{\beta_t} q^i + q^{\tau_{\beta_t}} \lesssim \tau_{\beta_t}\beta_{i-2\tau_{\beta_t}} + \tau_{\beta_t} q^i + q^{\tau_{\beta_t}}.
\end{align*}
Next we consider the case $i \le 2\tau_{\beta_t}$. From the triangle inequality, we have
\begin{align*}
    \left| \mathbb{E}^\mu\zeta_i(\widehat{\boldsymbol{\Theta}}_i, \boldsymbol{X}_i) \right| \le \left| \mathbb{E}^\mu\zeta_i(\widehat{\boldsymbol{\Theta}}_i, \boldsymbol{X}_i) - \mathbb{E}^\mu\zeta_i(\widehat{\boldsymbol{\Theta}}_0, \boldsymbol{X}_i)\right| + \left| \mathbb{E}^\mu\zeta_i(\widehat{\boldsymbol{\Theta}}_0, \boldsymbol{X}_i)\right|.
\end{align*}
From Lemma \ref{LEMMA:zeta_lipschitz} combined with \eqref{Theta_consec_bound}, we have
\begin{equation}\label{}
\left| \mathbb{E}^\mu\zeta_i(\widehat{\boldsymbol{\Theta}}_i, \boldsymbol{X}_i) - \mathbb{E}^\mu\zeta_i(\widehat{\boldsymbol{\Theta}}_0, \boldsymbol{X}_i)\right| \lesssim \sum_{l=0}^{i-1}\beta_l \lesssim \tau_{\beta_t} \beta_0.  \label{case2_first_bound}
\end{equation}
For the second term, since $\widehat{\boldsymbol{\Theta}}_0$ is deterministic,
\begin{align*}
\mathbb{E}^\mu \zeta_i(\widehat{\boldsymbol{\Theta}}_0, \boldsymbol{X}_i) = \left\langle \mathbb{E}^\mu\left(\boldsymbol{A}_i-\boldsymbol{A}\right)\boldsymbol{\Theta^*}, \underline{\boldsymbol{D}_i}(\boldsymbol{\Theta^*}-\widehat{\boldsymbol{\Theta}}_0)\right\rangle + \left\langle \mathbb{E}^\mu\left(\boldsymbol{b}-\boldsymbol{b}_i\right), \underline{\boldsymbol{D}_i}(\boldsymbol{\Theta^*}-\widehat{\boldsymbol{\Theta}}_0)\right\rangle,
\end{align*}
and therefore
\begin{equation}\label{case2_second_bound}
\left| \mathbb{E}^\mu\zeta_i(\widehat{\boldsymbol{\Theta}}_0, \boldsymbol{X}_i)\right| 
\le 2\left\|\mathbb{E}^\mu\left(\boldsymbol{A}_i-\boldsymbol{A}\right)\right\|R_{\boldsymbol{\Theta}}^2 + 2\left\|\mathbb{E}^\mu\left(\boldsymbol{b}-\boldsymbol{b}_i\right) \right\|R_{\boldsymbol{\Theta}}
\lesssim iq^i + \lambda^i \lesssim iq^i 
\end{equation}
where the second inequality follows from Lemma \ref{LEMMA:exp_A_bound} and the last inequality is by the definition $q := \max\{\lambda, \rho\} \in (0,1)$. Combining \eqref{case2_first_bound} and \eqref{case2_second_bound}, we get
\begin{align*}
    \left| \mathbb{E}^\mu\zeta_i(\widehat{\boldsymbol{\Theta}}_i, \boldsymbol{X}_i) \right| \lesssim \tau_{\beta_t} \beta_0 + iq^i.
\end{align*}
\end{proof}

\begin{lemma}\label{LEMMA:expect_xi_bound}
Suppose $(\beta_t)_{t \in \mathbb{N}}$ is a non-increasing sequence and $q= \max\{\lambda, \rho\}$. Given $t \in \mathbb{N}$, suppose $i > 2\tau_{\beta_t}$ then,
\begin{align*}
    \left| \mathbb{E}^\mu\xi_i(\widehat{\boldsymbol{\Theta}}_i, \boldsymbol{X}_i) \right| \lesssim \tau_{\beta_t}\beta_{i-2\tau_{\beta_t}} + \tau_{\beta_t} q^i + q^{\tau_{\beta_t}}.
\end{align*}
Otherwise, 
\begin{align*}
    \left| \mathbb{E}^\mu\xi_i(\widehat{\boldsymbol{\Theta}}_i, \boldsymbol{X}_i) \right| \lesssim \tau_{\beta_t} \beta_0 + iq^i.
\end{align*}
\end{lemma}
\begin{proof}
We begin by considering the case $i > 2\tau_{\beta_t}$. Again by the triangle inequality, we have
\begin{align}
    \left| \mathbb{E}^\mu\xi_i(\widehat{\boldsymbol{\Theta}}_i, \boldsymbol{X}_i) \right| &\le \left| \mathbb{E}^\mu\xi_i(\widehat{\boldsymbol{\Theta}}_i, \boldsymbol{X}_i) - \mathbb{E}^\mu\xi_i(\widehat{\boldsymbol{\Theta}}_{i-2\tau_{\beta_t}}, \boldsymbol{X}_i)\right| \label{xi_first_case_first_term}\\ &\quad + \left| \mathbb{E}^\mu\xi_i(\widehat{\boldsymbol{\Theta}}_{i-2\tau_{\beta_t}}, \boldsymbol{X}_i) - \mathbb{E}^\mu\xi_i(\widehat{\boldsymbol{\Theta}}_{i-2\tau_{\beta_t}}, \boldsymbol{X}_{i-\tau_{\beta_t}:i}) \right| \label{xi_first_case_second_term}\\ &\quad + \left |\mathbb{E}^\mu\xi_i(\widehat{\boldsymbol{\Theta}}_{i-2\tau_{\beta_t}}, \boldsymbol{X}_{i-\tau_{\beta_t}:i}) \right| \label{xi_first_case_third_term}
\end{align}
To obtain an upper bound of \eqref{xi_first_case_first_term}, note that
\begin{align}
    \left| \mathbb{E}^\mu\xi_i(\widehat{\boldsymbol{\Theta}}_i, \boldsymbol{X}_i) - \mathbb{E}^\mu\xi_i(\widehat{\boldsymbol{\Theta}}_{i-2\tau_{\beta_t}}, \boldsymbol{X}_i)\right| &\le \left| \mathbb{E}^\mu\xi_i(\widehat{\boldsymbol{\Theta}}_i, \boldsymbol{X}_i) - \mathbb{E}^\mu\xi_i(\boldsymbol{\Theta}_{i-1}, \boldsymbol{X}_i)\right|  + \cdots \nonumber \\&\quad+  \left| \mathbb{E}^\mu\xi_i(\widehat{\boldsymbol{\Theta}}_{i-2\tau_{\beta_t}+1}, \boldsymbol{X}_i) - \mathbb{E}^\mu\xi_i(\widehat{\boldsymbol{\Theta}}_{i-2\tau_{\beta_t}}, \boldsymbol{X}_i)\right|\nonumber \\
    &\le L_\xi \sum_{l=i-2\tau_{\beta_t}}^{i-1} \left\|\widehat{\boldsymbol{\Theta}}_{l+1} - \widehat{\boldsymbol{\Theta}}_{l}\right\| \label{xi_consec_bound} 
\end{align}
where the second inequality is due to Lemma \ref{LEMMA:xi_lipschitz}. Recall from \eqref{Theta_consec_bound} that 
\begin{align*}
 \left\|\widehat{\boldsymbol{\Theta}}_{l+1} - \widehat{\boldsymbol{\Theta}}_{l}\right\| \le \beta_l (A_{\max} R_{\boldsymbol{\Theta}} + b_{\max}).
\end{align*}
Plugging \eqref{Theta_consec_bound} back to \eqref{xi_consec_bound}, we have 
\begin{equation}\label{case1_xi_first_bound}
     \left| \mathbb{E}^\mu\xi_i(\widehat{\boldsymbol{\Theta}}_i, \boldsymbol{X}_i) - \mathbb{E}^\mu\xi_i(\widehat{\boldsymbol{\Theta}}_{i-2\tau_{\beta_t}}, \boldsymbol{X}_i)\right| \lesssim \sum_{l=i-2\tau_{\beta_t}}^{i-1}\beta_l 
\end{equation}
Next, by applying Lemma \ref{LEMMA:xi_x_deviance_bound} together with Jensen’s inequality, we immediately obtain the upper bound of the term in \eqref{xi_first_case_second_term}, namely,
\begin{equation}\label{case1_xi_second_bound}
\left| \mathbb{E}^\mu\xi_i(\widehat{\boldsymbol{\Theta}}_{i-2\tau_{\beta_t}}, \boldsymbol{X}_i) - \mathbb{E}^\mu\xi_i(\widehat{\boldsymbol{\Theta}}_{i-2\tau_{\beta_t}}, \boldsymbol{X}_{i-\tau_{\beta_t}:i}) \right| \lesssim \lambda^{\tau_{\beta_t}}.
\end{equation}
We now obtain an upper bound of the term in  \eqref{xi_first_case_third_term}. To this end, let us set
$$
g_i(\widehat{\boldsymbol{\Theta}}_{i-2\tau_{\beta_t}}, \boldsymbol{Y}_{i-\tau_{\beta_t}:i}) := \xi_i(\widehat{\boldsymbol{\Theta}}_{i-2\tau_{\beta_t}}, \boldsymbol{X}_{i-\tau_{\beta_t}:i}),
$$
where $\boldsymbol{Y}_{i-\tau_{\beta_t}:i} = \left(S^\mu_{i-\tau_{\beta_t}}, S^\mu_{i-\tau_{\beta_t}+1}, \cdots, S^\mu_{i-1}, \boldsymbol{X}_i\right)$. We further define $\boldsymbol{\Theta'}_{i-2\tau_{\beta_t}}$ and $\boldsymbol{Y'}_{i-\tau_{\beta_t}:i}$ as random variables drawn independently from the marginal distributions of $\widehat{\boldsymbol{\Theta}}_{i-2\tau_{\beta_t}}$ and $\boldsymbol{Y}_{i-\tau_{\beta_t}:i}$ respectively. Since
$$
\widehat{\boldsymbol{\Theta}}_{i-2\tau_{\beta_t}} \to S^\mu_{i-2\tau_{\beta_t}} \to S^\mu_{i-\tau_{\beta_t}} \to S^\mu_i \to \boldsymbol{X}_i = (S^\mu_i, S^\mu_{i+1}, \boldsymbol{z}_i)
$$
forms a Markov chain, an application of Lemma 9 in \citep{bhandari2018finite} results in
$$
\left|\mathbb{E}^\mu g_i(\widehat{\boldsymbol{\Theta}}_{i-2\tau_{\beta_t}}, \boldsymbol{Y}_{i-\tau_{\beta_t}:i}) \right| \lesssim \left|\mathbb{E}g_i(\boldsymbol{\Theta'}_{i-2\tau_{\beta_t}}, \boldsymbol{Y'}_{i-\tau_{\beta_t}:i}) \right| + \rho^{\tau_{\beta_t}},
$$
for all $i > 2\tau_{\beta_t}$. Since $\boldsymbol{\Theta'}_{i-2\tau_{\beta_t}}$ and $\boldsymbol{Y'}_{i-\tau_{\beta_t}:i}$ are independent to each other, we get
\begin{align*}
\mathbb{E}g_i(\boldsymbol{\Theta'}_{i-2\tau_{\beta_t}}, \boldsymbol{Y'}_{i-\tau_{\beta_t}:i}) &= \mathbb{E} \left[\left(\boldsymbol{\Theta^*}-\boldsymbol{\Theta'}_{i-2\tau_{\beta_t}}\right)^\top \left(\boldsymbol{A'}_{i-\tau_{\beta_t}:i}-\boldsymbol{A} \right)^\top \underline{\boldsymbol{D}_i} \left(\boldsymbol{\Theta^*}-\boldsymbol{\Theta'}_{i-2\tau_{\beta_t}}\right)\right] \\
&= \mathbb{E} \left[
\text{Trace}\left\{\left(\boldsymbol{A'}_{i-\tau_{\beta_t}:i}-\boldsymbol{A} \right)^\top \underline{\boldsymbol{D}_i}  \left(\boldsymbol{\Theta^*}-\boldsymbol{\Theta'}_{i-2\tau_{\beta_t}}\right)\left(\boldsymbol{\Theta^*}-\boldsymbol{\Theta'}_{i-2\tau_{\beta_t}}\right)^\top 
\right\}\right] \\
&= \text{Trace}\left[\mathbb{E} \left\{\left(\boldsymbol{A'}_{i-\tau_{\beta_t}:i}-\boldsymbol{A} \right)^\top \underline{\boldsymbol{D}_i}  \left(\boldsymbol{\Theta^*}-\boldsymbol{\Theta'}_{i-2\tau_{\beta_t}}\right)\left(\boldsymbol{\Theta^*}-\boldsymbol{\Theta'}_{i-2\tau_{\beta_t}}\right)^\top\right\} 
\right] \\
&= \text{Trace}\left[\mathbb{E}^\mu \left\{\left(\boldsymbol{A'}_{i-\tau_{\beta_t}:i}-\boldsymbol{A} \right)^\top \right\} \underline{\boldsymbol{D}_i} \mathbb{E}^\mu \left\{\left(\boldsymbol{\Theta^*}-\boldsymbol{\Theta'}_{i-2\tau_{\beta_t}}\right)\left(\boldsymbol{\Theta^*}-\boldsymbol{\Theta'}_{i-2\tau_{\beta_t}}\right)^\top\right\} 
\right].
\end{align*}
where the last equality comes from the independence between $\boldsymbol{\Theta'}_{i-2\tau_{\beta_t}}$ and $\boldsymbol{Y'}_{i-\tau_{\beta_t}:i}$. By the Von-Neumann’s trace inequality (see Theorem 7.4.1.1 of \citep{horn2012matrix}) with a nuclear norm notation $\|\cdot\|_*$, we have
\begin{align*}
\left|\mathbb{E}g_i(\boldsymbol{\Theta'}_{i-2\tau_{\beta_t}}, \boldsymbol{Y'}_{i-\tau_{\beta_t}:i})\right| &\le  \left\|\mathbb{E}^\mu \left(\boldsymbol{A'}_{i-\tau_{\beta_t}:i}\right)-\boldsymbol{A} \right\| \left\|\underline{\boldsymbol{D}_i}\mathbb{E}^\mu\left\{\left(\boldsymbol{\Theta^*}-\boldsymbol{\Theta'}_{i-2\tau_{\beta_t}}\right)\left(\boldsymbol{\Theta^*}-\boldsymbol{\Theta'}_{i-2\tau_{\beta_t}}\right)^\top 
\right\} \right\|_*\\
&\le  \left\|\mathbb{E}^\mu \left(\boldsymbol{A'}_{i-\tau_{\beta_t}:i}\right)-\boldsymbol{A} \right\| \left\|\underline{\boldsymbol{D}_i}\right\| \left\|\mathbb{E}^\mu \left\{\left(\boldsymbol{\Theta^*}-\boldsymbol{\Theta'}_{i-2\tau_{\beta_t}}\right)\left(\boldsymbol{\Theta^*}-\boldsymbol{\Theta'}_{i-2\tau_{\beta_t}}\right)^\top 
\right\} \right\|_*
\end{align*}
where the second inequality is due to an identity $\|\boldsymbol{A}\boldsymbol{B}\|_* \le \|\boldsymbol{A}\|\|\boldsymbol{B}\|_*$. Furthermore, notice that 
\begin{align*}
 \left\|\mathbb{E}^\mu \left\{\left(\boldsymbol{\Theta^*}-\boldsymbol{\Theta'}_{i-2\tau_{\beta_t}}\right)\left(\boldsymbol{\Theta^*}-\boldsymbol{\Theta'}_{i-2\tau_{\beta_t}}\right)^\top 
\right\} \right\|_* &\le \mathbb{E}^\mu\left\{\left\|\left(\boldsymbol{\Theta^*}-\boldsymbol{\Theta'}_{i-2\tau_{\beta_t}}\right)\left(\boldsymbol{\Theta^*}-\boldsymbol{\Theta'}_{i-2\tau_{\beta_t}}
\right)^\top \right\|_* \right\} \\&= \mathbb{E}^\mu\left\{\left\|\boldsymbol{\Theta^*}-\boldsymbol{\Theta'}_{i-2\tau_{\beta_t}} \right\|^2 \right\} \\  
&\le 4R_{\boldsymbol{\Theta}}^2
\end{align*}
where the first inequality is due to Jensen's inequality. Therefore, we arrive at 
\begin{align*}
\left|\mathbb{E}^\mu g_i(\boldsymbol{\Theta'}_{i-2\tau_{\beta_t}}, \boldsymbol{Y'}_{i-\tau_{\beta_t}:i})\right| \lesssim  \left\|\mathbb{E}^\mu \left( \boldsymbol{A'}_{i-\tau_{\beta_t}:i}\right)- \boldsymbol{A} \right\| \lesssim \tau_{\beta_t} q^i + \lambda^{\tau_{\beta_t}}
\end{align*}
where the last inequality follows from Lemma \ref{LEMMA:exp_A_bound}. This then gives us 
\begin{equation}\label{case1_xi_third_bound}
\left|\mathbb{E}^\mu g_i(\widehat{\boldsymbol{\Theta}}_{i-2\tau_{\beta_t}}, \boldsymbol{Y}_{i-\tau_{\beta_t}:i}) \right| \lesssim \tau_{\beta_t} q^i + \lambda^{\tau_{\beta_t}} + \rho^{\tau_{\beta_t}} \lesssim \tau_{\beta_t} q^i + q^{\tau_{\beta_t}},
\end{equation}
for $q = \max \{\lambda, \rho \}$. Combining \eqref{case1_xi_first_bound}, \eqref{case1_xi_second_bound} and \eqref{case1_xi_third_bound}, we get
\begin{align*}
    \left| \mathbb{E}^\mu \xi_i(\widehat{\boldsymbol{\Theta}}_i, \boldsymbol{X}_i) \right| \lesssim \sum_{l=i-2\tau_{\beta_t}}^{i-1}\beta_l + \tau_{\beta_t} q^i + q^{\tau_{\beta_t}} \lesssim \tau_{\beta_t}\beta_{i-2\tau_{\beta_t}} + \tau_{\beta_t} q^i + q^{\tau_{\beta_t}}.
\end{align*}
Next we consider the case $i \le 2\tau_{\beta_t}$. From the triangle inequality, we have
\begin{align*}
    \left| \mathbb{E}^\mu \xi_i(\widehat{\boldsymbol{\Theta}}_i, \boldsymbol{X}_i) \right| \le \left| \mathbb{E}^\mu \xi_i(\widehat{\boldsymbol{\Theta}}_i, \boldsymbol{X}_i) - \mathbb{E}^\mu \xi_i(\widehat{\boldsymbol{\Theta}}_0, \boldsymbol{X}_i)\right| + \left| \mathbb{E}^\mu \xi_i(\widehat{\boldsymbol{\Theta}}_0, \boldsymbol{X}_i)\right|.
\end{align*}
From Lemma \ref{LEMMA:xi_lipschitz} combined with \eqref{Theta_consec_bound}, we have
\begin{equation}
\left| \mathbb{E}^\mu\xi_i(\widehat{\boldsymbol{\Theta}}_i, \boldsymbol{X}_i) - \mathbb{E}^\mu\xi_i(\widehat{\boldsymbol{\Theta}}_0, \boldsymbol{X}_i)\right| \lesssim \sum_{l=0}^{i-1}\beta_l \lesssim \tau_{\beta_t} \beta_0.  \label{case2_xi_first_bound}
\end{equation}
For the second term, since $\widehat{\boldsymbol{\Theta}}_0$ is deterministic,
\begin{align*}
\mathbb{E}^\mu\xi_i(\widehat{\boldsymbol{\Theta}}_0, \boldsymbol{X}_i) = (\boldsymbol{\Theta^*}-\widehat{\boldsymbol{\Theta}}_0)^\top \mathbb{E}^\mu \left(\boldsymbol{A_i} -\boldsymbol{A} \right)^\top \underline{\boldsymbol{D}_i} (\boldsymbol{\Theta^*}-\widehat{\boldsymbol{\Theta}}_0)
\end{align*}
and therefore
\begin{equation}\label{case2_xi_second_bound}
\left| \mathbb{E}^\mu\xi_i(\widehat{\boldsymbol{\Theta}}_0, \boldsymbol{X}_i)\right| \le 4R_{\boldsymbol{\Theta}}^2 \left \| \mathbb{E}\boldsymbol{A}_i -\boldsymbol{A}\right\| 
\lesssim iq^i + \lambda^i 
\lesssim iq^i 
\end{equation}
where the second inequality follows from Lemma \ref{LEMMA:exp_A_bound} and in the last inequality we used the definition $q := \max\{\lambda, \rho\} \in (0,1)$. Combining \eqref{case2_xi_first_bound} and \eqref{case2_xi_second_bound}, we get
\begin{align*}
    \left| \mathbb{E}^\mu\zeta_i(\widehat{\boldsymbol{\Theta}}_i, \boldsymbol{X}_i) \right| \lesssim \tau_{\beta_t} \beta_0 + iq^i.
\end{align*}
\end{proof}

\begin{lemma}\label{LEMMA:prelim_all_time_exp_bound} For $t \in \mathbb{N}$, let $\beta_t = \frac{\beta_0}{(t+1)^s}$ and $s \in (0,1)$. With $\gamma > 0$,
\begin{align*}
    \sum_{i=0}^t \left(e^{-\gamma\sum_{k=i+1}^t \beta_k} \right)\beta^2_i \le 2\left(K_b e^{-\frac{\gamma}{2}\sum_{k=0}^t \beta_k} + \beta_t\right)\frac{e^{\frac{\gamma \beta_0}{2}} }{\gamma}, 
\end{align*}
where $K_b = \beta_0 e^{\frac{\gamma}{2}\sum_{k=0}^{i_0} \beta_k}$ for some $i_0 \in \mathbb{N}$. 
\end{lemma}
\begin{proof}
Let $T_t = \sum_{i=0}^{t-1} \beta_i$ and use the convention $\sum_{k=t+1}^t \beta_k = 0$ and $\sum_{k=t+1}^t \beta^2_k = 0$. Notice that
\begin{align}
\sum_{i=0}^t \left(e^{-\frac{\gamma}{2}\sum_{k=i+1}^t \beta_k} \right)\beta_i &\le \left(\sup_{i \ge 0} e^{\frac{\gamma}{2}\beta_i}\right) \left\{\sum_{i=0}^t \left(e^{-\frac{\gamma}{2}\sum_{k=i}^t \beta_k} \right)\beta_i \right\} 
\nonumber \\ &= \left(\sup_{i \ge 0} e^{\frac{\gamma}{2}\beta_i}\right) \left\{\sum_{i=0}^t \left(e^{-\frac{\gamma}{2}(T_{t+1}-T_i)} \right)\beta_i \right\} \nonumber \\
&\le \left(\sup_{i \ge 0} e^{\frac{\gamma}{2}\beta_i}\right)  \int_{0}^{T_{t+1}} e^{-\frac{\gamma}{2}(T_{t+1}-s)} ds \nonumber \\
&\le \left(\sup_{i \ge 0} e^{\frac{\gamma}{2}\beta_i}\right) \frac{2}{\gamma} \le \frac{2e^{\frac{\gamma \beta_0}{2}}}{\gamma}, \label{exp_order1_bound}
\end{align}
where we have used the definition of the left-Riemann sum in the first inequality. The last inequality is due to the fact that $\{\beta_t\}$ is a non-increasing sequence. Now consider 
\begin{align}
\sum_{i=0}^t \left(e^{-\gamma\sum_{k=i+1}^t \beta_k} \right)\beta^2_i &\le \sup_{0 \le i \le t} \left(\beta_i e^{-\frac{\gamma}{2}\sum_{k=i+1}^t \beta_k}\right) \left\{\sum_{i=0}^t \left(e^{-\frac{\gamma}{2}\sum_{k=i+1}^t \beta_k} \right)\beta_i \right\} \nonumber \\
&\le \sup_{0 \le i \le t} \left(\beta_i e^{-\frac{\gamma}{2}\sum_{k=i+1}^t \beta_k}\right)\frac{2e^{\frac{\gamma \beta_0}{2}}}{\gamma}\label{intermediate_exp_square_bdd}
\end{align}
where the last inequality follows from \eqref{exp_order1_bound}. Note that $\beta_i e^{-\frac{\gamma}{2}\sum_{k=i+1}^t \beta_k}$ is eventually increasing, i.e., after some time $i_0 \in \mathbb{N}$,  for all $t \ge i_0$, we have
$$
\sup_{i_0 \le i \le t} \left\{\beta_i \exp\left(-\frac{\gamma}{2}\sum_{k=i+1}^t \beta_k \right)\right\} \le \beta_t,
$$
where we used the convention $\sum_{k=t+1}^t \beta_k = 0$. Therefore, we have
\begin{align*}
    \eqref{intermediate_exp_square_bdd} &\le \left\{\sup_{0 \le i \le i_0} \left(\beta_i e^{-\frac{\gamma}{2}\sum_{k=i+1}^t \beta_k}\right) + \beta_t\right\}\frac{2e^{\frac{\gamma \beta_0}{2}} }{\gamma} \\
    &\le \left\{e^{-\frac{\gamma}{2}\sum_{k=0}^t \beta_k}\sup_{0 \le i \le i_0} \left(\beta_i e^{\frac{\gamma}{2}\sum_{k=0}^i \beta_k}\right) + \beta_t\right\}\frac{2e^{\frac{\gamma \beta_0}{2}} }{\gamma} \\
    &\le \left(K_b e^{-\frac{\gamma}{2}\sum_{k=0}^t \beta_k} + \beta_t\right)\frac{2e^{\frac{\gamma \beta_0}{2}} }{\gamma},
\end{align*}
where $K_b = \beta_0 e^{\frac{\gamma}{2}\sum_{k=0}^{i_0} \beta_k}$.
\end{proof}

\begin{lemma}\label{LEMMA:prelim_exp_bound} For $t \in \mathbb{N}$, let $\beta_t = \frac{\beta_0}{(t+1)^s}$ and $s\in (0,1)$. With $\gamma > 0$ and $\tau \in \{0, \cdots, t\}$,
\begin{enumerate}
    \item $\sum_{i=0}^{\tau} e^{-\gamma \sum_{k=i+1}^{t} \beta_{k}} \beta_{i}  \le \frac{e^{\gamma \beta_0}}{\gamma} e^{-\frac{\gamma \beta_0}{1-s}\left\lbrace(1+t)^{1-s}-\left(1+\tau\right)^{1-s}\right\rbrace}$. 
    \item $\sum_{i=2\tau_{\beta_t}+1}^{t} \left(e^{-\gamma \sum_{k=i+1}^{t} \beta_{k}} \beta_{i-2\tau_{\beta_t}} \beta_{i} \right)  \le \left[e^{\frac{-\gamma \beta_0}{2(1-s)}\left\lbrace(t+1)^{1-s}-1\right\rbrace} D_{\beta} \mathbb{I}_{\left(2\tau_{\beta_t}+1 < i_{\beta} \right)}+\beta_{t-2\tau_{\beta_t}} \right] \frac{2 e^{\gamma \beta_0 / 2}}{\gamma}$
\end{enumerate}
where 
$D_\beta = \exp\lbrace\left(\gamma / 2\right) \sum_{k=0}^{i_{\beta}} \beta_{k}\rbrace \beta_0$ for some $i_{f_\beta} \in \mathbb{N}$.
\end{lemma}
\begin{proof}
Let $T_t = \sum_{i=0}^{t-1} \beta_i$ and use the convention $\sum_{k=t+1}^t \beta_k = 0$. For the first statement, 
\begin{align*}
\sum_{i=0}^{\tau} e^{-\gamma \sum_{k=i+1}^{t} \beta_{k}} \beta_{i} & \leq \max _{i \geq 0}\left(e^{\gamma \beta_{i}}\right) \sum_{i=0}^{\tau} e^{-\gamma \sum_{k=i}^{t} \beta_{k}} \beta_{i} = e^{\gamma \beta_0} \sum_{i=0}^{\tau} e^{-\gamma\left(T_{t+1}-T_{i}\right)} \beta_{i} \\
& \leq e^{\gamma \beta_0} \int_{0}^{T_{\tau+1}} e^{-\gamma\left(T_{t+1}-s\right)} d s  \leq \frac{e^{\gamma \beta_0}}{\gamma} e^{-\gamma\left(T_{t+1}-T_{\tau+1}\right)} \\&= \frac{e^{\gamma \beta_0}}{\gamma} e^{-\gamma \beta_0 \sum_{k=\tau+1}^{t} 1 /(1+k)^{s}} \le \frac{e^{\gamma \beta_0}}{\gamma} e^{-\frac{\gamma \beta_0}{1-s}\left\lbrace(1+t)^{1-s}-\left(1+\tau\right)^{1-s}\right\rbrace}.
\end{align*}
For the second statement, first notice that 
\begin{align}
\sum_{i=2\tau_{\beta_t}+1}^{t} e^{-\gamma \sum_{k=i+1}^{t} \beta_{k}} \beta_{i} & \leq \max _{i \geq 0}\left(e^{\gamma \beta_{i}}\right)  \sum_{i=2\tau_{\beta_t}+1}^{t} e^{-\gamma \sum_{k=i}^{t} \beta_{k}} \beta_{i} = e^{\gamma \beta_0}  \sum_{i=2\tau_{\beta_t}+1}^{t} e^{-\gamma\left(T_{t+1}-T_{i}\right)} \beta_{i} \nonumber \\
& \leq e^{\gamma \beta_0}\int_{T_{2\tau_{\beta_t}+1}}^{T_{t+1}} e^{-\gamma\left(T_{t+1}-s\right)} ds =\frac{e^{\gamma \beta_0}}{\gamma} \left\lbrace1-e^{-\gamma \left(T_{t+1}-T_{2\tau_{\beta_t}+1}\right)}\right\rbrace \leq \frac{e^{\gamma \beta_0}}{\gamma}. \label{prelim_second_pre}
\end{align}
Then, we have
\begin{align}
\sum_{i=2\tau_{\beta_t}+1}^{t} e^{-\gamma \sum_{k=i+1}^{t} \beta_{k}} \beta_{i-2\tau_{\beta_t}} \beta_{i} & \leq \max _{i \in\left[2\tau_{\beta_t}+1, t\right]}\left\{e^{\left(-\gamma / 2\right) \sum_{k=i+1}^{t} \beta_{k}} \beta_{i-2\tau_{\beta_t}}\right\} \sum_{i=2\tau_{\beta_t}+1}^{t} e^{\left(-\gamma / 2\right) \sum_{k=i+1}^{t} \beta_{k}} \beta_{i} \nonumber \\
& \leq \max _{i \in\left[2\tau_{\beta_t}+1, t\right]}\left\{e^{\left(-\gamma / 2\right) \sum_{k=i+1}^{t} \beta_{k}} \beta_{i-2\tau_{\beta_t}}\right\} \frac{2 e^{\gamma \beta_0 / 2}}{\gamma} \label{part_c_exp_bound}
\end{align}
where the second inequality follows from \eqref{prelim_second_pre}. To bound the first term in \eqref{part_c_exp_bound}, note that the sequence $\left\{e^{\left(-\gamma / 2\right) \sum_{k=i+1}^{t} \beta_{k}} \beta_{i-2\tau_{\beta_t}}\right\}_{i\in \mathbb{N}}$ is eventually increasing. In other words, there exists $i_{\beta} \in \mathbb{N}$ such that, 
\begin{align*}
\max _{i \in\left[2\tau_{\beta_t}+1, t\right]}\left\{e^{\left(-\gamma / 2\right) \sum_{k=i+1}^{t} \beta_{k}} \beta_{i-2\tau_{\beta_t}}\right\}&=\beta_{t-2\tau_{\beta_t}} \quad \text{if} \quad 2\tau_{\beta_t}+1\ge i_{\beta}.
\end{align*}
If $2\tau_{\beta_t}+1<i_{\beta}$, then
\begin{align*}
&\max _{i \in\left[2\tau_{\beta_t}+1, t\right]}\left\{e^{\left(-\gamma / 2\right) \sum_{k=i+1}^{t} \beta_{k}} \beta_{i-2\tau_{\beta_t}}\right\}\\ &\leq \max _{i \in\left[2\tau_{\beta}+1, i_{\beta}\right]}\left\{e^{\left(-\gamma / 2\right) \sum_{k=i+1}^{t} \beta_{k}} \beta_{i-2\tau_{\beta_t}}\right\}+\max _{i \in\left[i_{\beta}+1, t\right]}\left\{e^{\left(-\gamma / 2\right) \sum_{k=i+1}^{t} \beta_{k}} \beta_{i-2\tau_{\beta_t}}\right\} \\
& \leq e^{\left(-\gamma / 2\right) \sum_{k=0}^{t} \beta_{k}} \max _{i \in\left[2\tau_{\beta_t}+1, i_{\beta}\right]}\left\{e^{\left(\gamma / 2\right) \sum_{k=0}^{i} \beta_{k}} \beta_{i-2\tau_{\beta_t}}\right\}+\beta_{t-2\tau_{\beta_t}} \\
& \leq e^{-\left(\gamma / 2\right) \sum_{k=0}^{t} \beta_{k}} e^{\left(\gamma / 2\right) \sum_{k=0}^{i_{\beta}} \beta_{k}} \beta_{0}+\beta_{t-2\tau_{\beta_t}} \\
& \leq e^{\frac{-\gamma \beta_0}{2(1-s)}\left\lbrace(t+1)^{1-s}-1\right\rbrace} D_{\beta}+\beta_{t-2\tau_{\beta_t}}
\end{align*}
where $D_\beta = e^{\left(\gamma / 2\right) \sum_{k=0}^{i_{\beta}} \beta_{k}} \beta_0$. Combining everything, we get
\begin{align*}
\sum_{i=2\tau_{\beta_t}+1}^{t} e^{-\gamma \sum_{k=i+1}^{t} \beta_{k}} \beta_{i-2\tau_{\beta_t}} \beta_{i} \leq \left[e^{\frac{-\gamma \beta_0}{2(1-s)}\left\lbrace(t+1)^{1-s}-1\right\rbrace} D_{\beta} \mathbb{I}_{\left(2\tau_{\beta_t}+1 < i_{\beta} \right)}+\beta_{t-2\tau_{\beta_t}} \right] \frac{2 e^{\gamma \beta_0 / 2}}{\gamma}. 
\end{align*}
\end{proof}

\section{Additional Experimental Details}
We provide detailed explanations of the experimental setups for both evaluation and control experiments in this supplementary results section. For each run, we set the exponential weight parameter $\lambda = 0.25$ and the step-size ratio parameter $c_\alpha = 1.0$. Under this hyperparameter configuration, we evaluate four methods: (i) average-reward TD($\lambda$), (ii) average-reward implicit TD($\lambda$) without projection, and (iii–iv) average-reward implicit TD($\lambda$) with projection, using parameter radius $R_{\boldsymbol{\Theta}} \in \{1000,\ 5000\}$. For the projection of the average-reward estimate, we fix the radius $R_\omega = 1$, which safely bounds the true average-reward since $\omega^\mu \in [-1,\ 1]$ by construction in all our settings.

\subsection{Evaluation Experiments}\label{subsec:APPEND_EVAL_LOSS}
\paragraph{Computing true parameters.}
We compute the oracle quantities used in the loss $(\widehat{\omega}_t - \omega^{\mu})^2 + \|\Pi_{\mathbb{O}}(\widehat{\pmb{\theta}}_t - \boldsymbol{\theta}^*)\|^2$ as follows. Given a transition matrix $\boldsymbol P^\mu$, reward vector $\boldsymbol r^\mu$, and feature matrix $\boldsymbol{\Phi}$, we first obtain the stationary distribution $\boldsymbol{\pi}^\mu$ satisfying ${\boldsymbol{\pi}^\mu}^\top \boldsymbol P^\mu = {\boldsymbol{\pi}^{\mu}}^\top$ and ${\boldsymbol{\pi}^\mu}^\top \boldsymbol{e} = 1$. The average reward is then $\omega^{\mu} = {\boldsymbol{\pi}^\mu}^\top {\boldsymbol r}^\mu$. In addition, to compute the optimal weight vector $\boldsymbol{\theta}^*$, we first solve for the basic differential value $\boldsymbol v^\mu$, which is the unique solution to
$$
(\boldsymbol I - \boldsymbol P^\mu)\boldsymbol v^\mu = \boldsymbol r^\mu - \omega^{\mu} \boldsymbol{e},\qquad
{\boldsymbol{\pi}^\mu}^\top \boldsymbol v^\mu = 0.
$$
Because $\boldsymbol{e}$ and $\boldsymbol v^\mu$ are included as columns of $\boldsymbol{\Phi}$, which is of full rank, the optimal weight vector $\boldsymbol{\theta}^*$ can be obtained from solving the linear system
$
\boldsymbol{\Phi}\boldsymbol{\theta}^* = \boldsymbol v^\mu.
$
In an analogous fashion, to compute the projection to the space $\mathbb{O}$, we solve for $\boldsymbol{\theta}_{\boldsymbol{e}}$ satisfying $\boldsymbol{\Phi}\boldsymbol{\theta}_{\boldsymbol{e}} = \boldsymbol{e}.$
Recall that $\boldsymbol{\theta}_{\boldsymbol{e}}$ defines the projection direction that removes the constant component of the error. The projection operator to the space $\mathbb{O}$, can be expressed as $\boldsymbol{I}-\frac{\boldsymbol{\theta}_e{\boldsymbol{\theta}_e}^\top}{\|\boldsymbol{\theta}_e\|^2}$, i.e., $\Pi_{\mathbb{O}}(\boldsymbol{\theta}) = \left(\boldsymbol{I}-\frac{\boldsymbol{\theta}_e{\boldsymbol{\theta}_e}^\top}{\|\boldsymbol{\theta}_e\|^2}\right)\boldsymbol{\theta}.$

\subsubsection{MRP}\label{sec:MRP}
\paragraph{Experiment setup.}
We describe the construction of transition probabilities, rewards, and the feature matrix following \citep{zhang2021finite}; details are reproduced here for completeness.

\begin{itemize}
\item \textbf{Transition probabilities:} For each state $s$, we generated a probability distribution over the $|\mathcal S|=100$ states by drawing ($|\mathcal S|-1$) i.i.d samples from $\mathrm{Unif}[0,1]$. We then sorted them, and took successive differences. The final entry was set to ensure the components sum to one.
\item \textbf{Rewards:} Each state $s$ received a reward sampled independently from $\mathrm{Unif}[0,1]$.
\item \textbf{Feature matrix:} Let $d$ denote the feature dimension. We draw $\widetilde{\boldsymbol{\Phi}}\in\mathbb{R}^{|\mathcal S|\times(d-2)}$ with i.i.d. $\mathrm{Bernoulli}(0.5)$ entries. We then appended the all-ones vector $\boldsymbol{e}$ and the basic differential value $\boldsymbol v^\mu$ as columns to form
$$
\boldsymbol{\Phi} = \left[\widetilde{\boldsymbol{\Phi}}\;\; \boldsymbol{e} \;\; \boldsymbol v^\mu\right].
$$
If needed, we repeated the sampling until $\boldsymbol{\Phi}$ had full column rank, then was row-normalized so that $\|\boldsymbol{\phi}(s)\| \le 1$ for all $s\in\mathcal S$.
\end{itemize}

\paragraph{Additional results.} 
We report additional MRP results under a decaying step-size schedule. For each schedule, we run 50 independent trials, initialize $\boldsymbol{\theta}_0$ by sampling each coordinate from $\mathrm{Unif}[-1,1]$, and set the initial average-reward estimate $\widehat{\omega}_{0}=0$. Complementing the constant step-size results in the main text, Figure \ref{fig:MRP_selected_stepsizes} shows performance under $\beta_t=\beta_0/(t+1)^{0.99}$ for initial step-sizes $\beta_0\in\{0.1,0.2,\ldots,3.0\}$ using the same hyperparameters $(\lambda=0.25,\,c_\alpha=1.0)$. Solid lines denote the mean loss across runs and shaded regions indicate 95$\%$ confidence intervals. As $\beta_0$ increases, average-reward implicit TD($\lambda$) methods keep the loss controlled, whereas average-reward TD($\lambda$) diverges for step-sizes larger than 2.0. For $\beta_t=1.8/(t+1)^{0.99}$, the full loss trajectory (right panel) further underscores the gap: average-reward TD($\lambda$) remains markedly worse than its implicit counterparts.

\begin{figure}[t]
\centering
\caption{\small MRP experiment results under decaying step-size schedule \( \beta_t = \beta_0 / (t+1)^{0.99} \), with exponential weighting parameter and step-size ratio set to \( (\lambda, c_{\alpha}) = (0.25, 1.0) \). Solid lines denote the mean, and shaded regions represent 95\% confidence intervals. (Left) Loss value for initial step-sizes from 0.1 to 3.0.
(Right) Full trajectory of the loss value with initial step-size $\beta_0 = 1.8$.}
\label{fig:MRP_selected_stepsizes}
    \centering
    \begin{subfigure}[t]{0.4\linewidth}
        \includegraphics[width=\linewidth]{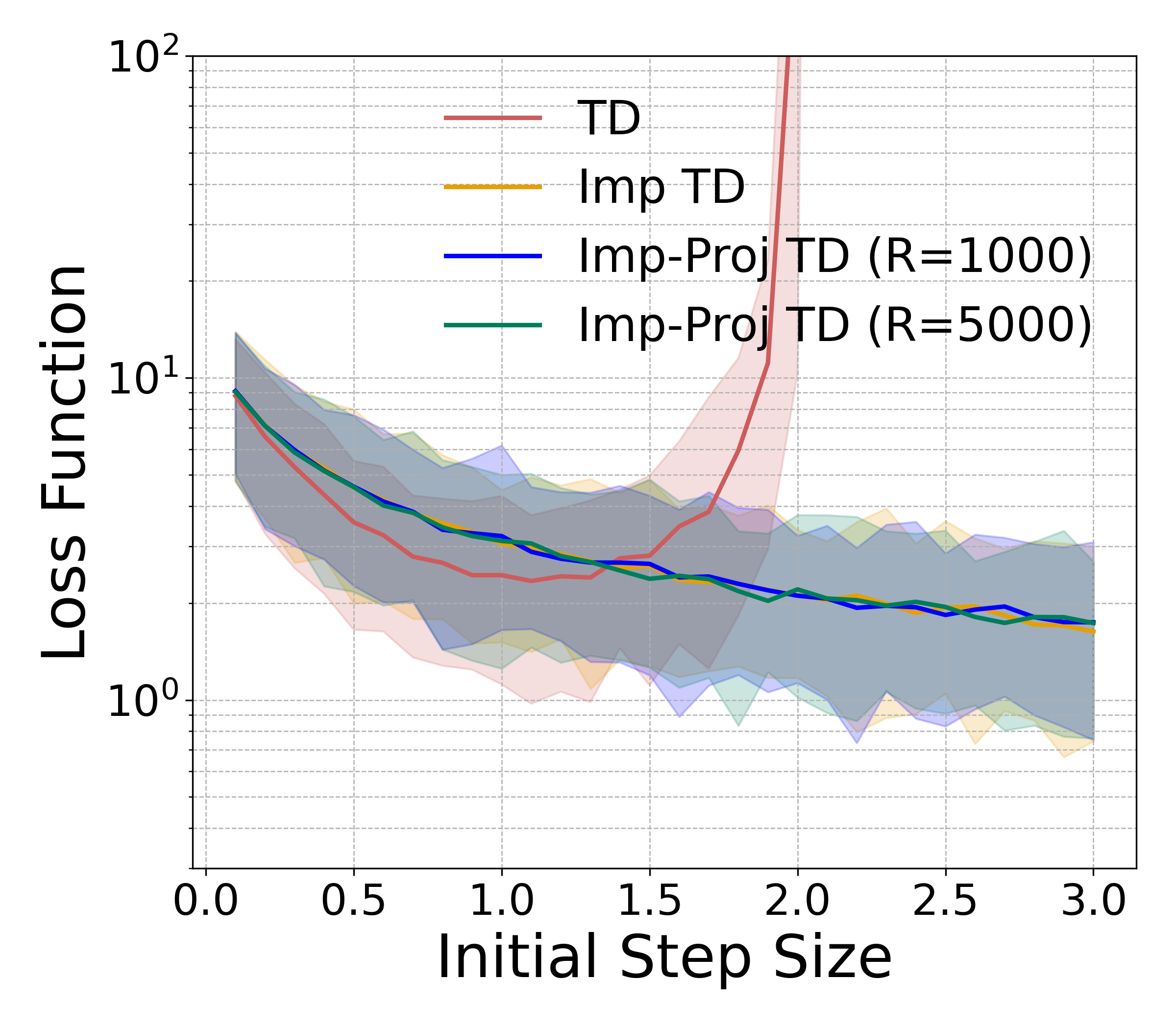}
    \end{subfigure}
    \begin{subfigure}[t]{0.4\linewidth}
        \includegraphics[width=\linewidth]{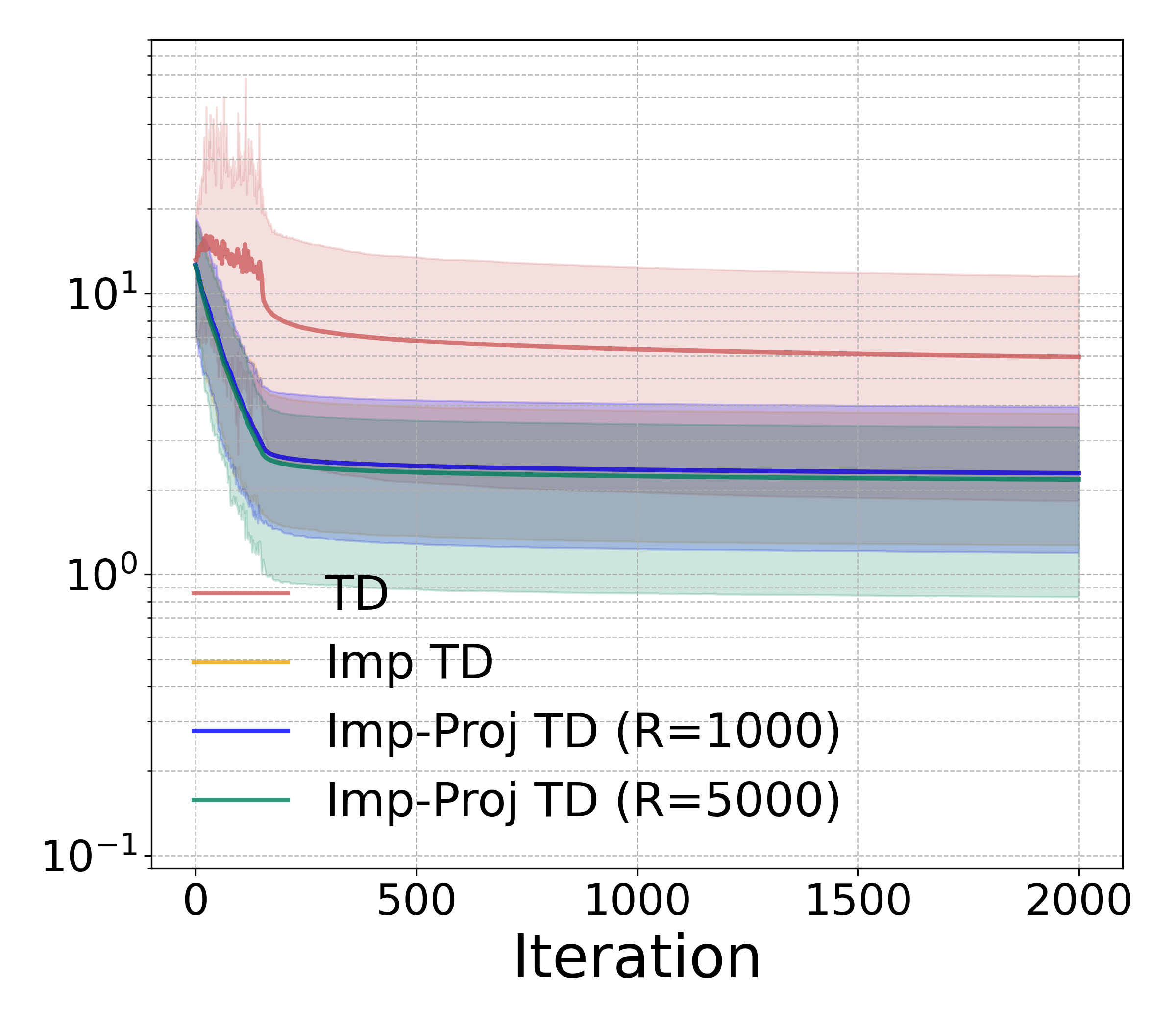}
    \end{subfigure}
\end{figure}

\subsubsection{Boyan Chain}
\label{sec:APPEND_Boyan}
\paragraph{Experiment setup.} 
We describe the construction of the transition probabilities, reward function, and feature matrix for the average-reward Boyan chain.  
The original Boyan chain was introduced by \citep{boyan2002technical} and later adapted to the average-reward setting by \citep{zhang2021average}. The chain consists of 13 states and two actions, denoted by $\{s_0, s_1, \ldots, s_{12}\}$ and $\{a_0, a_1\}$, respectively.

\begin{itemize}
    \item \textbf{Transition probabilities:} The transition probabilities of the Boyan chain are defined as 
\begin{align*}
   & p(s_{i-2} \mid s_{i}, a_0) = 1, \quad p(s_{i-1} \mid s_{i}, a_1) = 1, \quad \forall i\in\{2,3,\ldots, 12\}, \\
   & p(s_0 \mid s_1, a_0) = p(s_0 \mid s_1, a_1) = 1, \\
   & p(s_j \mid s_0, a_0) = p(s_j \mid s_0, a_1) = \tfrac{1}{13}, \quad \forall j \in \{0,1,\ldots, 12\}.
\end{align*}

    \item \textbf{Reward:} The reward function is defined as $r(i,a_0) = 0.5$ and $r(i,a_1) = 1$ for all $i \in \{0,1,\ldots, 12\}$.

    \item \textbf{Feature matrix:} 
    Consistent with the MRP experiment, we first construct the matrix $\widetilde{\boldsymbol{\Phi}}$, and append both the all-ones vector $\boldsymbol{e}$ and the differential value function $\boldsymbol v^\mu$ as columns
$$
\boldsymbol{\Phi} = \left[\widetilde{\boldsymbol{\Phi}}\;\; \boldsymbol{e} \;\; \boldsymbol v^\mu\right].
$$
    Following \cite{boyan2002technical}, the matrix $\widetilde{\boldsymbol{\Phi}}$ is constructed by linearly interpolating between the one-hot vectors $(1,0,0,0)$ and $(0,0,0,1)$ in increments of $1/4$. Specifically, the representation starts with state $0$ as $(1,0,0,0)$, then passes through intermediate states such as $(\tfrac{3}{4},\tfrac{1}{4},0,0)$ and $(\tfrac{1}{2},\tfrac{1}{2},0,0)$, eventually reaching $(0,0,0,1)$. Finally, we normalize each row to ensure $\|\boldsymbol{\phi}(s)\| \le 1$ for all $s\in\mathcal S$.
\end{itemize}

\paragraph{Additional results.} 
We also report extended results for the Boyan experiments. For each step-size schedule, we conduct $50$ independent runs with $T=2000$. 
Each component of the parameter vector $\widehat{\theta}_0$ is initialized as $\mathcal{U}[-1,1]$, and the initial average-reward estimate is set to $\widehat{\omega}_0 = 0$. 
In each experiment, we first sample actions in each state from a $\text{Binomial}(n=13, p=0.5)$ distribution. The sampled action will determine the deterministic policy to be evaluated. For each such sampled deterministic policy, an associated transition probability matrix $\boldsymbol{P}^\mu$ is induced. We provide results under the constant step-size schedule. The results as a function of the initial step-size, with hyperparameters \( (\lambda, c_{\alpha}) = (0.25, 1.0) \), are shown in Figure \ref{fig:Boyan_constant_presentation}. As shown in the left panel, the loss increases monotonically with the step-size across all methods. However, the growth is substantially reduced for the average-reward implicit TD($\lambda$). In contrast, the average-reward TD($\lambda$) diverges under similar conditions. The right panel further illustrates the trajectory of the loss value over training. At a moderately large initial step-size (\( \beta_0 = 0.5 \)), the loss of average-reward TD($\lambda$) exceeds that of its implicit counterpart.

\begin{figure}[t]
\centering
\caption{Boyan experiment results under the constant step-size, with exponential weighting parameter and step-size ratio set to $(\lambda, c_{\alpha}) = (0.25, 1.0)$. The solid line represents the mean, and the shaded region denotes the 95\% confidence interval. (Left) Loss value with initial step-size $\beta_0$, from 0.1 to 3.0. (Right) Loss value over iterations with $\beta_0 = 0.5$.}
\label{fig:Boyan_constant_presentation}
    \centering
    \begin{subfigure}[t]{0.4\linewidth}
        \includegraphics[width=\linewidth]{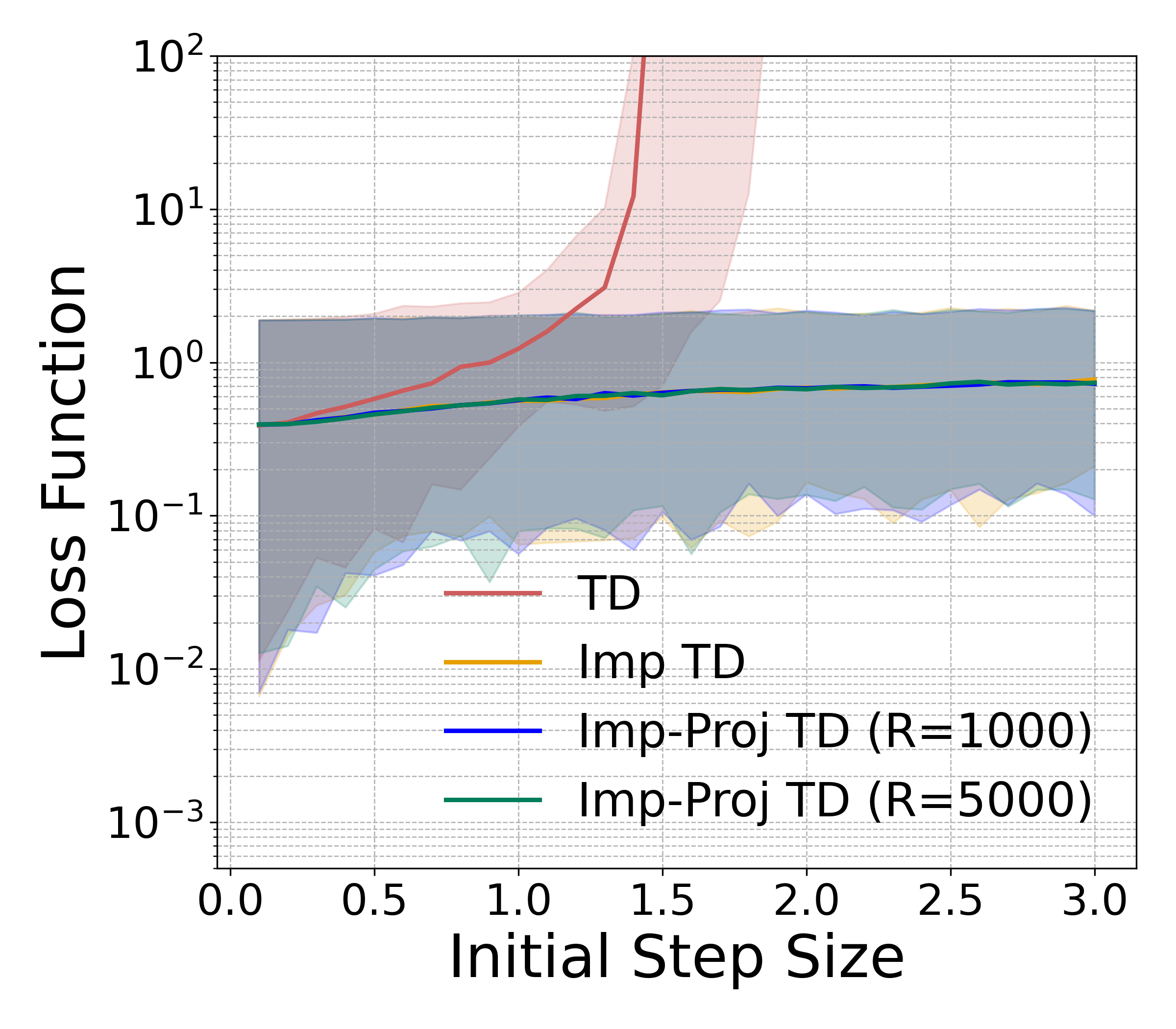}
    \end{subfigure}
    \begin{subfigure}[t]{0.4\linewidth}
    \includegraphics[width=\linewidth]{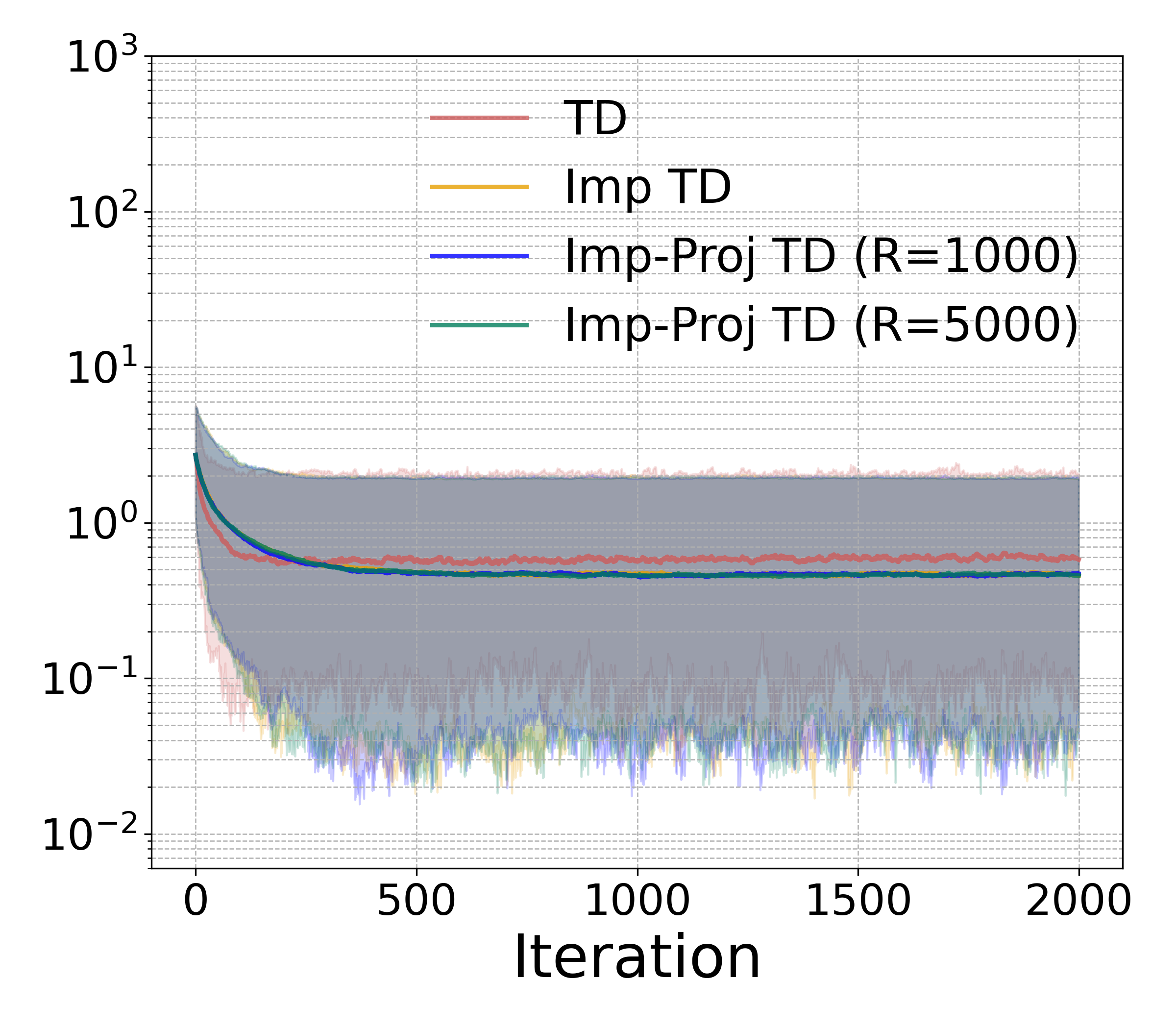}
    \end{subfigure}
\end{figure}

\subsection{Control Experiments}
\label{sec:control_experiments}
We provide details of the control experiment setup. For each action $a$, we form the joint feature  
$
\boldsymbol{\phi}(s,a) = \boldsymbol{\phi}_{\mathrm{RBF}}(s)\otimes e_{a},
$
where $e_{a}$ is the one‐hot encoding of $a \in \mathcal{A}$.  
The state-action value function is approximated by $\widehat{Q}(s,a) = \boldsymbol{\phi}(s,a)^\top \widehat{\boldsymbol{\theta}}_t$, where $\widehat{\boldsymbol{\theta}}_t$ denotes the weight parameter at iteration $t$. We adopt SARSA with $\epsilon$-greedy exploration: the agent selects the greedy action with probability $1-\epsilon$ and a random action with probability $\epsilon$. The exploration parameter $\epsilon$ starts at $0.25$, drops to $0.125$ after 5000 iterations, and is set to $0$ after 10000 iterations. Each experiment runs for $T=15000$ steps, with 30 independent runs. Each component of the initial parameter estimate $\widehat{\theta}_0$ is initialized from $\mathcal{U}[-0.5,0.5]$ and the initial average-reward estimate is $\widehat{\omega}_0 = 0$. The step-size schedule is $\beta_t = \beta_0 / (t + 400)^{0.99}$, with $\beta_0 \in \{400 \times 0.25,\ 400 \times 0.50,\ \ldots,\ 400 \times 1.5\}$. Hence, the effective initial step-size $\beta_0 / 400^{0.99}$ ranges from 0.25 to 1.5.

\subsubsection{Access-Control}
\label{sec:access-control}
\paragraph{Experiment setup} We explain the state space, action space, and the reward function of the access-control experiment. 

\begin{itemize}
    \item \textbf{States and actions:} The state space is defined by the number of free servers and the class of arriving customer. Let $n\in\mathbb{N}$ be the total number of servers and $\mathcal{C}=\{1,2,\ldots,C\}$ the set of customer classes.
The state at time $t$ is
$
S^\mu_t = (k_t, c_t)\in\{0,1,\ldots,n\}\times \mathcal{C},
$
where $k_t$ is the number of free servers and $c_t$ is the class of the arriving customer.
Arrivals are equiprobable, that is:
$
\mathbb{P}(c_t = c) = \tfrac{1}{C}~\text{for each }c\in\mathcal{C},
$
and this distribution is unknown to the decision maker. The decision maker either accepts the customer (action $a_0$) and assigns a free server, or rejects the customer (action $a_1$). Hence, the feasible action set is
\[
\mathcal{A}(S^\mu_t)=
\begin{cases}
\{a_0\ (\text{accept}), a_1\ (\text{reject})\}, & k_t>0,\\
\{a_1\ (\text{reject})\}, & k_t=0.
\end{cases}
\]

\item \textbf{Reward:}
The one-step reward is
$
R^\mu_t = \frac{2^{c_t}}{2^{C}}\mathbb{I}\{a_t=a_0\},
$
i.e., the decision maker receives $2^{c_t}/2^{C}$ if the customer is accepted and $0$ otherwise.

\item \textbf{State transitions:}
If $a_t=a_0$ and $k_t>0$, one available server is allocated. Let
$
\tilde{k}_t = k_t - \mathbb{I}\{a_t=a_0\}
$
then $b_t = n-\tilde{k}_t$ is the number of busy servers immediately after the action.
Each occupied server completes independently with probability $p\in(0,1)$ at the end of the period. Therefore, the number of free servers at the next step is
$
k_{t+1} = \min\{n,\ \tilde{k}_t + Y_t\},
$
where
$
Y_t \sim \mathrm{Binomial}(b_t,p).
$
Lastly, the next arriving class $c_{t+1}$ is drawn independently and uniformly from $\mathcal{C}$. 
\end{itemize}
The goal is to find an optimal policy $\mu$ that maximizes the average-reward $\omega^{\mu}$. We follow the setup in \cite{barto2021reinforcement}, with $C=4$ customer classes, $n=10$ total servers, and a service completion rate $p=0.06$. Accepted customers occupy a server until completion, at which point the server is freed. For feature representation, states are rescaled to $[0,1]$ and embedded via a single-scale random Fourier feature map \citep{scikit-learn,rahimi2008weighted} implemented with scikit-learn’s \texttt{RBFSampler}. The map uses twenty randomly drawn features and sets the inverse length-scale parameter to one, so that inner products in the resulting feature space provide a good approximation to the RBF kernel.

\subsubsection{Pendulum}
\label{sec:pendulum}

\paragraph{Experiment Setup}
At each time step $t$, the state is
$
S_{t}^{\mu}=(\cos\eta_{t},\sin\eta_{t},\dot{\eta}_{t}),
$
where $\eta_{t}$ is the pendulum angle and $\dot{\eta}_{t}$ its angular velocity. The action corresponds to applying a torque to the pendulum. We discretize the continuous action space $[-2,2]$ into five actions $(\mathcal{A}=\{-2,-1,0,1,2\})$. Unlike the episodic setting, this environment has no terminal state and runs indefinitely, so we optimize the long-run average-reward. The per step reward is
$
R_{t}^{\mu} = -\frac{\eta_{t}^{2} + 0.1\dot{\eta}_{t}^{2} + 0.001a_{t}^{2}}{16.27},
$
where the normalization factor 16.27 scales the reward into $[-1,0]$. We use the Gymnasium implementation of the pendulum environment \citep{towers2024gymnasium}. To approximate the state–action value, we use random Fourier features to approximate the RBF kernel. Concretely, we create two separate \texttt{RBFSampler} feature vectors—one using an inverse length-scale  of 0.5 and the other 1.0, each with 150 features, and then concatenate them into a single 300-dimensional feature representation.

\subsection{Effect of Step-Size Ratio}
In this section, we study how the step-size ratio $c_\alpha$ affects stability and performance.  
We revisit the policy evaluation settings for the MRP and average-reward Boyan chain described in Sections~\ref{sec:MRP} and~\ref{sec:APPEND_Boyan}, respectively.  
As in the main experiments, we report the loss value of the form $\bigl(\widehat{\omega}-\omega^\mu\bigr)^2+\bigl\|\Pi_{\mathbb{O}}(\widehat{\pmb{\theta}}-\pmb{\theta}^\ast)\bigr\|^2
$ as the evaluation metric. Under the decaying step-size schedule $\beta_t = \beta_0 / (t+1)^{0.99}$, we vary the step-size ratio $c_\alpha \in \{0.01, 0.05, 0.1, 0.125, 0.25, \dots, 1.5\}$ while fixing the exponential weighting parameter $\lambda = 0.25$ and the initial step-size $\beta_0 = 1.0$. Figure \ref{fig:Boyan_cal_effect} summarizes the result.  
For small values of $c_\alpha$ (e.g., $c_\alpha \leq 0.1$), the average-reward implicit TD($\lambda$) method exhibits a modest increase in the loss value. However, beyond this threshold, the method remains stable across the entire range of $c_\alpha$ values.

\begin{figure}[t]
\centering
\caption{Effect of the step-size ratio $c_\alpha$ in both the MRP and Boyan experiments under a decaying step-size schedule, with exponential weighting parameter $\lambda = 0.25$ and initial step-size $\beta_0 = 1.0$. Solid lines denote mean values; shaded regions represent 95\% confidence intervals. (Left) MRP \quad (Right) Boyan}
\label{fig:Boyan_cal_effect}
    \centering
    \begin{subfigure}[t]{0.4\linewidth}
        \includegraphics[width=\linewidth]{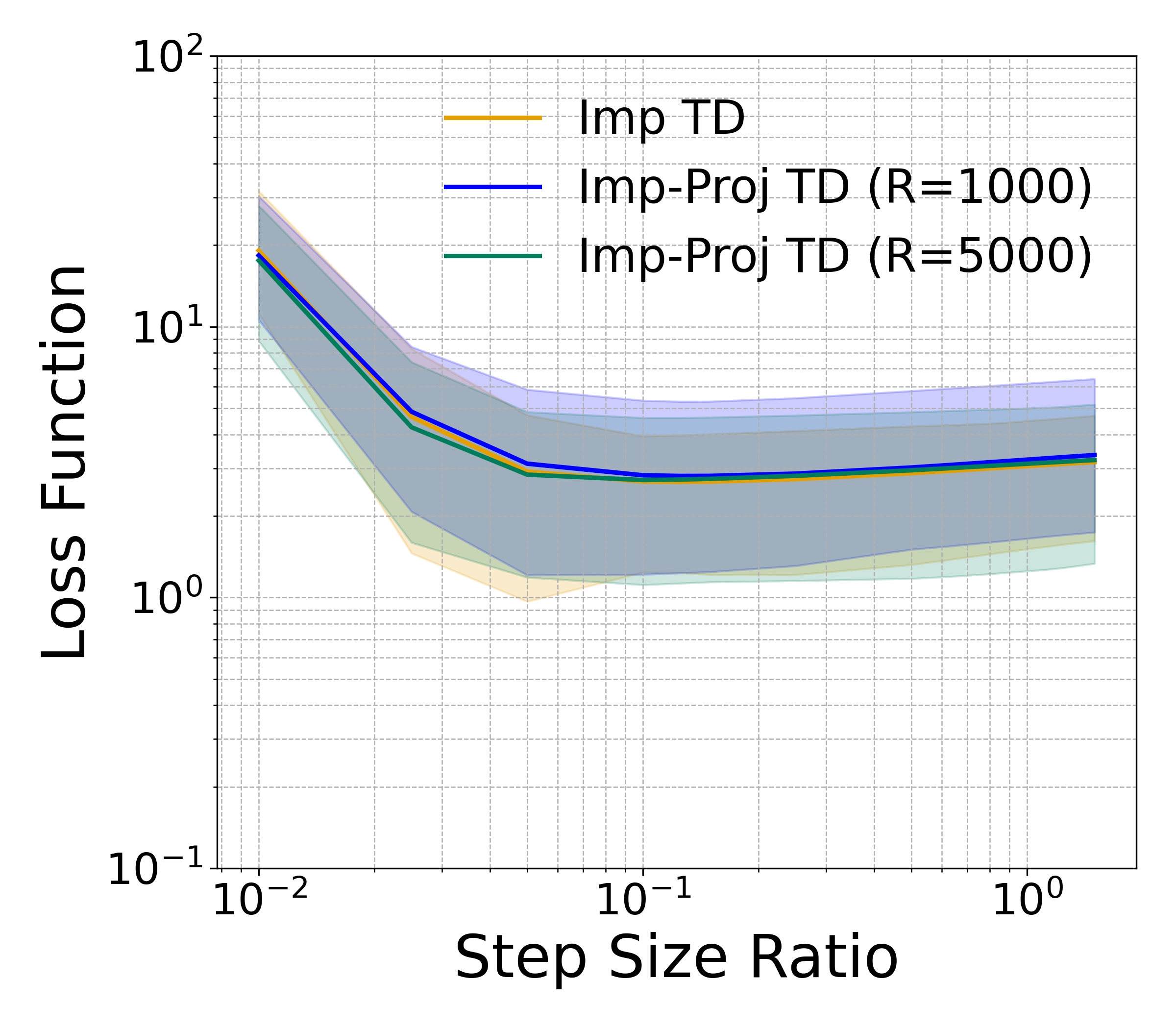}
    \end{subfigure}
    \begin{subfigure}[t]{0.4\linewidth}
        \includegraphics[width=\linewidth]{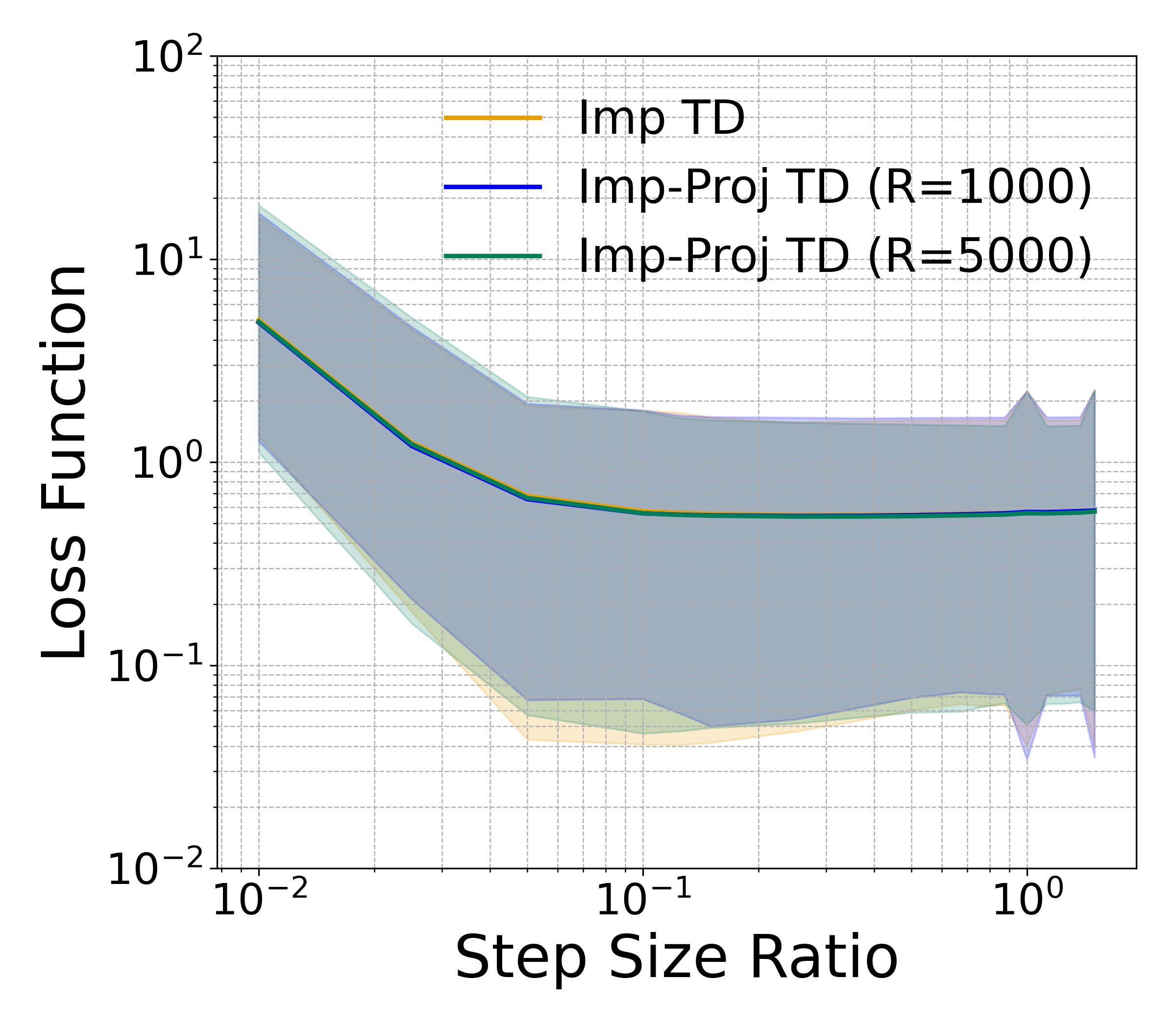}
    \end{subfigure}
\end{figure}

\bibliographystyle{plain}
\bibliography{reference}
\end{document}